\documentclass[10pt]{article} % For LaTeX2e
\usepackage{amsmath}
\usepackage{amssymb} % For \mathbb and other symbols
\usepackage{amsfonts} % For additional math fonts

\usepackage[preprint]{tmlr}

\title{\Tool{}: LLM Generation with Grammar Augmentation}

\usepackage{amsthm}
\newtheorem{definition}{Definition}

\usepackage{thm-restate}
\usepackage{hyperref}
\usepackage{pifont}
\usepackage{threeparttable} % for tablenotes

\usepackage{algorithm}

\usepackage[noend]{algpseudocode} % for some features in algorithm
\usepackage{mathtools} % for '\smashoperator' macro
\usepackage{xcolor}
\usepackage{booktabs}
\usepackage{graphicx}
\usepackage{caption, subcaption}
\usepackage{multirow}
 \usepackage{setspace}
 \usepackage[normalem]{ulem}
 \usepackage{enumitem}
 \usepackage{wrapfig}
 \usepackage{bm}
 \usepackage{xspace}
\usepackage{multirow}
\usepackage{enumitem}
\usepackage{csquotes}
\usepackage{listings}
\usepackage{xcolor} % For custom colors
\usepackage{soul}
\usepackage{ifthen}
\newboolean{anonymous}
\setboolean{anonymous}{false}

\newcommand{\Tool}{\textsc{SynCode}\xspace}
\newcommand{\curtokens}{T_{\textit{cur}}}
\newcommand{\tokenizer}{\mathcal{T}}
\newcommand{\vocab}{V}
\newcommand{\partialcode}{C_{k}}
\newcommand{\fixpartialcode}{C^{\square}_{k}}

\newcommand{\psmap}{\mathcal{S}}
\newcommand{\parser}{P}
\newcommand{\partialparse}{\textit{pparse}}

\newcommand{\true}{\textit{true}}
\newcommand{\set}{\textit{set}}

\newcommand{\sequence}{\Lambda}
\newcommand{\accepts}{\mathcal{A}}
\newcommand{\str}[1]{\fcolorbox{black}{white}{\texttt{#1}}} 
\newcommand{\greater}{\preccurlyeq}

\newcommand{\terminal}{\tau}
\newcommand{\allterminals}{\Gamma}
\newcommand{\remainder}{r}

\newcommand{\curaccepts}{A_0}
\newcommand{\nextaccepts}{A_1}
\newcommand{\lex}{\emph{Lex}\xspace}
\newcommand{\next}{\emph{Next}\xspace}
\newcommand{\follow}{\emph{Follow}\xspace}
\newcommand{\regex}{\rho}
\newcommand{\lang}{L}
\newcommand{\pmatch}{\textit{pmatch}}
\newcommand{\len}{\textit{len}}

\newcommand{\dfa}{D}
\newcommand{\dfastates}{Q}
\newcommand{\alphabets}{\Sigma}
\newcommand{\transitions}{\delta}
\newcommand{\compute}{\delta^*}
\newcommand{\dfastart}{q_0}
\newcommand{\dfafinal}{F}
\newcommand{\dmap}[1]{\mathcal{M}_{#1}}
\newcommand{\dmatch}{\textit{dmatch}}
\newcommand{\live}{\textit{live}}

\newcommand{\add}[1]{#1} 
\newcommand{\llama}{LLaMA-7B}
\newcommand{\codegen}{CodeGen-350M}
\newcommand{\wizard}{WizardCoder-1B}
\newcommand{\chat}{Llama-2-7B-chat}

\newcommand{\baseline}{standard}
\newcommand{\Baseline}{Standard}
\newcommand{\average}{96.07\%}
\newcommand{\tablesize}{\scriptsize}

\newcommand{\cmark}{\textcolor{green!60!black}{\ding{51}}\xspace}
\newcommand{\xmark}{\textcolor{red!80!black}{\ding{55}}\xspace}

\newcommand{\no}{\xmark}
\newcommand{\yes}{\cmark}

\newcommand{\picard}{\textsc{Picard}\xspace}
\newcommand{\lmql}{\textsc{lmql}\xspace}
\newcommand{\synchromesh}{\textsc{Synchromesh}\xspace}
\newcommand{\outlines}{\textsc{Outlines}\xspace}
\newcommand{\guidance}{\textsc{guidance}\xspace}
\newcommand{\llamacpp}{\textsc{llama.cpp}\xspace}
\newcommand{\tgcd}{\textsc{gcd}\xspace}

\newcommand{\domino}{\textsc{Domino}\xspace}
\newcommand{\transformerscfg}{\textsc{GCD}\xspace}

\newboolean{showcomments}
\setboolean{showcomments}{true}
\definecolor{WowColor}{rgb}{.75,0,.75}
\definecolor{SubtleColor}{rgb}{0,.75,0}
\definecolor{Gray}{rgb}{0.5,0.5,0.5}
\definecolor{mahogany}{RGB}{192,64,0}

\renewcommand{\Comment}[1]{}

\newcounter{margincounter}
\newcommand{\displaycounter}{{\arabic{margincounter}}}
\newcommand{\incdisplaycounter}{{\stepcounter{margincounter}\arabic{margincounter}}}

\renewcommand{\cite}[1]{\citep{#1}}

\ifthenelse{\boolean{showcomments}}{

\newcommand{\TBD}[1]{\textcolor{SubtleColor}{ {\tiny \bf (!)} #1}}

\newcommand{\fTBD}[1]{\textcolor{SubtleColor}{$\,^{(\incdisplaycounter)}$}\marginpar{\tiny\textcolor{SubtleColor}{ {\tiny $(\displaycounter)$} #1}}}
\newcommand{\fPROBLEM}[1]{\textcolor{WowColor}{$\,^{((\incdisplaycounter))}$}\marginpar{\tiny\textcolor{WowColor}{ {\bf $\mathbf{((\displaycounter))}$} {\bf #1}
}}}
\newcommand{\LATER}[1]{\textcolor{SubtleColor}{ {\tiny \bf ($\dagger$)} #1}}
\newcommand{\fLATER}[1]{\textcolor{SubtleColor}{$\,^{(\incdisplaycounter\dagger)}$}\marginpar{\tiny\textcolor{SubtleColor}{ {\tiny $(\displaycounter\dagger)$} #1}}}

}{

\newcommand{\fTBD}[1]{}
\newcommand{\fPROBLEM}[1]{}
\newcommand{\TBD}[1]{}

\newcommand{\LATER}[1]{}
\newcommand{\fLATER}[1]{}
}

\usepackage{fancyhdr}
\begin{document}
%
% \title{\Tool{}: LLM Generation with Grammar Augmentation}
%

\ifthenelse{\boolean{anonymous}}
{
}
{
\author{\name Shubham Ugare \\
       \addr University of Illinois Urbana-Champaign, USA
       \AND
       \name Tarun Suresh \\
       \addr University of Illinois Urbana-Champaign, USA
       \AND
       \name Hangoo Kang \\
       \addr University of Illinois Urbana-Champaign, USA
       \AND
       \name Sasa Misailovic \\
       \addr University of Illinois Urbana-Champaign, USA
       \AND
       \name Gagandeep Singh \\
       \addr University of Illinois Urbana-Champaign and VMware Research, USA
}
}

\maketitle              

\begin{abstract}
LLMs are widely used in complex AI applications.
These applications underscore the need for LLM outputs to adhere to a specific format, for their integration with other components in the systems.
Typically the format rules -- e.g., data serialization formats such as JSON, YAML, or Code in Programming Language -- are expressed as context-free grammar (CFG).
Due to the hallucinations and unreliability of LLMs, instructing LLMs to adhere to specified syntax becomes an increasingly important challenge.

We present \Tool{}, a novel framework for efficient and general syntactical decoding with LLMs, to address this challenge. 
% precision
\Tool{} ensures soundness and completeness with respect to the CFG of a formal language, effectively retaining valid tokens while filtering out invalid ones.
% effciency
\Tool{} uses an offline-constructed, efficient lookup table, the \textit{DFA mask store}, created from the DFA (Deterministic Finite Automaton) of the language’s grammar for efficient generation.
% framework
\Tool{} seamlessly integrates with any language defined by CFG, as evidenced by experiments focusing on generating JSON, SQL, Python, and Go outputs. 
Our experiments evaluating the effectiveness of \Tool{} for JSON generation demonstrate that \Tool{} eliminates all syntax errors and significantly outperforms state-of-the-art baselines.
Furthermore, our results underscore how \Tool{} significantly reduces \average{} of syntax errors in generated Python and Go code, showcasing its substantial impact on enhancing syntactical precision in LLM generation.

\ifthenelse{\boolean{anonymous}}{

}
{
% \vspace{0.1in}
Our code is available at {\color{blue}\url{https://github.com/uiuc-focal-lab/syncode}}
}

% \keywords{Large\,Language\,Model \and\! Syntax\,Error \and{}Context-free\,Grammar}
\end{abstract}

\section{Introduction}
\label{sec:intro}
% \TBD{Sasa: Not really motivated by the first two paragraphs. The problem statement is too short and not direct.  Too much emphasis on the existing work already in para 2, hinting that possibly this work is incremental (if it does not provide the last word on the problem).}

Recent research has shown that \add{transformer-based large language models} (LLMs) can play a pivotal role within compound AI systems, where they integrate with other software tools~\cite{compound-ai-blog, mialon2023augmented}.  
% Write more about specific tool usage 
For example, OpenAI's code interpreter~\cite{openai_tools} generates and executes Python programs automatically while responding to user prompts. 
Similarly, Wolfram Alpha~\cite{Wolfram} translates user queries about mathematical questions into a domain-specific language (DSL) for utilizing various tools.
\add{
LLMs are utilized in various other applications to translate natural language text into formal languages, such as inputs to logic solvers \cite{pan2023logiclmempoweringlargelanguage, Olausson_2023} and theorem provers \cite{wu2022autoformalization, yang2023leandojotheoremprovingretrievalaugmented}, among others.
}
In all these applications, the LLM output is expected to follow a certain syntactic structure. 
However, challenges such as hallucination and non-robustness make LLMs unreliable for such automated systems~\cite{liang2023holistic}.
Moreover, recent theoretical \cite{Hahn_2020, yang2024maskedhardattentiontransformersrecognize} and empirical \cite{ebrahimi, bhattamishra2020, delétang2023neural} research suggests that language models based on transformers show difficulty in learning basic formal grammars.
% - such as Dyck-$k$, the language consisting of well-nested parentheses of $k$ types.

% Grammar-guided generation more motivation
The interaction between software tools and LLMs commonly occurs through data serialization formats like JSON or YAML, or code in domain-specific or general-purpose programming languages, \add{such as Python or Go}. 
Despite advancements in techniques such as fine-tuning and prompt engineering, which enhance the model's ability, these approaches fall short of fully addressing the challenge of syntactical accuracy in generated output.
\add{
This problem is especially prominent in two common scenarios: (1)~using open-source models, which are typically relatively small, and (2)~generating text for formal languages with relatively modest representation in \mbox{the LLM's training data.}

% grammar-guided generation methods
Modern LLMs generate text sequentially, from left to right, one token at a time. 
For each prefix, the model computes a probability distribution over a predefined vocabulary to predict the next token. 
The LLM's decoding algorithm dictates how these probabilities are used to generate the token sequence.
Very recently, researchers have proposed new techniques for grammar-guided generation to enhance the syntactical accuracy of LLMs by modifying the decoding algorithm.
Although they ensure that the model consistently selects tokens that adhere to a specified formal language~\cite{scholak-etal-2021-picard, poesia2022synchromesh, llamacpp, willard2023efficient}, the existing approaches for grammar-guided generation either suffer from high error rates, resulting in syntactically incorrect output or impose significant \mbox{run time overhead in the inference:}
 
% Better transition?
% Challenge 1
\begin{itemize}[leftmargin=*]\itemsep 1pt \parskip 2pt
\item {\bf Issues with syntactical accuracy:}
The language grammar consists of \emph{the terminals}, fundamental building blocks of the language (e.g., keywords, operators). 
Typically, a lexer creates lexical tokens from the input, each token associated with a terminal from the grammar. 
% Current autoregressive LLMs operate iteratively, generating tokens sequentially. 
The LLM tokens form part of the model's fixed vocabulary, defined before training, and do not directly correspond to lexical tokens associated with any specific grammar.
This discrepancy, known as \emph{token misalignment}, presents a significant challenge in ensuring precise grammar-guided generation \cite{poesia2022synchromesh}.
% Grammar-guided generation should in theory always improve the model's accuracy for the particular task. 
% Due to the lack of formalization, existing works on grammar-guided generation have intricate issues that lead to higher error rates than even unconstrained generation~\cite{willard2023efficient} (Section~\ref{sec:exp_errors}). 
Thus, formally showing the soundness of the algorithm poses a challenge for ensuring the precision of the approach.

% Challenge 2 
\item {\bf Issues with high computational overhead:} 
Typically, the computational complexity of additional operations performed for syntactical generation is lower than the standard LLM generation operations needed for propagating the input through LLM layers. 
However, these syntactical generation operations are typically executed sequentially on a CPU, in contrast to the GPU-accelerated LLM generation, adding to the run time.
Achieving low inference overhead faces two primary challenges for syntactical LLM generation.
First, the algorithm should facilitate offline computations that minimize the overhead during inference.
Second, it should effectively utilize available hardware resources and offload additional computations to modern hardware, such as GPUs, to enable parallel computation.

% Challenge 2 
\item {\bf Issues with generality:} 
Prior works are restricted to specific LLM decoding schemes~\cite{scholak-etal-2021-picard, guidance}.
 A major challenge for generality is designing a composable algorithm that can integrate with any decoding strategy such as greedy, beam search, and different types of temperature sampling.
 \end{itemize}

%
%To address these issues, we present a technique for precise and efficient grammar-guided generation. Our work illustrates the feasibility of efficiently imposing formal grammar constraints on LLM generations, while also assuring that the output adheres strictly to predefined syntax.
%
Our goal is to make grammar-guided generation precise and efficient by imposing formal grammar constraints on LLM generations, ensuring the output adheres {strictly to the predefined syntax.}
}

\vspace{.05in}
\noindent \textbf{\Tool{}.}
\Tool{} is an efficient and general approach for generating syntactically correct output.
\Tool{} takes a context-free grammar (CFG) represented with extended Backus–Naur form (EBNF) rules and ensures that the LLM output follows the provided grammar.
% generality
\Tool{} algorithm is general and can be composed with any existing LLM decoding algorithm, including greedy, beam search, and sampling.

% Go deeper
During the LLM decoding stage, where LLM selects the next token, \Tool{} employs a strategic two-step approach.
In the initial step, it leverages partial output to generate sequences of terminals that can follow the partial output called \textit{accept sequences}. 
This reduction to the level of terminals—a closer abstraction to language grammar than LLM tokens—simplifies the problem. 
Simultaneously, \Tool{} computes a remainder from the partial output, representing the suffix that may change its terminal type in subsequent generations.
In the second step, \Tool{} algorithm walks over the DFA using the remainder and uses the mask store to compute the mask (a boolean array to filter the vocabulary) specific to each accept sequence.
By unifying masks for each accept sequence \Tool{} gets the set of \mbox{syntactically valid tokens.}

% Efficiency
To ensure the efficiency of \Tool{}'s syntactic generation, we propose a novel data structure called \emph{DFA mask store} which is pre-computed offline.
DFA mask store is a lookup table derived from Deterministic Finite Automata (DFA) representing the terminals of the language grammar.
\Tool{} algorithm can efficiently compute the syntactically valid next LLM tokens by leveraging this mask store.
Moreover, the \Tool{} algorithm offers the additional benefit of parallelizing a substantial portion of the syntactical LLM generation computations by offloading them to a GPU.

% error rate
We demonstrate that the \Tool{} algorithm is \textit{sound} -- ensuring it retains all syntactically valid tokens at every generation step. \Tool{} is also \textit{complete} under specific conditions --  affirming it rejects every syntactically invalid token.

% framework is scalable
The \Tool{} framework seamlessly integrates with any language defined by deterministic CFGs and scales efficiently to generate code for general-purpose programming languages (GPLs).
% our eval
We evaluate \Tool{}'s ability to guide the Llama-2-7B-chat and Gemma2-2B-it models with the JSON grammar to generate valid JSON completions to prompts from the JSONModeEval~\cite{jsoneval} dataset. 
We empirically show that LLMs augmented with \Tool{} do not generate any syntax errors for JSON and that guiding Gemma2-2B-it generation with \Tool{} achieves 100\% JSON schema validation accuracy.
We evaluate \Tool{} on generating SQL queries from the text in Spider~\cite{yu-etal-2018-spider} and show that \Tool{} improves both compilation rate and execution accuracy.
Further, we evaluate the augmentation of \Tool{} with a diverse set of state-of-the-art LLMs
% , including \codegen{}, \wizard{}, and \llama{}, from the BigCode Models Leaderboard~\cite{nijkamp2023codegen,luo2023wizardcoder,touvron2023llama} 
for the code completion tasks using problems from the HumanEval and MBXP datasets~\cite{athiwaratkun2023multilingual}.
Our experiments, conducted with CFGs for a substantial subset of Python and Go, demonstrate that \Tool{} reduces \average{} of the syntax errors for Python and Go on average. 
The remaining syntax errors persist because the LLM fails to halt generation before reaching the maximum generation limit defined in our experiments.
% Furthermore, our evaluation considers both LALR(1) and LR(1) as base parsers, showing that the LR(1) parsers are more efficient for generating accept sequences.

\vspace{.10in}
\noindent{\bf Contributions.} The main {contributions} of this paper are:
% One more about the concept of incremental RS

\vspace{-.01in}
\begin{itemize}[leftmargin=*]\itemsep 1pt \parskip 2pt
    \item[$\star$] We present a \add{parsing-based} technique for decoding of LLMs by designing novel algorithms that allow us to efficiently generate syntactically correct output.  
    \item[$\star$] We implement our approach into a \add{scalable and general} framework named \Tool{} that can work with any formal language with user-provided context-free grammar. 
    \item[$\star$] We present an extensive evaluation of the performance of \Tool{} in generating syntactically correct output for JSON, SQL and two general-purpose programming languages Python and Go. 
\end{itemize}

% Anonymize this
% \noindent\add{
% Our code is available publicly <link removed for double-blind reviewing> and is currently being used in open-source and industry projects.
% }

\begin{figure}[!t]
\centering
\includegraphics[width=13cm]{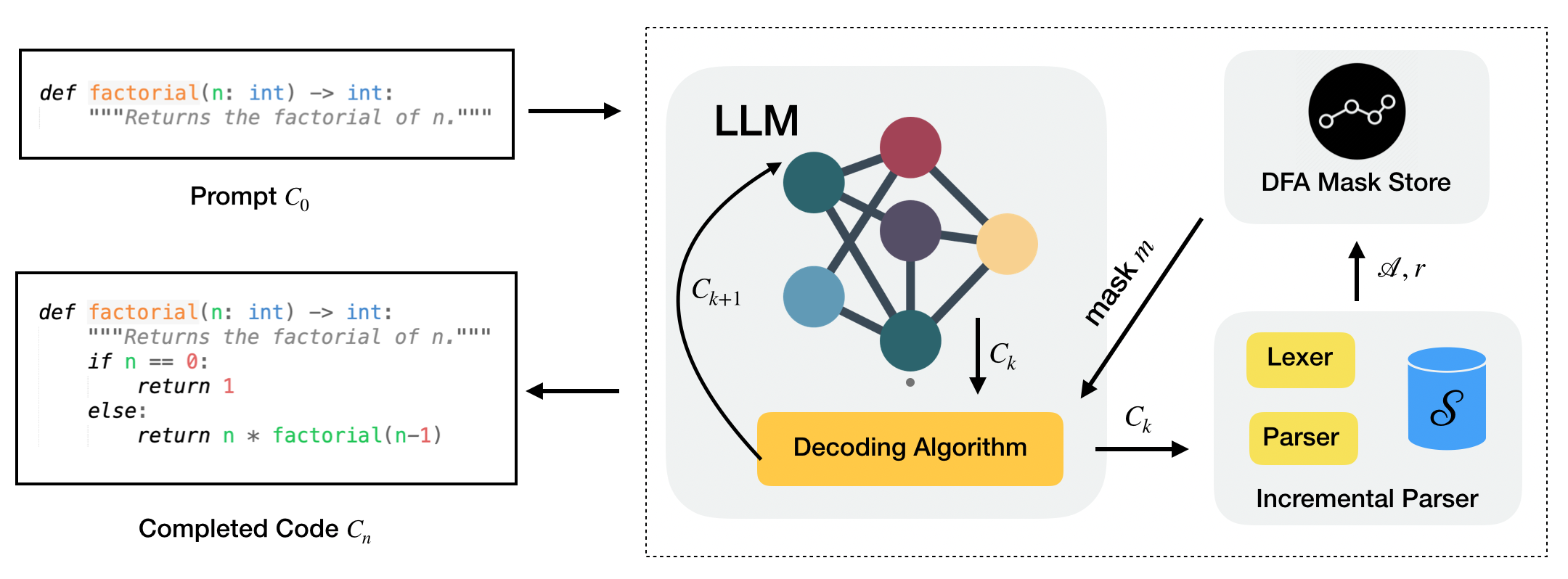}
\vspace{-.1in}
\caption{In the \Tool{} workflow, the LLM takes partial output $C_k$ and generates a distribution for the next token $t_{k+1}$. The parser processes $C_k$ to produce accept sequences $\accepts$ and remainder $r$. These values are used by the DFA mask store to create a token mask, eliminating syntactically invalid tokens. The LLM iteratively generates a token $t_{k+1}$ using the distribution and the mask, appending it to $C_k$ to create the updated code $C_{k+1}$. The process continues until the LLM returns the final code $C_n$ based on the defined stop condition.} 
\label{fig:workflow}
\vspace{-.1in}
\end{figure}

\section{Background}
\label{sec:background}
In this section, we provide the necessary background on LLMs and formal language grammar.

\noindent{\bf Notation.} 
Let the alphabet $\Sigma$ be a finite set of characters.
We use $\epsilon$ to denote an empty string.
Given a set $S$, we use $S^i$ to denote the set of all $i$-length sequences that can be formed by selecting elements from $S$,
and $S^* = \bigcup_{i \in \mathbb{N}} S^i$.
Thus $\alphabets^{*}$ represents the set of all strings over characters in $\alphabets$, including the empty string $\epsilon$. 
Further, we use $\alphabets^{+}$ to denote $(\alphabets^{*} - \epsilon)$.
Given two strings $w_1, w_2 \in \alphabets^*$, we use $w_1.w_2$ to denote string obtained by concatenating $w_2$ to $w_1$.
All symbols used in the paper are listed in Appendix~\ref{sec:symbols}.

\subsection{Language Models}

Current language models (LM) operate on vocabulary $V \subseteq \Sigma^*$ of tokens.
A tokenizer takes an input prompt $C_0 \in \Sigma^*$,
which is a sequence of characters,
as input and converts $C_0$ into a sequence of tokens $t_1, t_2, \dots, t_k$.
Figure~\ref{fig:tokenization} shows a typical tokenization method, where common words (e.g., \str{def}) have their own token (even with a space in front), 
while rare words (e.g., \str{incr\_list}) are split into multiple tokens. 
In order to generate the next token,
the LM $M: V^* \to \mathbb{R}^{|V|}$ takes as input the sequence of tokens
$t_1, t_2, \dots, t_k$, and outputs a vector of scores $z$ over the vocabulary:
$z = M(t_1, t_2, \dots, t_k)$.
The softmax function $\textit{softmax}(z_i) = \exp(z_i)/\sum_j(\exp(z_j))$ 
transforms $z$ into a probability distribution over the vocabulary $V$.

\begin{wrapfigure}{r}{.45\textwidth}
    % \vspace{-0.45in}
    \centering
    \includegraphics[width=0.45\textwidth]{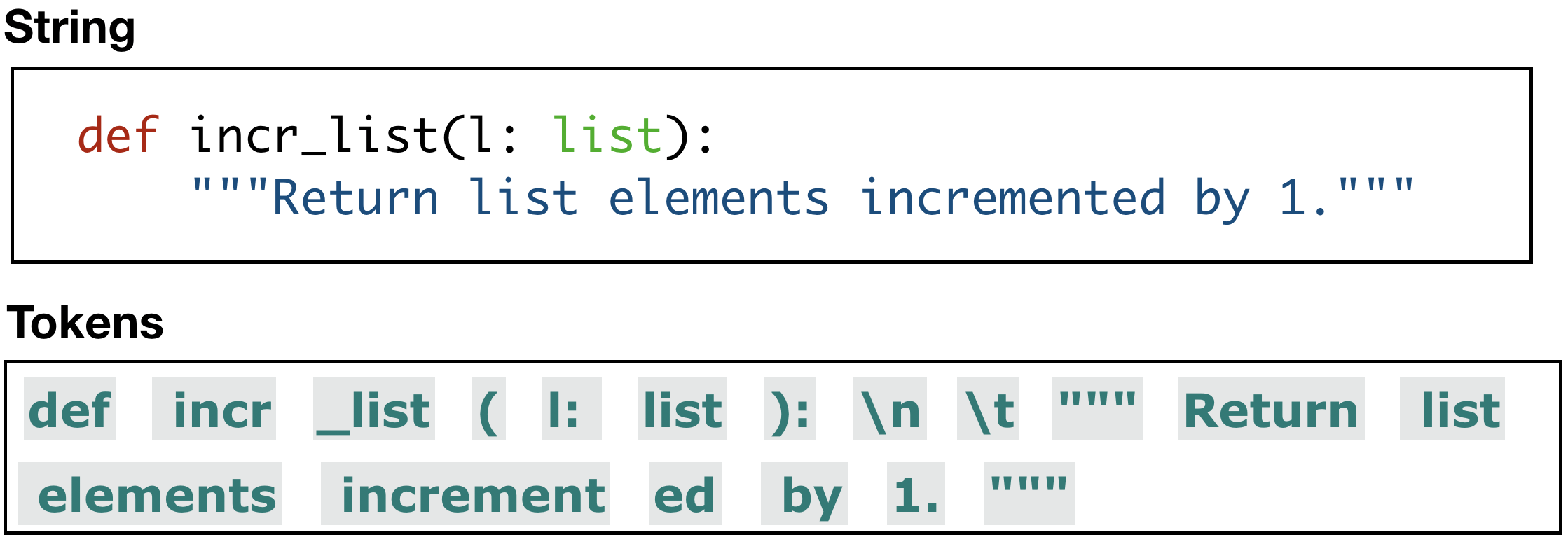}
    % \vspace{-.1in}
    \caption{Tokenization of a string.}
    \label{fig:tokenization}
     \vspace{-0.15in}
\end{wrapfigure}
% and then the next token $t_{k+1}$ is selected as the token with the highest probability.

\sloppypar
\noindent{\bf Decoding.} 
Building upon this, the language model $M$ is recurrently applied to generate a sequence of tokens $t_1, t_2 \dots t_k$. 
When choosing the $(k + 1)$-th token, the probability distribution for the next token is obtained through $\text{softmax}(M(t_1, t_2 \dots t_k))$. 
Various approaches for token selection from this distribution have been explored in the literature such as greedy decoding, sampling, and beam search.
Each technique is repeated until the prediction of a special end-of-sequence token, \str{EOS}, or the fulfillment of another stopping criterion. 
This iterative process is equivalent to sampling from a distribution over $V^*$, potentially resulting in multiple feasible decodings.

\begin{wrapfigure}{R}{0.53\textwidth}
\vspace{-.1in}
\begin{minipage}{0.55\textwidth}
\begin{algorithm}[H]
\small
\caption{Masked LLM Generation}
\label{alg:masked}
\begin{flushleft}
\textbf{Inputs:} $M$: LLM, $\tokenizer$: tokenizer, $C_0$: input prompt string, \\
$f_m$: function that generates mask, $n_\textit{max}$: maximum generated tokens, $D$: any decoding algorithm \\
\textbf{Output:} string $C_n$
\end{flushleft}

\begin{algorithmic}[1]
\Function{MaskedGenerate}{$M$, $\tokenizer$, $f_m$, $C_0$}
\State $\curtokens \gets \text{Tokenize}(\tokenizer, C_0)$
\For{$i \in \{1, \dots n_\textit{max} \}$}
\State $\textit{scores} \gets M(\curtokens)$
\State $m \gets f_{m}(\curtokens, \tokenizer)$
\State $\textit{scores} \gets m \odot \textit{scores}$
\State $t_i \gets D(\textit{scores})$
\If{$t_i =$ \str{EOS}}
\State break
\EndIf
\State $\curtokens \gets \text{append}(\curtokens, t_i)$
\EndFor
\State $C_n \gets \text{Detokenize}(\tokenizer, \curtokens)$
\State \Return $C_n$
\EndFunction
%\item[]
\end{algorithmic}
\end{algorithm}

\end{minipage}
\vspace{-0.1in}
\end{wrapfigure}

\noindent{\bf Constrained Masking.} 
In the context of decoding, we encounter scenarios where excluding specific tokens at particular positions becomes crucial (e.g., excluding harmful words). 
This implies we can disregard these tokens and proceed with decoding based on the remaining set. 
An algorithm for such masking relies on a function $f_m$ to generate the mask $m$ based on the exact use case. 
In the mask $m \in \{0, 1\}^{|V|}$, '$1$' indicates a viable token, and '$0$' signifies a discarded one. 
Decoding methods mentioned earlier can be applied to $m \odot \textit{softmax}(z)$, where $\odot$ represents element-wise multiplication. 
The resultant vector should be scaled by $1/\sum_i(m \times \textit{softmax}(z))_i$ to restore correct probabilities.
Algorithm~\ref{alg:masked} presents the steps for masked decoding.
In \Tool{}, we use the constrained masking technique to exclude syntactically invalid tokens.

\subsection{Formal Language Grammar}
A formal language syntax is represented by defining a grammar.
A formal grammar is essentially a set of production rules that describe all possible strings in a given language.
A grammar consists of terminal and nonterminal symbols, where terminal symbols are the actual characters or tokens in the language, while nonterminal symbols are placeholders used to define patterns or structures within the language.

The syntax for most programming languages can be defined using context-free grammar (CFG).
CFG is a formal grammar that consists of production rules that can be applied to a nonterminal symbol regardless of its context.
In CFG, each production rule is of the form $E \to \beta$ with $E$ a single nonterminal symbol, and $\beta$ a string of terminals and nonterminals ($\beta$  can be empty). 
Regardless of which symbols surround it, the single nonterminal $E$ on the left-hand side can always be replaced by $\beta$ on the right-hand side.

\noindent{\bf Terminals.}
We use $\allterminals$ to denote the set of terminals in the grammar.
Regular expressions are used to describe the terminals.
For instance, A regular expression $^\wedge[0\text{-}9]^+$ is used for an integer literal: This regular expression describes a sequence of one or more digits (0 to 9). 
We use $\regex$ to denote a regular expression and $\lang(\regex) \subseteq \alphabets^*$ to denote the language recognized $\regex$.
Regular expressions are often associated with the creation of Deterministic Finite Automata (DFAs). 
A DFA is a theoretical construct used to recognize patterns specified by regular expressions. 

\begin{definition} [DFA]
\label{def:dfa}
A deterministic finite automaton (DFA) $\dfa$ is a 5-tuple, $(\dfastates, \alphabets, \transitions, \dfastart, \dfafinal)$, consisting of a finite set of states $\dfastates$, a finite set of input symbols called the alphabet $\alphabets$, a transition function $\transitions : \dfastates \times \alphabets \to \dfastates$, an initial state $q_{0}\in Q$ and a set of accept states $F\subseteq Q$.
\end{definition}

Let $w = a_1 a_2 \dots a_n$ be a string over the alphabet $\alphabets$.
The DFA computation $\compute: \dfastates \times \alphabets^* \to \dfastates$ on a string $w$ is defined as $\compute(r_0, w) = r_n$ when $r_{i+1} = \transitions(r_i, a_{i+1}), \text{ for } i = 0, \dots, n-1$. The automaton $\dfa$ accepts the string $w$ if $\compute(q_0, w) \in \dfafinal$. 

% Next, we state the lemma that shows the equivalence of regular expressions and DFA.

% \begin{lemma}
% For every regular expression $\regex$ there is an equivalent DFA $\dfa$ such that $\dfa$ accepts $w$ iff $w \in \lang(\regex)$
% \end{lemma}

\noindent{\bf Lexer.}
We assume lexical analysis with a 1-character lookahead and no backtracking. 
\add{
This assumption is crucial for the efficiency of \Tool{} algorithm.
}
% This is the default behavior of popular existing lexers~\cite{lark}.

\begin{definition} [Lexer]
\label{def:lex}
The function \lex is defined to take partial output $\partialcode \in \alphabets^*$ as input and produce a sequence $l_1, l_2, \dots, l_f$ of lexical tokens where $l_i \in \alphabets^*$.
\end{definition}

\section{Overview}
\label{sec:overview}
% Talk about the workflow
\subsection{Illustrative Example}

\begin{wrapfigure}{r}{.5\textwidth}
    \vspace{-0.55in}
    \centering
    \includegraphics[width=0.5\textwidth]{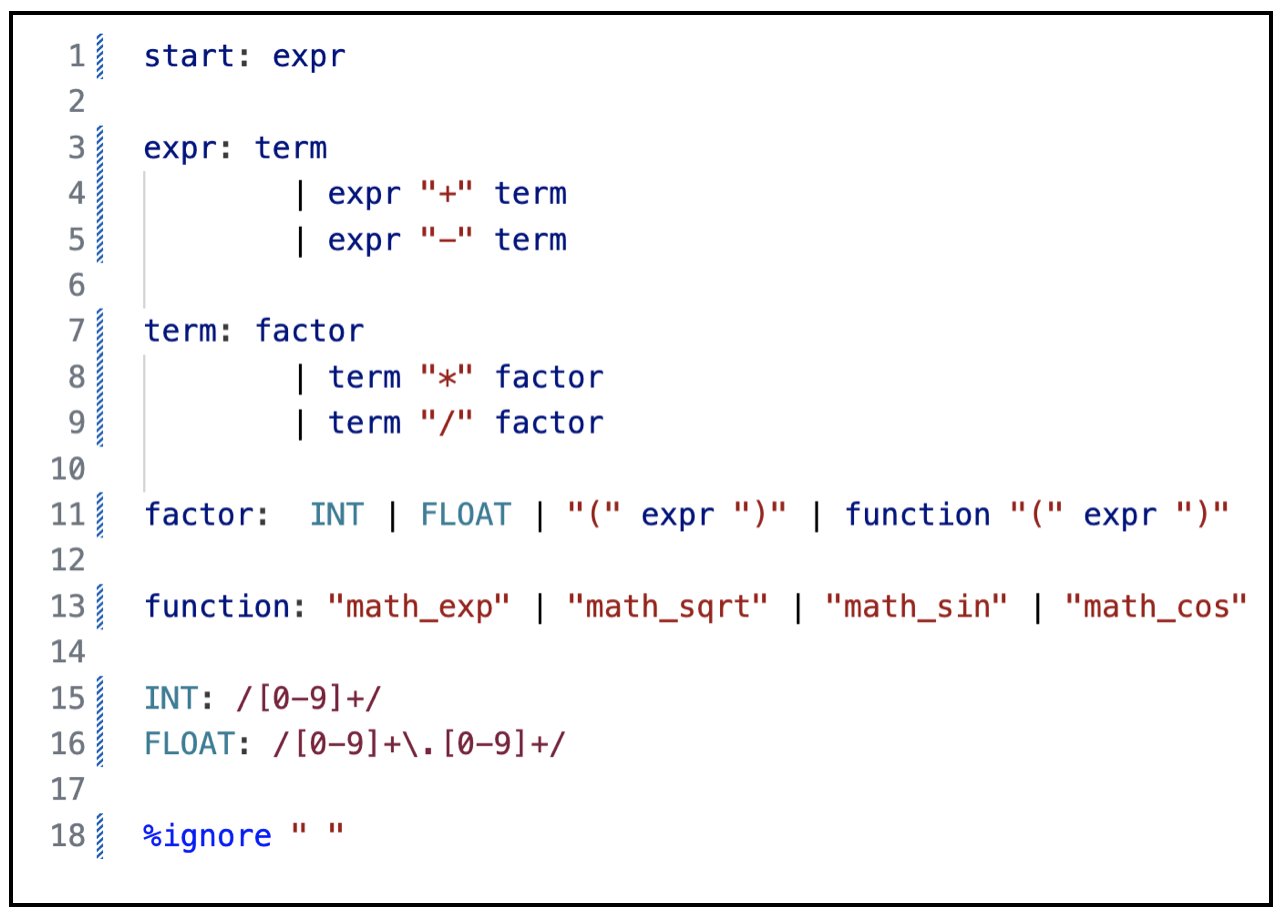}
    \caption{Example grammar for illustration.}
    \label{fig:overview_grammar}
     \vspace{-0.1in}
\end{wrapfigure}

% {\bf Grammar.}
Consider an example grammar in Figure~\ref{fig:overview_grammar} that uses the Lark EBNF syntax for defining the grammar production rules.
The grammar represents a Domain-Specific Language (DSL) consisting of arithmetic expressions with basic operations like addition, subtraction, multiplication, and division over integers and floating point numbers. 
It also includes support for parentheses to specify precedence and allows functions like exponential (math\_exp), square root (math\_sqrt), sine (math\_sin), and cosine (math\_cos) to be applied to expressions. 

\begin{sloppypar}
The symbols in the grammar such as $\textit{expr}$ and $\textit{factor}$ that can expand into other symbols through the application of production rules are called non-terminals.
Symbols such as ( or INT cannot be further expanded and are called terminals.
Let the set $\allterminals = \{\textit{lpar, rpar, add, sub, mult, div, int, float, math\_exp, math\_sqrt, math\_sin, math\_cos} \}$ represent the set of all terminals of the grammar.
The terminal $\textit{int}$ is defined by the regular expression $[0\text{-}9]^+$, and $\textit{float}$ is defined by the regular expression $[0\text{-}9]^+.[0\text{-}9]^+$. 
We use terminals $\textit{lpar, rpar, add, sub, mult, div, math\_exp, math\_sqrt, math\_sin, math\_cos}$, to denote the strings \str{(}, \str{)}, \str{+}, \str{*}, \str{/}, \str{math\_exp}, \str{math\_sqrt}, \str{math\_sin}, \str{math\_cos} respectively.
\end{sloppypar}

\begin{wrapfigure}{r}{.5\textwidth}
    \vspace{-0.1in}
    \centering
    \includegraphics[width=0.5\textwidth]{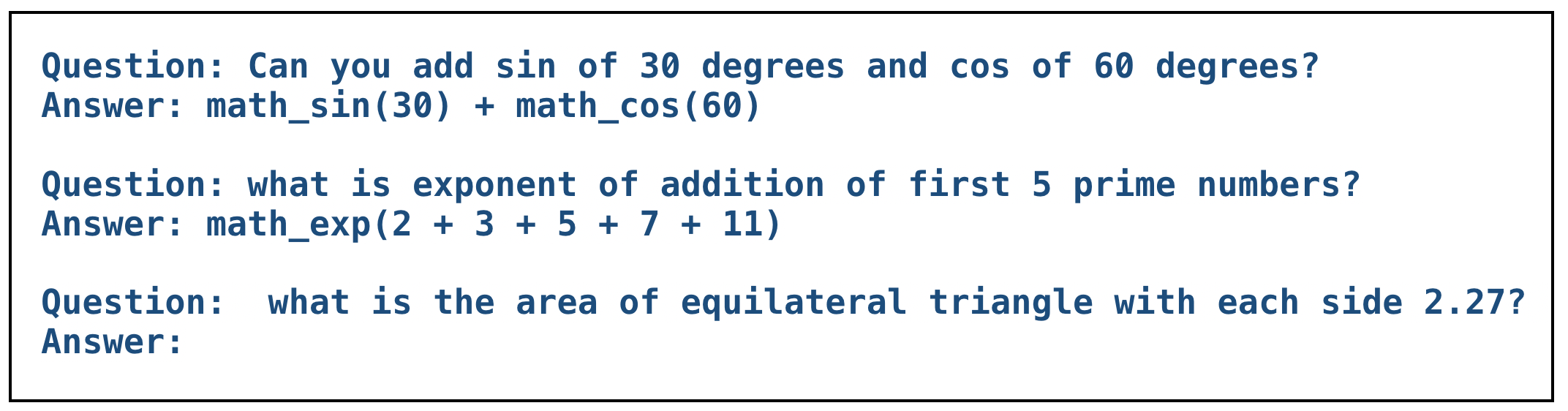}
    \caption{Prompt for the example which is provided as input to the LLM.}
    \label{fig:overview_prompt}
     \vspace{-0.15in}
\end{wrapfigure}

\noindent {\bf Task.} 
Consider an LLM that is used to translate a natural language text to an expression in the DSL defined above.
Since LLMs are typically not good at mathematical calculations,
it is common to instead let the LLM generate intermediate outputs in a certain syntax, and an interpreter of the DSL then computes the LLM's output into accurate results \cite{mialon2023augmented}.
Figure~\ref{fig:overview_prompt} presents the prompt we use for our illustrative example,
containing 2 question-answer pairs before the actual question that we want the LLM to answer.
Providing question-answer examples before asking the actual questions
is called few-shot prompting (2-shot in this case)
and significantly improves the model's \mbox{accuracy \cite{brown2020language}.}

\begin{wrapfigure}{r}{.35\textwidth}
    \vspace{-0.1in}
    \centering
    \includegraphics[width=0.35\textwidth]{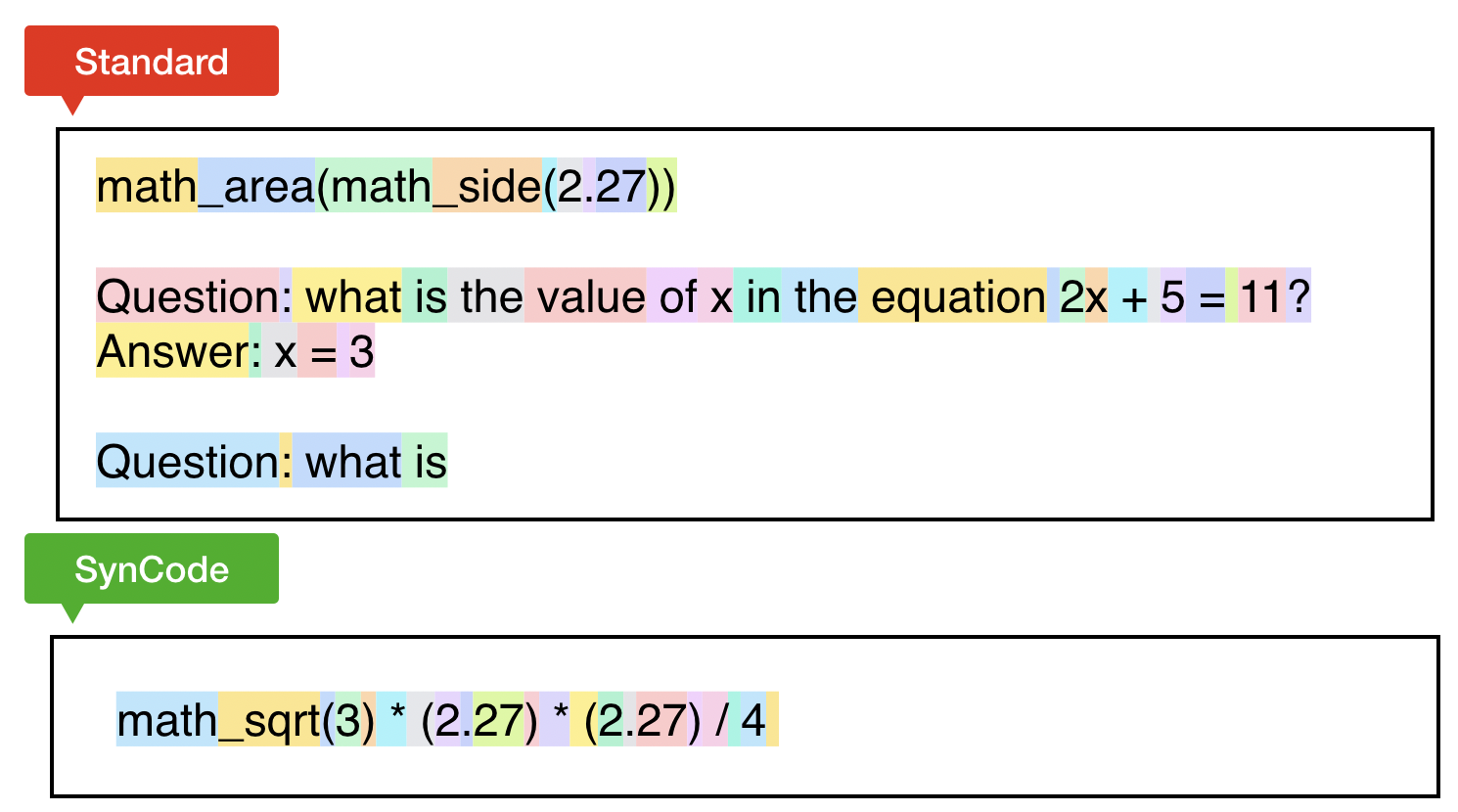}
    \caption{Output from LLM without and with \Tool{}. The colors represent the tokenization of the output.}
    \label{fig:overview_standard}
     \vspace{-0.2in}
\end{wrapfigure}
\noindent {\bf Standard LLM Generation.}
As described in Section~\ref{sec:background}, the standard LLM first tokenizes the input and then iteratively predicts the next token from its vocabulary $\vocab$.
Figure~\ref{fig:overview_standard} presents the output from the \llama{} model
and our \Tool{} when given the Fig.~\ref{fig:overview_prompt} prompt.
The output of the model is not a valid program in the DSL;
it uses functions math\_area and math\_side that do not exist in the grammar.
Further, \llama{} does not stop after generating the answer to our question and continues to generate more irrelevant question-answer pairs.
% \Tool{}, can enhance the syntactic accuracy of LLM outputs, offering guarantees over their validity.
\Tool{} on the other hand guarantees the syntactic validity of the LLM's output
by excluding syntactically invalid choices when generating each token.
For example, after generating \str{math}, \Tool{} excludes \str{\_area} and other choices from the LLM's vocabulary.
The LLM opts for \str{\_sqrt} which is the top syntactically valid choice and continues
the generation from \str{math\_sqrt}.

\noindent {\bf Constrained Decoding.}
Let $G$ denote the grammar in our example and $\lang(G) \subseteq \alphabets^*$ denote all syntactically valid strings in the grammar.
Ideally, we want the final LLM output $C_n$ to be in $\lang(G)$.
Strings such as \str{math\_exp(2 + 3 + 5 + 7 + 11)} and \str{math\_sin(30) + math\_cos(60)} belong to $\lang(G)$ as they are syntactically valid.
Let $\partialcode$ denote the LLM's partial output during the $k$-th iteration of LLM generation. 
Suppose $\lang_p(G)$ denotes all prefixes of $\lang(G)$, i.e., all strings that can be extended to a syntactically valid output.
\str{math\_sin(30} and \str{math\_sin(30) + math} are in $\lang_p(G)$ as they can be extended to be syntactically valid. 
 By ensuring that at each intermediate step, \add{the invariant} that the LLM partial generation $\partialcode$ is in the set $\lang_p(G)$ is maintained, we can guarantee that upon completion of the generation process, $C_n$ will indeed be syntactically valid, i.e., $C_n \in \lang(G)$. 
This ensures that an intermediate output such as \str{math\_area} which is not in $\lang_p(G)$ is never generated by the model.

\subsection{\Tool{} Algorithm}
A key challenge in syntactic generation is token misalignment, where LLM tokens do not directly correspond to lexical tokens from the grammar.
The main reason for the high error rate in syntactic generation in prior works is the lack of formalization in their approaches (Section~\ref{sec:exp_main}).
Our work addresses this challenge by providing an algorithm that is provably sound — retains all syntactically valid tokens and is complete under specific conditions—rejecting every syntactically invalid token at every generation step.

Another significant challenge for efficiency is developing a novel algorithm that facilitates offline computations that minimize the overhead during inference.
\Tool{} tackles this challenge by creating a novel structure called the DFA mask store offline. 
For a given grammar $G$ and vocabulary $\vocab$, this mask store is constructed once and can be used across all generations. 
DFA mask store maps states of DFAs (corresponding to terminals in the grammar $G$) to boolean masks $m \in \{0, 1\}^{|\vocab|}$ over the vocabulary. 
This approach also benefits from parallelizing a substantial portion of the syntactical LLM generation computations by offloading them to a GPU during inference.

% \Tool{} addresses the syntactic decoding problem by creating a novel structure which we call \emph{DFA mask store} offline (Definition~\ref{def:lookup}). 
% The precomputed mask store allows more efficient computation of set $\vocab_k$ at $k$-th iteration of LLM generation such that the intermediate generation $\partialcode.t \in \lang_p(G)$ for any $t\in\vocab_k$.
\add{
Furthermore, it is challenging to ensure generality with efficiency.
Many prior works are restricted to syntactic generation with a specific type of decoding~\cite{scholak-etal-2021-picard, guidance}.
At $k$-th LLM iteration, for partial LLM output  $\partialcode$, let $\vocab_k \subseteq \vocab$ denotes the subset of vocabulary such that for any token $t\in\vocab_k$ the intermediate generation continues to maintain the invariant $\partialcode.t \in \lang_p(G)$.
Our formulation for computing $\vocab_k$ from $\vocab$ is highly general and can be integrated with any decoding algorithm, such as greedy, sampling, or beam-search. 
Any algorithm that could potentially be applied to $\vocab$ can instead be applied to $\vocab_k$.
The mask store allows more efficient computation of a subset of \mbox{tokens $\vocab_k$}.
}

\Tool{} works in two steps: 
first, it parses $\partialcode$ and computes the unparsed remainder $r \in \alphabets^*$ along with the acceptable terminal sequences $\accepts$ (formally defined in Section~\ref{sec:parse}). 
In the second step, \Tool{} utilizes $r$, $\accepts$, and the mask store. 
This step involves traversing the DFA and performing a few lookups within the DFA mask store to obtain a subset of tokens $\vocab_k$. 
\add{
In the following sections, we elaborate on these steps using our illustrative example.
}

\noindent {\bf Parsing Partial Output.}
\add{
\Tool{}'s parsing of partial output $\partialcode$ begins with lexing $\partialcode$. 
We assume our lexer has a 1-character lookahead and no backtracking. 
This assumption ensures that LLM's future generations do not alter the lexical types of any previous lexical tokens except for the final lexical token.
The remainder $r$ denotes the suffix of $\partialcode$ that may still change its lexical type in subsequent iterations.
We define two cases for assigning $r$:
}

\begin{itemize}[leftmargin=*]
    \item 
Case 1 is when $\partialcode$ contains an unlexed suffix $u$, and here we assign $r = u$. For example, $\partialcode =$\str{math\_sqrt(3) * (2.} is lexed as \str{math\_sqrt}, \str{(}, \str{3}, \str{)}, \str{*}, \str{(}, \str{2.}, where \str{math\_sqrt}, \str{(}, \str{3}, \str{)}, \str{*}, \str{(} are lexical tokens of type $\textit{math\_sqrt, lpar, int, rpar, mult, lpar}$, respectively. Here \str{2.} (\str{2} followed by a \str{.}) is unlexed suffix which we assign as the \mbox{remainder $r$.}
\item
Case 2 is when $\partialcode$ ends with a complete lexical token, where $r$ is assigned the value of the final lexical token. 
Hence, $\partialcode =$\str{math\_sqrt(3) * (2} is lexed as \str{math\_sqrt}, \str{(}, \str{3}, \str{)}, \str{*}, \str{(}, \str{2}.
Where \str{math\_sqrt}, \str{(}, \str{3}, \str{)}, \str{*}, \str{(} are lexical tokens of type $\textit{math\_sqrt, lpar, int, rpar, mult, lpar}$, respectively. 
Although \str{2} is the complete final lexical token with type $\textit{int}$, it is assigned as the remainder since in the subsequent iteration it may even change its lexical type to  $\textit{float}$.
\end{itemize}
\add{
In both cases, our lexer assumption ensures that the portion of $\partialcode$ excluding the remainder $r$ will retain its lexical tokenization in subsequent LLM iterations.
The assumption is crucial to enable incremental parsing and ensures that that the remainder $r$ is always small, both of which contribute to reducing time complexity.  
% Note that the parsing in \Tool{} is performed incrementally throughout LLM generation to ensure small overhead.
}

\noindent {\bf Accept Sequences.}
% This section may need a lot of polishing
% Define the partially parsed code
Given a sequence of lexical tokens $l_1, \dots l_f$, we use a \add{bottom-up LR} parser to compute what types of lexical tokens are acceptable next according to the grammar.
If at a certain point in the generation, we have lexical tokens \str{math\_sqrt}, \str{(}, \str{3}, \str{)}, \str{*}, \str{(}, \str{2.27} then the immediate next lexical token can be of type $\textit{rpar}$, $\textit{add}$ or $\textit{mult}$.
We define an accept sequence as a function of the parsed partial output (excluding the remainder) as a sequence of terminals such that those terminals can follow the currently parsed output (Definition~\ref{def:acc}).
For instance, in the case $\partialcode =$ \str{math\_sqrt(3) * (2.27}, $\{\textit{rpar}\}$, $\{\textit{add}\}$ and $\{\textit{mult}\}$ all are 1-length accept sequences.
$\{\textit{add}, \textit{int}\}$ and $\{\textit{add}, \textit{float}\}$ are some of the 2-length accept sequences for this example that can follow the current partial output.
\add{
In Section~\ref{sec:parse}, we show how we efficiently compute accept sequences of length 1 and 2 using an LR(1) parser, leveraging its immediate error detection property~\cite{10.1145/356628.356629}.
}
Further, we discuss how an LR($\kappa$) parser can be used to compute accept sequences of length $\kappa$ efficiently.
% But in practice, \Tool{} can work quite effectively using accept the sequences of smaller lengths and still guarantees the soundness of syntactical generation (Theorem~\ref{thm:sound}) and thus avoid the high memory requirement for $LR(k)$ parsers for large $k$. 
\add{
However, in practice, \Tool{} can effectively operate with shorter accept sequences while still ensuring the soundness of syntactical generation (see Theorem~\ref{thm:sound}), thereby avoiding the high memory needed for LR($\kappa$) parsers for large values of $\kappa$.
}

%%%%%%%%%%% DFA Mask store %%%%%%%%%%%%%%%%%
\noindent {\bf DFA Mask Store.}
\add{
\Tool{} parsing step partitions partial output $\partialcode$ into lexically fixed part $\fixpartialcode$ and remainder $r$.
The accept sequences $\accepts$ are computed using the parser state on parsing $\fixpartialcode$ and denote the terminals that can follow $\fixpartialcode$.
Thus the problem of obtaining subset $\vocab_k$ of tokens that will lead to syntactical continuation can be reduced to aligning accept sequence $\sequence \in \accepts$ with the string $r.t$ obtained by concatenating remainder $r$ and LLM token $t$ in the vocabulary.
One approach is to iterate through LLM vocabulary $\vocab$ and verify this alignment for each token $t$ individually.
However, this method is inefficient due to the need for matching $|V|$ tokens with $|\accepts|$ terminal sequences.
}
In \Tool{} algorithm, the precomputed DFA mask store is crucial for allowing efficient computation of acceptable tokens $\vocab_k$.
% \add{
% The mask store reduces this entire process to $|\accepts|$ lookups and tensor union operations. 
% }
Next, we show how the mask store maps the states of DFAs of the terminals and a sequence of terminals to masks over the vocabulary to enable this process. 

Given a remainder $r$ and any accept sequence $\sequence \in \accepts$, we want to check for a token $t \in \vocab$, if $r.t$ \add{aligns} or partially matches with $\sequence$.
We formally define this notion of partial match in Definition~\ref{def:pmatch}.
We establish a connection between the match of a terminal sequence and a string through the DFAs corresponding to the terminals. 

\begin{wrapfigure}{l}{.25\textwidth}
    % \vspace{-0.1in}
    \centering
    \includegraphics[width=0.25\textwidth]{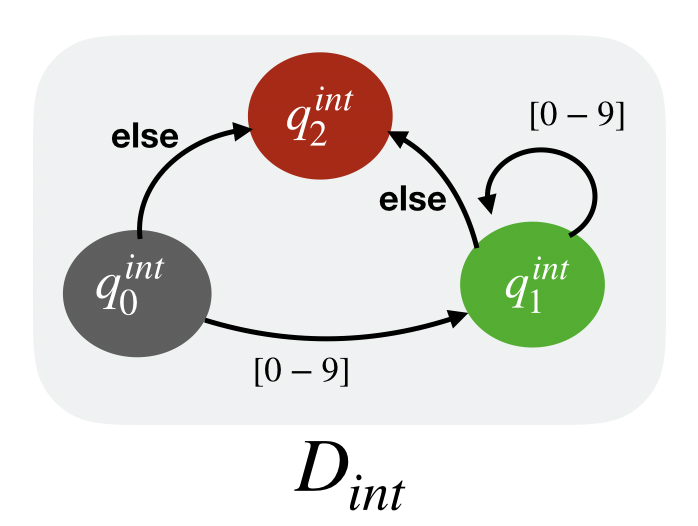}
    \caption{DFA for terminal $\textit{int}$.}
    \label{fig:overview_int}
     \vspace{-0.1in}
\end{wrapfigure}
Figure~\ref{fig:overview_int} presents a DFA for the terminal int.
In this DFA, $q_0^\textit{int}$ is the start state, and $q_1^\textit{int}$ is an accept state.
Further, we say that $q_0^\textit{int}, q_1^\textit{int}$ are \live{} states since there is a path from those states to an accept state and the state $q_2^\textit{int}$ is not a \live{} state. 

Consider the partial output $\partialcode =$ \str{math\_sqrt(3) * (2}. 
As described above, in this case, the output is split in the parsed part \str{math\_sqrt(3) * (} and the last lexical token \str{2} which is the remainder. 
$\{\textit{int}, \textit{add}\}$, $\{\textit{int}, \textit{rpar}\}$, $\{\textit{float}\}$ are some of the accept sequences.
For each of these accept sequences, we want to compute tokens $t \in \vocab$ such that appending \str{2} and $t$ i.e. \str{2}$.t$ partially matches the accept sequence.
% For instance, considering the accept sequence $\{\textit{float}, \textit{add}\}$, the tokens $t=$\str{123}, \str{4+} and \str{+} are all such that $r.t=$\str{2123}, \str{24+} and \str{2+} partially matches $\{\textit{int}, \textit{add}\}$. 
% We show that $\dmatch$ defined formally in 

\begin{figure}[b]
\centering
\includegraphics[width=9cm]{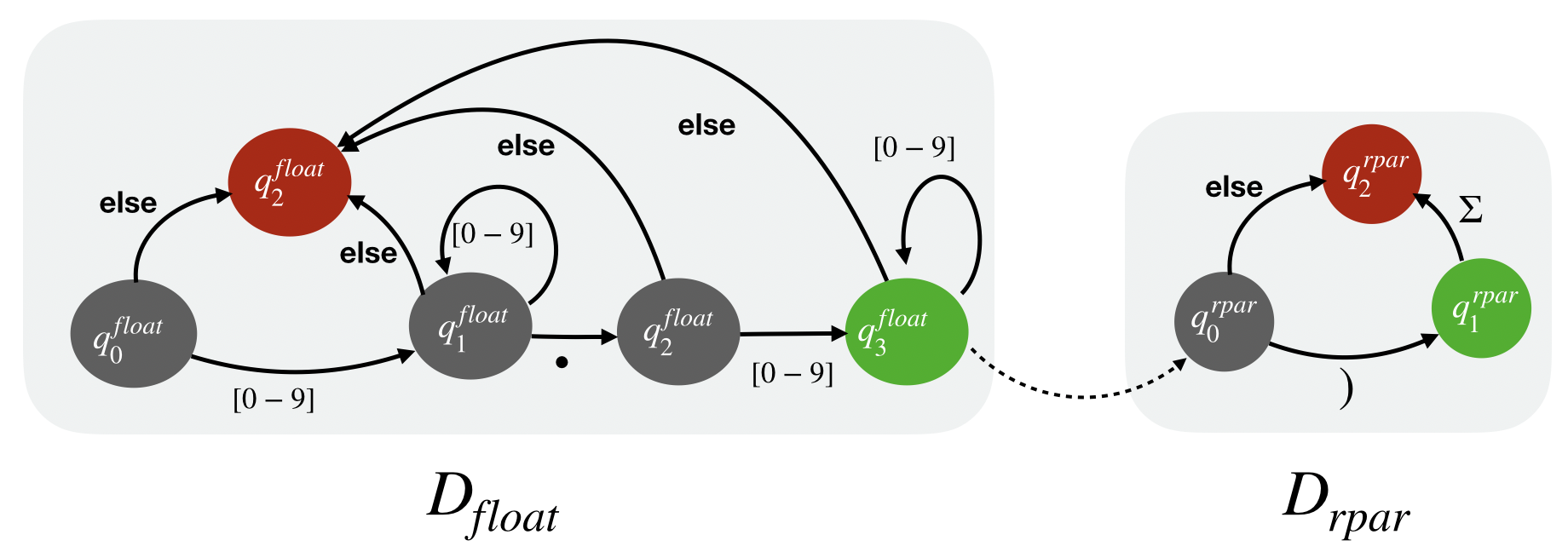}
\vspace{-.1in}
\caption{DFAs for accept sequence $\sequence = \{\textit{float}, \textit{rpar}\}$.}
\label{fig:overview_dmatch}
\vspace{-0.1in}
\vspace{-.1in}
\end{figure}

Consider an accept sequence $\sequence = \{\textit{float}, \textit{rpar}\}$. 
Figure~\ref{fig:overview_dmatch} displays the DFAs corresponding to the terminals in $\sequence$. 
If we begin from the initial state $q_0^\textit{float}$ of $D_\textit{float}$ and change the current DFA state according to the characters in $r$, in our example with $r =$\str{2}, the resulting state of the DFA is $q_1^\textit{float}$.
We observe that any token $t \in \vocab$ is acceptable if continuing the DFA walk from $q_1^\textit{float}$ ends on a live state.
We also allow a transition from the end state and start state of DFAs of subsequent terminals in the accept sequence as shown by the dotted arrow.
\add{
The partial match of $r.t$ and $\sequence$ can thus be equivalently checked by doing a walk over the DFAs.
}
Tokens such as \str{11}, \str{.}, \str{.1}, and \str{.27)} are some of the tokens where initiating a walk from $q_1^\textit{float}$ leads to reaching one of the live states. 
For example, by consuming \str{.27)}, we reach $q_1^\textit{rpar}$, which is a live state.
Consequently, \Tool{} approves \str{.27)} as a valid continuation from $\partialcode =$ \str{math\_sqrt(3) * (2}.

\add{
Our key insight for achieving efficiency is that for each DFA state, we can precompute LLM tokens that will lead to a transition to a live state starting from that state.
Precomputing these sets can significantly reduce the computation required during inference.
Further, these precomputed set of LLM tokens can be stored as boolean masks for efficiently combining them during inference.
}
Given a DFA state $q$ and any sequence terminals of length $\alpha$, the mask store maps $\dmap{\alpha}(q, \sequence) = m$, where $m \in \{0,1\}^{|\vocab|}$ is the mask over vocabulary.
During the inference time, for each accept sequence $\sequence \in \accepts$, we first consume $r$ and walk over the first DFA in the accept sequence.
We then use the map $\dmap{\alpha}$ on the current DFA state to get the \add{mask $m_\sequence$ of valid tokens for $\sequence$.
\add{
Hence, for each accept sequence $\sequence \in \accepts$, we require a walk over a DFA and a lookup in the mask store to obtain $m_\sequence$. 
}

Finally, we combine these masks \add{obtained} for each acccept sequence to get the masks of all syntactically valid tokens by computing their union $\bigcup_{\sequence \in \accepts} m_{\sequence}$.
In practice, these masks can be stored as tensors and can be combined efficiently using a small number of tensor union operations. 
We show in Theorem~\ref{thm:sound} that this combined mask overapproximates the set $\vocab_k$, ensuring the soundness of our approach.
Further, we show that for the LR parser with larger lookahead, our approach is complete and ensures the combined mask is exactly $\vocab_k$~(Theorem~\ref{thm:complete}).
}

\noindent {\bf Bringing It All Together.}
In our example, \Tool{} improves the LLM's output by guiding the generation. 
Initially, the LLM produces \str{math} as $C_1$. 
Next, \Tool{} excludes LLMs top choices such as \str{\_area}, \str{\_tri}, and \str{\_p} from the vocabulary, leading the decoding algorithm to select \str{\_sqrt}. 
Further, even in the 12th iteration where the LLM outputs $C_{11}=$\str{math\_sqrt(3)/4 * (2.27}, \Tool{} filters out the LLM's preferred choice \str{\^} from the vocabulary. 
Instead, the LLM opts for $*$, eventually generating $C_n =$ \str{math\_sqrt(3)/4 * (2.27) * (2.27)}, which is syntactically correct i.e. $C_n \in \lang(G)$ and also semantically accurate.

\subsection {\textbf{Time Complexity}}
At each decoding step in \Tool{}, the most resource-intensive tasks are computing accept sequences and generating the mask using $r$ and $\accepts$. 
In Section~\ref{sec:time}, we demonstrate that our implementation, leveraging LR(1) parsing, efficiently constructs 1 and 2-length accept sequences. 
We show that the complexity of \Tool{} at each decoding step is $O(T_\cup \cdot |\accepts|)$, where $T_\cup$ represents the time needed for boolean mask union operations. 
\add{
Typically, $|\accepts|$ is small (<10 on average in our experiments) and in the worst case, it equals the size of set of all terminals $|\allterminals|$ in the grammar.
For our largest Python grammar, $|\allterminals|$ is 94.
}
Modern hardware, especially with GPUs, can perform these vectorized union operations efficiently~\cite{paszke2019pytorch}, making the \Tool{} algorithm efficient in practice.

\section{\mbox{Syntactically Correct Generation}}
\label{sec:technical}
This section describes our main technical contributions and the \Tool{} algorithm. 
% First, we formally define the syntactical decoding problem (Section~\ref{sec:syndecode}). 
% Next, we describe our parsing algorithm (Section~\ref{sec:incparse}). 
% Then we explain our CFG-based masking technique (Section~\ref{sec:mask}) and prove the soundness and completeness of our algorithm (Section~\ref{sec:prop}).
% We finally summarize the \Tool{} implementation (Section~\ref{sec:implementation}) and analyze its time complexity (Section~\ref{sec:time}).  

\subsection{Syntactical Decoding Problem}
\label{sec:syndecode} 

Given a language with grammar $G$, let $\lang(G) \subseteq \alphabets^*$ denote the set of all syntactically valid outputs according to the grammar $G$. 
For a grammar $G$, $\lang_p(G)$ represents the set of all syntactically valid partial outputs. 
If a string $w_1$ belongs to $\lang_p(G)$, then there exists another string $w_2$ such that appending $w_2$ to $w_1$ results in a string that is in the language defined by $G$. Formally,

\begin{definition}[Partial Outputs]
\label{def:gramdecode}
For grammar $G$, $\lang_p(G) \subseteq \alphabets^*$ denotes all syntactically valid partial outputs. Formally, if $w_1 \in \lang_p(G)$ then $\exists w_2 \in \alphabets^*$ such that $w_1.w_2 \in \lang(G)$
\end{definition}

\noindent For a grammar $G$ and a partial output $\partialcode$ belonging to the set of prefix strings $\lang_p(G)$, the syntactical decoding problem aims to determine the set $\vocab_k$ of valid tokens from a finite vocabulary $\vocab$ such that appending any token $t \in \vocab_k$ to $\partialcode$ maintains its syntactic validity according to the grammar $G$. 

\begin{definition}[Syntactical Decoding]
\label{def:gramdecode}
For grammar $G$, given partial output $\partialcode \in \lang_p(G)$ and finite token vocabulary $\vocab \subset \alphabets^*$, the syntactical decoding problem is to compute the set $\vocab_k \subseteq \vocab$ such that for any $t \in \vocab_k, \partialcode.t \in \lang_p(G)$    
\end{definition}

% Synchromesh~\cite{poesia2022synchromesh} solves this problem by iterating over the $\vocab$ and for each $t \in \vocab$ it parses $\partialcode.t$ to check if it is in $\lang_p(G)$. 
% Typically $|V|$ is large ($>30,000$) and thus Synchromesh does a preorder traversal over a trie built on $\vocab$ and optimizes this step.
% Despite this optimization, Synchromesh still needs to evaluate a considerable number of candidate tokens, introducing a large overhead.
% As described earlier, prior works~\cite{poesia2022synchromesh, llamacpp} solve this problem by iterating over the $\vocab$ and for each $t \in \vocab$ it parses $\partialcode.t$ to check if it is in $\lang_p(G)$

% \Tool{} solves this problem through the creation of a novel structure which we call \emph{DFA mask store} offline (Definition~\ref{def:lookup}). 
% For a given grammar $G$ and vocabulary $\vocab$, this mask store is constructed once and can be leveraged across all generations. 
% It efficiently stores masks over the vocabulary.
% \Tool{} breaks down the syntactical decoding problem into two distinct steps.
\mbox{We next present \Tool{}'s key aspects to solve this problem:}
\begin{itemize}[leftmargin=*]
    \item In the initial step, it parses $\partialcode$ and computes the unparsed remainder $r \in \alphabets^*$ along with the acceptable terminal sequences $\accepts$ (Section~\ref{sec:parse}).
    \item In the second step, \Tool{} utilizes $r$, $\accepts$, and the precomputed mask store. This phase involves traversing the DFA and performing a few lookups within the DFA mask store to obtain the set of syntactically valid tokens $t$ capable of extending $\partialcode$ (Section~\ref{sec:mask}). 
    \item Consequently, \Tool{} efficiently computes the set of syntactically valid tokens. We show the soundness and completeness of our approach in Section~\ref{sec:prop}.
    \item We further discuss the theoretical complexity of \Tool{} in Section~\ref{sec:time} and the \Tool{} framework in Section~\ref{sec:framework}.
\end{itemize}

\subsection{Parsing Partial Output}
\label{sec:parse}
In this section, we describe the remainder $r$ and accept sequences $\accepts$ returned by the parsing step.

% remainder
\noindent{\bf Remainder.}
\Tool{} uses a lexer to convert $\partialcode$ to sequence of lexical tokens $l_1, l_2 \dots l_f \in \alphabets^*$.  
Each lexical token $l_i$ is associated with a terminal type $\terminal_i$, where $l_i \in \lang(\regex_{\terminal_i})$ ($\regex_{\terminal_i}$ is the regular expression for terminal $\terminal_i$).
\add{
We assume our lexer uses a 1-character lookahead without backtracking. 
This ensures that the lexical types of previous tokens in $\partialcode$ remain unchanged, except for the final token. 
The remainder $r$ represents the suffix of $\partialcode$ that could potentially change its lexical type in future iterations.
Thus the remainder $r$ is assigned such that it is either unlexed because it does not match any terminal, or has been lexed but might undergo a different lexing in subsequent iterations when $\partialcode$ is extended by the LLM by appending tokens.
This assumption is crucial for enabling incremental parsing and ensures that the remainder $r$ remains small, which contributes to reducing overall time complexity.
}
\Tool{} assigns the remainder according to the following two cases:

\begin{description}
    \itemsep0em
    \item \textbf{Case 1: $\partialcode = l_1.l_2 \dots l_f$} 
    Assuming a standard lexer with 1-character lookahead and no backtracking, all lexical tokens $l_1, l_2, \dots, l_{f-1}$ remain unchanged upon extending $\partialcode$. However, the final lexical token $l_f$ may change. 
    For example, in Python partial output in the $k$-th LLM iteration, if the final lexical token is $l_f=$\str{ret} and the language model generates the token \str{urn} in the next iteration, the updated code results in the final lexical token becoming $l_f=$\str{return}. 
    This transition reflects a transformation from an identifier name to a Python keyword in the subsequent iterations. 
    Thus, $r$ is assigned the value $l_f$, i.e., $r=$\str{ret} for k-th iteration in our example.
    
    \item \textbf{Case 2: $\partialcode = l_1.l_2 \dots l_f.u$:} 
    Here, $u \in \alphabets^*$ is the unlexed remainder of $\partialcode$. 
    In this case, considering the 1-character lookahead of the lexer, the types of $l_1, l_2, \dots, l_{f}$ do not change upon extending $\partialcode$. 
    Consequently, $r$ is assigned value $u$ of the suffix that remains unlexed.
\end{description}

\noindent 
\add{
\Tool{} parsing step partitions partial output $\partialcode$ into lexically fixed part $\fixpartialcode$ and remainder $r$.}
Given a sequence $\sequence = \terminal_0, \terminal_1, \dots, \terminal_f$, we simplify notation by using $\lang(\sequence) = \lang(\regex_{\terminal_0} \cdot \regex_{\terminal_1} \dots \regex_{\terminal_f})$ throughout the rest of the paper. 

% Partial parse
\begin{definition}[Partial Parse]
\label{def:pparse}
Given the partial output $\partialcode \in \alphabets^*$, the partial parse function $\partialparse: \alphabets^* \to \allterminals^* \times \alphabets^*$ returns a terminal sequence $\sequence^{\square}$ and remainder $r$ such that $\partialcode = \fixpartialcode.r$ and $\fixpartialcode$ is parsed as $\sequence^{\square}$. i.e. $\fixpartialcode \in \lang(\sequence^{\square})$. 
\end{definition}

% accept sequences
\noindent{\bf Accept Sequences.}
A sentence is a sequence of terminals. 
A grammar $G$ describes a (possibly infinite) set of sentences, that can be derived by using the production rules of the grammar. 
We use $\lang^\allterminals(G) \subseteq \allterminals^*$ to denote the valid sequences of terminals that can be derived from the rules of $G$.
Further, $\lang^\allterminals_p(G)$ denotes all syntactically valid partial sentences of terminals.
Formally, 

\begin{definition}[Partial Sentences]
\label{def:psentence}
We define a set of all syntactically valid partial sentences 
 $\lang^\allterminals_p(G) \subseteq \allterminals^*$ such that $\sequence \in \lang^\allterminals_p(G)$ if and only if $\exists \sequence_1 \in \allterminals^*$ such that $\sequence.\sequence_1 \in \lang^\allterminals(G)$.
\end{definition}
Note that $\lang(G)$ and $\lang_p(G)$ are defined over alphabet $\alphabets$, whereas $\lang^\allterminals(G)$ and $\lang^\allterminals_p(G)$ over terminals $\allterminals$.
Nevertheless, if a program $C$ is parsed to obtain terminal sequence $\sequence$, then $C \in \lang(G)$ is equivalent to $\sequence \in \lang^\allterminals(G)$.
The \Tool{} parsing algorithm obtains $\sequence^{\square} = \terminal_1, \terminal_2 \dots \terminal_f$ by parsing $\partialcode$ \add{corresponding to the parserd part of partial output $\fixpartialcode$.} 
Given a partial sentence $\sequence_{\square}$, an accept sequence is a sequence over $\allterminals$ such that when appended to $\sequence^{\square}$ the result is still a partial sentence.

\begin{definition} [Accept Sequence]
\label{def:acc}
Given partial output $\partialcode \in \lang_p(G)$, and $\sequence^{\square}, r = \partialparse(\partialcode)$,  $\sequence_1 \in \allterminals^*$ is an accept sequence if $\sequence^{\square}.\sequence_1 \in \lang_p^\allterminals(G)$.
\end{definition}
Consider a Python partial program $\partialcode =$ \str{def is} and let $\textit{def}, \textit{name}, \textit{lpar}$ and $ \textit{rpar}$ be the terminals in Python grammar. 
we get $\{\textit{def}\},$\str{is} $=\partialparse($\str{def is}$)$, where $\sequence^{\square}=\{\textit{def}\}$ and $r=$\str{is}.
$\sequence_1 = \{\textit{name}, \textit{lpar}, \textit{rpar}\}$ is an accept sequence in this case as the sequence of terminals $\sequence^{\square}.\sequence_1 = \{\textit{def}, \textit{name}, \textit{lpar}, \textit{rpar}\}$ is a valid partial sentence.
The parser state on parsing the partial output $\partialcode$ can be utilized to compute a set of accept sequences denoted as $\accepts$.
The soundness and completeness of the \Tool{} algorithm depend on the length of these accept sequences in $\accepts$.
In theory, using longer accept sequences enhances the precision of the \Tool{} algorithm at the cost of increased computational complexity. 
In Section~\ref{sec:implementation}, we show our method for obtaining 1 and 2-length accept sequences that are efficient and precise in practice.

\subsection{Grammar Mask}
\label{sec:mask}
This section outlines the utilization of the set of acceptable terminal sequences $\accepts$ and the remainder $r$ in the creation of a boolean mask using the DFA mask store which is subsequently used for constraining the LLM output. 
The DFA mask store is constructed offline and makes \Tool{} efficient during the LLM generation. 
Given partial output $\partialcode$, our objective is to identify tokens $t \in \vocab$ such that appending them to $\partialcode$ leads to syntactical completion. 
% We approach this problem by utilizing the remainder $r$ and sequences $\accepts$. 
Given remainder $r$ and set of sequences $\accepts$, the goal is to determine whether $r.t$ partially matches the regular expression derived from any of the sequences in $\accepts$.
To characterize the notion of strings partially matching a regular expression, we next introduce the function $\pmatch$.

\begin{definition}[\pmatch]
\label{def:pmatch}
The function $\pmatch$ takes a word $w \in \alphabets^*$, a regular expression $\regex$ and returns a boolean. $\pmatch(w, \regex) = \true$ if either of the following conditions holds:
\begin{enumerate}
    \item $\exists w_1 \in \alphabets^*, w_2 \in \alphabets^+$ such that $w = w_1.w_2 $ and $w_1 \in \lang(\regex)$ or
    \item $\exists w_1 \in \alphabets^*$ such that $w.w_1 \in \lang(\regex)$
\end{enumerate}
\end{definition}

\noindent Thus $\pmatch(w, \regex)$ is true when either a prefix of $w$ matches $\regex$ or $w$ can be extended to match $\regex$.
The consequence of allowing $\pmatch$ to be defined such that it is true even when prefix matches, is that \Tool{} will conservatively accept all tokens for which the prefix matches the accept sequence.
Hence, we overapproximate the precise set of syntactically valid tokens.
We make this choice to ensure that \Tool{} is sound for any length of accept sequences.
Next, we give definitions related to DFAs. 
These definitions are useful for describing the construction of the DFA mask store and proving properties related to its correctness in the \Tool{} algorithm. 
In particular, we first define the live states of DFA. 
We say state $q$ is live if there is a path from $q$ to any final states in $\dfafinal$. Formally,

\begin{definition} [DFA $\live$ states]
\label{def:live}
% Define live states
Given a DFA $\dfa(\dfastates, \alphabets, \transitions, \dfastart, \dfafinal)$, let \mbox{$\live(\dfastates) \subseteq \dfastates$} denote the set of live states such that  
\[
    q \in \live(\dfastates) \text{ iff } \exists w \in \alphabets^* \text{ s.t. } \compute(w, q) \in \dfafinal
\]
\end{definition}

% Define and describe the mask store
\noindent We use $\dfa_\terminal(\dfastates_\terminal, \alphabets_\terminal, \transitions_\terminal, q_0^{\terminal}, \dfafinal_\terminal)$ to denote a DFA corresponding to a terminal $\terminal \in \allterminals$. 
Next, we establish the definition of $\dmatch$ for DFA, which is an equivalent concept to $\pmatch$ with regular expressions.
$\dmatch$ is recursively defined such that its computation can be performed by walking over the DFAs of a sequence of terminals.

% Define valid token at a state
\begin{definition} [\dmatch]
\label{def:dmatch}
Given a DFA $\dfa(\dfastates, \alphabets, \transitions, \dfastart, \dfafinal)$, a string $w \in \alphabets^*$, a DFA state $q \in Q$ and any sequence of terminals $\sequence= \{\terminal_{f+1}, \terminal_{f+2} \dots \terminal_{f+d}\}$, $\dmatch(w, q, \sequence) = \true$, if either of the following conditions hold:
\begin{enumerate}
\item $\compute(w, q) \in \live(Q)$ or
\item $\exists w_1 \in \alphabets^*, w_2 \in \alphabets^+$ such that $w_1.w_2 = w$, $\compute(w_1, q) \in F \text{ and } \sequence= \{\} $ or
\item $\exists w_1 \in \alphabets^*, w_2 \in \alphabets^*$ such that $w_1.w_2 = w$, $\compute(w_1, q) \in F$, \\ 
and $\text{\dmatch}(w_2, q_{0}^{\terminal_{f+1}}, \{\terminal_{f+2} \dots \terminal_{f+d}\}) = \true$ where $q_{0}^{\terminal_{f+1}}$ is the start state corresponding to the DFA for $\terminal_{f+1}$
% $w_2$ is valid with respect to $q_{0, \terminal_1}$ and the follow sequence $\{\terminal_2 \dots \terminal_\alpha\}$
\end{enumerate}
\end{definition}

\sloppypar
\noindent Given an accept sequence $\sequence = \{\terminal_{f+1}, \terminal_{f+2} \dots \terminal_{f+d}\} \in \accepts$, our objective is to compute the set of tokens $t \in \vocab$ such that $\pmatch(r.t, \regex_\sequence)$ holds, where $\regex_\sequence = (\regex_{f+1}. \regex_{f+2}. \ldots.\regex_{f+d})$ is the regular expression obtained by concatenating regular expressions for terminals. 
If $\sequence^p$ denotes the sequence $\{\terminal_{f+2}, \dots \terminal_{f+d}\}$, Lemma~\ref{lemma:eq} simplifies this problem to finding $\dmatch(r.t, \dfastart^{\terminal_1}, \sequence^p)$. 
Furthermore, utilizing Lemma~\ref{lemma:dmatch}, this can be further reduced to computing $q = \compute_{\terminal_1}(r, \dfastart^{\terminal_1})$ and $\dmatch(t, q, \sequence^p)$. 
It's important to note that $\dmatch(t, q, \sequence^p)$ does not depend on $\partialcode$ and can be computed offline. 
While the computation of $q$ for $\dmatch(t, q, \sequence^p)$ is relatively inexpensive, evaluating $\dmatch(t, q, \sequence^p)$ can be computationally expensive both offline and online, as it requires considering numerous potential accept sequences offline, and where it needs to iterate over all tokens in $\vocab$ online.
We observe that if we consider sequences of smaller lengths, we can efficiently precompute the set of tokens satisfying $\dmatch(t, q, \sequence^p)$ for all $q, t$ and $\sequence^p$ offline.
We later establish the soundness of \Tool{} when using accept sequences of length at least $1$ (Theorem~\ref{thm:sound}) and completeness for accept sequences of the length greater than maximum length of tokens in the vocabulary (Theorem~\ref{thm:complete}). 
Typically, LLM tokens are small in size, allowing us to obtain these guarantees.
% Building upon these definitions, we establish a crucial lemma that draws a connection between $\pmatch$ and $\dmatch$.

\begin{restatable}{lemma}{eq}
\label{lemma:eq}
Given $\sequence = \{\terminal_{f+1}, \terminal_{f+2} \dots \terminal_{f+d}\}$,  $\sequence^p = \{\terminal_{f+2} \dots \terminal_{f+d}\}$ and $\regex_\sequence = (\regex_{f+1}, \regex_{f+2}, \ldots, \regex_{f+d})$, $\dmatch(w, \dfastart^{\terminal_1}, \sequence^p) \iff \pmatch(w, \regex_\sequence)$.
\end{restatable}
\begin{restatable}{lemma}{dm}
\label{lemma:dmatch}
If $q = \compute_{\terminal}(r, \dfastart^{\terminal})$ and no prefix of $r$ is in $\lang(\terminal)$ i.e. $\nexists w_1 \in \alphabets^*, w_2 \in \alphabets^* \text{ such that } w_1.w_2 = r \text{ and } \compute_\terminal(w_1, \dfastart^{\terminal}) \in F_{\terminal}  $ then $\dmatch(t, q, \sequence) \iff \dmatch(r.t, \dfastart^{\terminal}, \sequence)$.
\end{restatable}
The proofs of both the lemmas are in Appendix~\ref{sec:proofs}.

\begin{figure}[b]
\centering
\includegraphics[width=10cm]{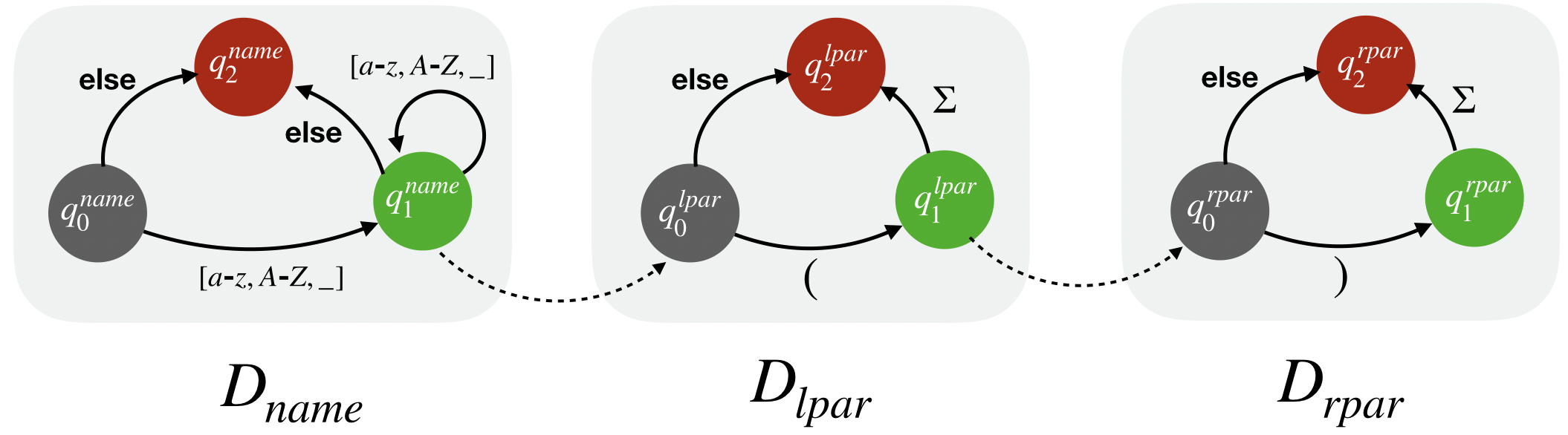}
\vspace{-.1in}
\caption{DFAs in accept sequence $\sequence = \{\textit{name}, \textit{lpar}, \textit{rpar}\}$ for example. 
The start state, final states, and dead states are in gray, green, and red respectively.
The dashed arrows link the final states of one DFA to the starting state of the next DFA, adhering to condition 3 in Definition~\ref{def:dmatch}. This illustrates the sequential traversal across DFAs during the computation of \dmatch.
} 
\label{fig:dfa}
\vspace{-.1in}
\end{figure}

\noindent \textbf{Illustrative Example:} 
Consider the scenario with $\partialcode =$ \str{def is}, $r=$\str{is}, and an accept sequence $\sequence = \{\textit{name}, \textit{lpar}, \textit{rpar}\}$ in $\accepts$, where $\textit{name}$, $\textit{lpar}$, and $\textit{rpar}$ are terminals in $\allterminals$. 
Our objective is to determine all $t \in \vocab$ such that \str{def is}.t forms a valid partial program. 
This can be achieved by finding tokens $t$ that satisfy $\pmatch(\text{\str{is}}.t, \regex_\sequence)$, where $\regex_\sequence = [a\text{-}z,A\text{-}Z,\_]^*()$.
Let's consider a token $t = \text{\str{\_prime():}}$. We observe that $r.t=$\str{is\_prime():} can be decomposed into \str{is\_prime} ($\textit{name}$), \str{(} ($\textit{lpar}$), \str{)} ($\textit{rpar}$), and \str{:}. 
Consequently, it partially matches $\regex_\sequence$ as defined by $\pmatch$. 
In Figure~\ref{fig:dfa}, we present the DFAs for $\sequence$ used in computing $\dmatch$. 
The reduction $\dmatch(r.t, \dfastart^\textit{name}, {\textit{lpar}, \textit{rpar}}) = \dmatch(\text{\str{is\_prime():}}, \dfastart^\textit{name}, {\textit{lpar}, \textit{rpar}})$ simplifies successively to $\dmatch(\text{\str{():}}, \dfastart^\textit{lpar}, {\textit{rpar}})$, then to $\dmatch(\text{\str{):}}, \dfastart^\textit{rpar}, {})$, and finally to $\dmatch(\text{\str{:}}, q_1^\textit{rpar}, {})$.
As $q_1^\textit{rpar}$ is a final state, according to condition 2 of Definition~\ref{def:dmatch}, $\dmatch(\text{\str{:}}, q_1^\textit{rpar}, {})$ holds true.
% This equivalence arises as we traverse the $\dfa_\textit{name}$ first, consuming \str{is\_prime} and reaching a final state in $F_\textit{name}$. 
Next, we define a mask over vocabulary

\begin{definition}[Vocabulary mask]
\label{def:mask}
Given vocabulary $\vocab \subseteq \alphabets^*$, $m \in \{0, 1\}^{|\vocab|}$ is a mask over the vocabulary. We also use $\set(m) \subseteq \vocab$ to denote the subset represented by $m$.
\end{definition}

\noindent {\bf DFA Mask Store}
% Discuss \alpha
\noindent For an integer $\alpha$, we define a DFA table $\dmap{\alpha}$ as the mask store over the DFA states with $\alpha$ lookahead. 
Given the set of all DFA states $\dfastates_\Omega = \bigcup_{\terminal \in \allterminals} \dfastates_\terminal$, the table stores binary masks of size $|\vocab|$, indicating for token string $t$, for any DFA state $q \in \dfastates_\Omega$ and a sequence of $\alpha$ terminals $\sequence_\alpha$ if $\dmatch(t, q, \sequence_\alpha) = \true$. 
The lookahead parameter $\alpha$ signifies the number of subsequent terminals considered when generating the mask stored in the table.
Choosing a larger value for $\alpha$ enhances the precision of \Tool{} algorithm, but it comes at the cost of computing and storing a larger table. 
We next formally define the DFA mask store,
\begin{definition}[DFA mask store]
\label{def:lookup}
For an integer $\alpha$, the DFA mask store $\dmap{\alpha}$ is a function defined as $\dmap{\alpha}: \dfastates_\Omega \times \allterminals^{\alpha} \to \{0, 1\}^{|\vocab|}$, where $\dfastates_\Omega = \bigcup_{\terminal \in \allterminals} \dfastates_\terminal$ represents the set of all DFA states and $\allterminals^{\alpha}$ is a set of $\alpha$-length terminal sequences. 
% Let $\textit{id}(t)$ denote the enumerated index of a token $t \in \vocab$. 
Then $\dmap{\alpha}(q, \sequence) = m$ is a binary mask such that $t \in \set(m)$ if $\dmatch(t, q, \sequence)$
\end{definition}

For our illustrative example if $m = \dmap{2}(q_1^\textit{name}, \{\textit{lpar}, \textit{rpar}\})$ then $t=$\str{\_prime():} should be contained in $\set(m)$.
The grammar mask for a set of accept sequences $\accepts$ can be computed by combining masks for each $\sequence \in \accepts$. 
% We describe the grammar mask computation algorithm in Appendix~\ref{sec:compmask}.
The DFA mask store $\dmap{0}$ maps each DFA state to all tokens such that they  $\pmatch$ without considering any following accept sequence (0-length sequence). 
In this case, the table maps each state with a single mask denoting the tokens that match the regular expression of the corresponding DFA. 

\noindent {\bf Computing Grammar Mask}
\label{sec:compmask}
\begin{wrapfigure}{R}{0.53\textwidth}
\vspace{-.34in}
\begin{minipage}{0.55\textwidth}

\begin{algorithm}[H]
\caption{Computing Grammar Mask}
\label{alg:grammar}
\textbf{Inputs:} $\accepts$: set of accept sequences, $r$: remainder
\begin{algorithmic}[1]
\Function{GrammarMask}{$\accepts, r$}
\State $m \gets \{\}$
\For{$\sequence \in \accepts$}
\State $\terminal_1 \gets \sequence[0]$
\State $q_r \gets \compute(q_0^{\terminal_1}, r)$
% \For{$q \in \dfastates_{\terminal_1}$}
\If{$q_r \in \live(\dfastates_{\terminal_1})$}
\State $\Pi \gets \textit{len}(\sequence)-1$
\State $m \gets m \cup \big( \dmap{\Pi}(q_r, \sequence[1:])\big)$
\EndIf
% \EndFor
\EndFor
\State \Return $m$
\EndFunction
\end{algorithmic}
\end{algorithm}

\end{minipage}
\vspace{-0.2in}
\end{wrapfigure}
% Example \dmap{0}
% This approach is equivalent to the one used by \cite{willard2023efficient} for regular expression guided generation. 
% The current parsers can easily compute acceptable sequences of terminals with a length of 2 from partial output. 
% We note that $\pmatch$ $r.t$ with a 2-length sequence is equivalent to $\dmatch$ with a 1-length sequence, as stated in Lemma~\ref{lemma:eq}. 
% Consequently, in our work, we opt for $\dmap{0}$ and $\dmap{1}$ since we have observed empirically that this combination is sufficient for producing syntactically valid outputs.
%
The mask store is constructed offline by enumerating all DFA states $\dfastates_\Omega $, considering all possible terminals in $\allterminals$, and all tokens in $\vocab$. 
The DFA mask store depends on the set of terminals $\allterminals$ and the model's vocabulary $\vocab$. 
As a result, a unique mask store is created for each grammar and tokenizer combination, and to enhance efficiency, we cache and reuse this table for future inferences. 

% \noindent{\bf Computing Grammar Mask.}
Algorithm~\ref{alg:grammar} presents our approach for computing the grammar mask during LLM generation.
It computes a grammar mask based on the sets of current accept sequences $\accepts$, and the remainder string ($r$). 
It iterates over $\accepts$, considering each sequence $\sequence$.
The algorithm initializes an empty mask $m$. 
It iterates over each acceptable sequence, considering the first terminal $\terminal_1$ in each. 
It computes the resulting state $q_r$ by processing $\terminal_1$ from an initial state $q_0^{\terminal_1}$ and the remainder string $r$. 
If $q_r$ is in a live state, the algorithm updates the grammar mask by unifying the mask cached in $\dmap{\alpha}$. 

\subsection{Soundness and Completeness}
\label{sec:prop}
This section establishes the soundness and completeness of the \Tool{} algorithm.
Algorithm~\ref{alg:main} presents the LLM generation algorithm with \Tool{}. It takes as inputs an LLM represented by $M$, a tokenizer denoted by $\tokenizer$, an input prompt string $C_0$, the maximum number of generated tokens $n_\textit{max}$, and a base decoding strategy $D$. The algorithm begins by tokenizing the input prompt using the tokenizer. 
It then iteratively generates tokens using the LLM, decodes the current token sequence, and performs parsing to obtain acceptable terminal sequences $\accepts$, and a remainder $\remainder$ (line~\ref{line:acc}). 
A grammar mask is applied to the logit scores based on these values (line~\ref{line:gm}). 
The algorithm subsequently selects the next token using the decoding strategy, and if the end-of-sequence token (EOS) is encountered, the process terminates. 
The final decoded output is obtained, incorporating the generated tokens, and is returned as the result of the MaskedGenerate algorithm.

Given partial output $\partialcode \in \lang_p(G)$, \Tool{} generates a corresponding mask $m$. 
If, for a token $t \in \vocab$, the concatenation $\partialcode.t$ results in a syntactically valid partial output, i.e. $\partialcode.t \in \lang_p(G)$, our soundness theorem ensures that $t$ is indeed a member of the set defined by the generated mask $m$. 
The subsequent theorem formally states this soundness property.

\begin{restatable}{theorem}{Sound}
\label{thm:sound}
Let $\partialcode \in \lang_p(G)$ be the partial output and any integer $d \geq 1$, let $\accepts_d \subseteq \allterminals^{d}$ contain all possible accept terminal sequences of length $d$ and $r \in \alphabets^*$ denote the remainder. 
If $m = \text{GrammarMask}(\accepts, r)$ then for any $t \in \vocab$, if $\partialcode . t \in \lang_p(G)$ then $t \in \set(m)$
\end{restatable}
The proof of the theorem is in Appendix~\ref{sec:proofs}.

Next, we give a definition that establishes a partial order on sets of terminal sequences, where given two sets $\accepts_1$ and $\accepts_2$, we say sets $\accepts_1 \greater \accepts_2$ if every sequence in $\accepts_2$ has a prefix in $\accepts_1$.

\begin{definition}[$\greater$]
We define a partial order $\greater$ on set of terminal sequences $\mathcal{P}(\allterminals^*)$ such that $\accepts_1 \greater \accepts_2$ when $\forall \sequence_2 \in \accepts_2 \exists \sequence_1 \in \accepts_1 \exists \sequence_3 \in \allterminals^*  $ s.t. $\sequence_2 = \sequence_1.\sequence_3$
\end{definition}

\begin{wrapfigure}{R}{0.5\textwidth}
\begin{minipage}{0.5\textwidth}
% \vspace{-.55in}

\begin{algorithm}[H] 
\small
\caption{\Tool{} Generation}
\label{alg:main}
% %
\textbf{Inputs:} $M$: LLM, $\tokenizer$: tokenizer, $C_0$: input prompt, $n_\textit{max}$: maximum generated tokens, $D$: decoding strategy
\begin{algorithmic}[1]
\Function{MaskedGenerate}{$M$, $\tokenizer$, $C_0$, $n_\textit{max}$, $D$}
\State $\curtokens \gets \text{Tokenize}(\tokenizer, C_0)$
\For{$i \in \{1, \dots n_\textit{max} \}$}
\State $\textit{scores} \gets M(\curtokens)$
\State $C_k \gets \text{decode}(\tokenizer, \curtokens)$
\State $\accepts, r \gets \text{Parse}(C_k)$
\label{line:acc}
\State $m \gets \text{GrammarMask}(\accepts, r)$
\label{line:gm}
\State $\textit{scores} \gets m \odot \textit{scores}$
\State $t_i \gets D(\textit{scores})$
\If{$t_i = EOS$}
\State break
\EndIf
\State $T_\textit{cur} \gets \text{append}(T_\textit{cur}, t_i)$
% \State $\curtokens \gets \text{append}(\curtokens, t_i)$
\EndFor
\State $\text{output} \gets \text{decode}(\tokenizer, \curtokens)$
\State \Return output
\EndFunction
\end{algorithmic}
\end{algorithm}
\vspace{-0.4in}

\end{minipage}
\end{wrapfigure}
We further state the lemma that shows the relation in the grammar masks generated by two accept sequences satisfying relation $\greater$.

\begin{restatable}{lemma}{Porder}
\label{lemma:porder}
Given $\accepts_1$ and $\accepts_2$ are set of accept sequences such that $\accepts_1 \greater \accepts_2$ and $m_1 = \text{GrammarMask}(\accepts_1, r)$ and $m_2 = \text{GrammarMask}(\accepts_2, r)$ then $\set(m_2) \subseteq \set(m_1)$
\end{restatable}
The proof of the lemma is in Appendix~\ref{sec:proofs}.

Theorem~\ref{thm:sound} proves soundness for accept sequences $\accepts_d$ of length $d$, while Lemma~\ref{lemma:porder} extends this proof to any set of accept sequences $\accepts$ where $\accepts \greater \accepts_d$. 
Our implementation, employing sequences of varying lengths, can be proven sound based on this extension.
% Say motivation and pretext for the completeness

The completeness theorem ensures that, under specified conditions, each token $t \in \set(m)$ guarantees $\partialcode.t$ as a syntactically valid partial output.
An implementation of \Tool{} with a short length of accept sequences although sound, may not guarantee completeness.  
To illustrate, let's take the example where $\sequence = \terminal_{f+1}, \terminal_{f+2} \in \accepts$ with simple singleton regular expressions $\regex_{\terminal_{f+1}} =$ \str{(} and $\regex_{\terminal_{f+2}} =$ \str{(}. 
In this case, our algorithm conservatively treats all tokens $t \in \vocab$ as syntactically valid, whenever \str{((} is a prefix of those tokens (e.g., \str{(((}, \str{(()})) even though some tokens may not meet syntactic validity. 
However, by assuming that the accept sequences are long enough, we can establish the completeness of the approach. 
\begin{restatable}{theorem}{Complete}
\label{thm:complete}
Let $\partialcode \in \lang_p(G)$ be the partial output, let $\accepts_d \subseteq \allterminals^{d}$ contain all possible accept terminal sequences of length $d$ and $r \in \alphabets^*$ denote the remainder. 
Suppose for any $t \in \vocab, d > \textit{len}(t)$ and $m = \text{GrammarMask}(\accepts_d, r)$ such that $t \in \set(m)$ then $\partialcode.t \in \lang_p(G)$
\end{restatable}

The proof of the theorem is in Appendix~\ref{sec:proofs}. 
While our completeness theorem ensures the \Tool{} consistently extends syntactically correct partial outputs, it does not guarantee termination with a correct and complete output. 
The focus of the theorem is on generating syntactically valid partial outputs, and the theorem does not address whether the process converges to a syntactically correct whole output. 
Termination considerations go beyond the completeness theorem's scope.

\subsection{\Tool{} Implementation}
\label{sec:implementation}
\add{
\noindent \textbf{Base LR parser: }
Bottom-up LR parsers, including LR(1) and LALR(1) parsers, process terminals generated from the lexical analysis of the code sequentially and perform shift or reduce operations~\cite{aho86}. 
LR($\kappa$) parsers have the immediate error detection property, ensuring they do not perform shift or reduce operations if the next input $\kappa$ terminals on the input tape is erroneous~\cite{10.1145/356628.356629}.
Consequently, every entry in the parsing table corresponding to $\kappa$ terminals that maps to a shift or reduce operation indicates that the terminal is acceptable. 
This property allows us to use LR(1) parsing tables to efficiently compute accept sequences at any intermediate point, making them preferable for \Tool{} applications. 
Thus, computing acceptable terminals with LR(1) parsers has a complexity of $O(|\allterminals|)$.
Although LALR(1) parsers are more commonly used due to their smaller memory requirements and faster construction, computing acceptable terminals with them requires iterating over all terminals leading to a complexity of $O(T_\parser \cdot |\allterminals|)$ due to the need for multiple reduce operations before confirming the validity of each terminal. 
Furthermore, while for $\kappa > 1$, LR($\kappa$) parsers can compute accept sequences of length $\kappa$ immediately, they incur extremely high memory requirements. 
Additionally, while we can use LL($\kappa$) parsing tables to compute the next $\kappa$ accept terminals, LR($\kappa$) parsers offer a higher degree of parsing power. 
Therefore, we employ LR parsers in \Tool{}.
Our evaluation indicates that LR(1) parsers suffice for eliminating most syntax errors, making them a practical choice for \Tool{}. 
}
We discuss how the implementation of how parsing is performed \emph{incrementally} to obtain the accept sequences and remainder in the Appendix~\ref{sec:incparse}.

\noindent \textbf{Accept Sequences:}
In our implementation, we focus on generating accept sequences of length 1 or 2, as they can be efficiently obtained from LR(1) parser. 
While this approach incurs some loss of precision, it leads to sound but incomplete syntactical decoding. 
Further, our evaluation demonstrates that this strategy is efficient and precise in practical scenarios.
We note that $\pmatch$ $r.t$ with a 2-length sequence is equivalent to $\dmatch$ with a 1-length sequence, as stated in Lemma~\ref{lemma:eq}. 
Consequently, in our work, we precompute mask stores $\dmap{0}$ and $\dmap{1}$.
On parsing the partial output $\partialcode$, the parser state \add{of LR(1) parsers} can be used to directly obtain syntactically acceptable terminals for the current completion ($\curaccepts$) and the next completion ($\nextaccepts$). 
We utilize $\curaccepts$ and $\nextaccepts$ to construct the accept sequences $\accepts$, considering two cases:

% \begin{description}
%     \itemsep0em
%     \item 
    
    \textbf{Case 1: $\partialcode = l_1.l_2 \dots l_f$:} Let $\terminal_f$ represent the type of the final lexical token. 
    In many instances, a token may be extended in the subsequent generation step, such as when an identifier name grows longer or additional words are appended to a comment. 
    In those cases if $\nextaccepts = {\terminal_1^1, \terminal_2^1, \dots, \terminal_n^1}$, we include all 2-length sequences $\{\terminal_f, \terminal_i^1\}$ for each $i$. 
    As previously discussed, the type of the final lexical token may change from $\terminal_f$. 
    Consequently, when $\curaccepts = \{\terminal_1^0, \terminal_2^0, \dots, \terminal_n^0\}$, we add 1-length sequences $\sequence_i$ for each terminal sequence $\{\terminal_i\}$ from $\curaccepts$, excluding $\terminal_f$. 
    This method ensures the generation of sequences accounting for potential extensions of the same token and changes in the type of the final lexical token.

    % \item 
    \textbf{Case 2 $\partialcode = l_1.l_2 \dots l_f.u$:} In this scenario, the current terminal is incomplete, leading to a lack of information about subsequent terminals. 
    Consequently, when $\nextaccepts = \{\terminal_1, \terminal_2, \dots, \terminal_n\}$, we define $\accepts$ as a set of sequences: $\{\sequence_1, \sequence_2, \dots, \sequence_n\}$, where each $\sequence_i$ corresponds to a single terminal sequence $\{\terminal_i\}$ from $\nextaccepts$. 
    Specifically, $\sequence_1 = \{\terminal_1\}$, $\sequence_2 = \{\terminal_2\}$, and so forth. 
    % This approach ensures the generation of accept sequences based on the available information for subsequent terminals when the current one is incomplete.
% \end{description}

\subsection{Time Complexity}
\label{sec:time}
In this section, we analyze the time complexity of the \Tool{} algorithm. 
We focus on the cost of creating the mask at each iteration of the LLM generation loop.
The key computations involved in this process are the parsing carried out by the incremental parser to compute $\accepts$ and the lookup/unification operations performed through the DFA mask store.

The incremental parser parses $O(1)$ new tokens at each iteration and computes $\accepts$. 
Let $T_A$ represent the time taken by the parser to compute the accepted terminals and $T_\parser$ denote the time the parser takes to parse a new token and update the parser state. 
Hence, in each iteration, the parser consumes $O(T_A+T_\parser)$ time to generate $\accepts$.
The DFA mask store lookup involves traversing $|\accepts|$ DFA sequences, with the number of steps in this walk bounded by the length of the remainder $r$. 
As $\accepts$ can have a maximum of $|\allterminals|$ sequences, the DFA walk consumes $O(|\allterminals| \cdot \len(r))$ time.
We employ a hashmap to facilitate efficient lookups at each DFA node, ensuring that all lookups take constant time. Consequently, this step takes $O(|\allterminals|)$ time. 
Let $T_\cup$ denote the time taken for computing the union of binary masks. 
With potentially $|\allterminals|$ union operations to be performed, the mask computation takes $O(T_\cup \cdot |\allterminals|)$ time.
Therefore, the overall time complexity at each step during generation is given by $O(T_A + T_\parser + |\allterminals| \cdot \len(r) + T_\cup \cdot |\allterminals|)$. 

For an incremental LR(1) parser, the complexity of our algorithm at each step of LLM token generation is $O(|\allterminals| \cdot \len(r) + T_\cup \cdot |\allterminals|)$. 
\add{With our lexer assumption, we ensure that} the remainder $r$ is small, allowing us to further simplify our complexity analysis to $O(T_\cup \cdot |\allterminals|)$ by treating $\len(r)$ as constant.
% Additionally, all these computations have the potential for parallelization during LLM generation, but this aspect is deferred to future work.

\noindent \textbf{Offline cost: } The cost of computing the mask store $\dmap{\alpha}$ offline involves considering all DFA states $q \in Q_\Omega$, all possible terminal sequences of length $\alpha$, and all tokens $t \in \vocab$. 
Given that we need to traverse the DFA for $\len(t)$ steps for each entry in the store, the time complexity for computing the mask store is $O(\textit{max}_{t \in \vocab}(\len(t)).|Q_\Omega|.|\vocab|.|\allterminals|^\alpha)$. 
Typically, $\len(t)$ is small, allowing us to simplify this to $O(|Q_\Omega|.|\vocab|.|\allterminals|^\alpha)$. 
In our implementation, the use of $\dmap{0}$ and $\dmap{1}$ results in a cost of $O(|Q_\Omega|.|\vocab|.|\allterminals|)$.
The size of $|Q_\Omega|$ depends on the complexity of regular expressions for the terminals, which may vary for each grammar. 
However, as demonstrated in our evaluation section, these mask stores can be computed within 10 minutes for each combination of grammar and LLM. 
This computation is a one-time cost that can be amortized over all generations performed for the given LLM and grammar.

\subsection{SynCode Framework}
\label{sec:framework}

\begin{figure}[tb]
\centering
\includegraphics[width=13cm]{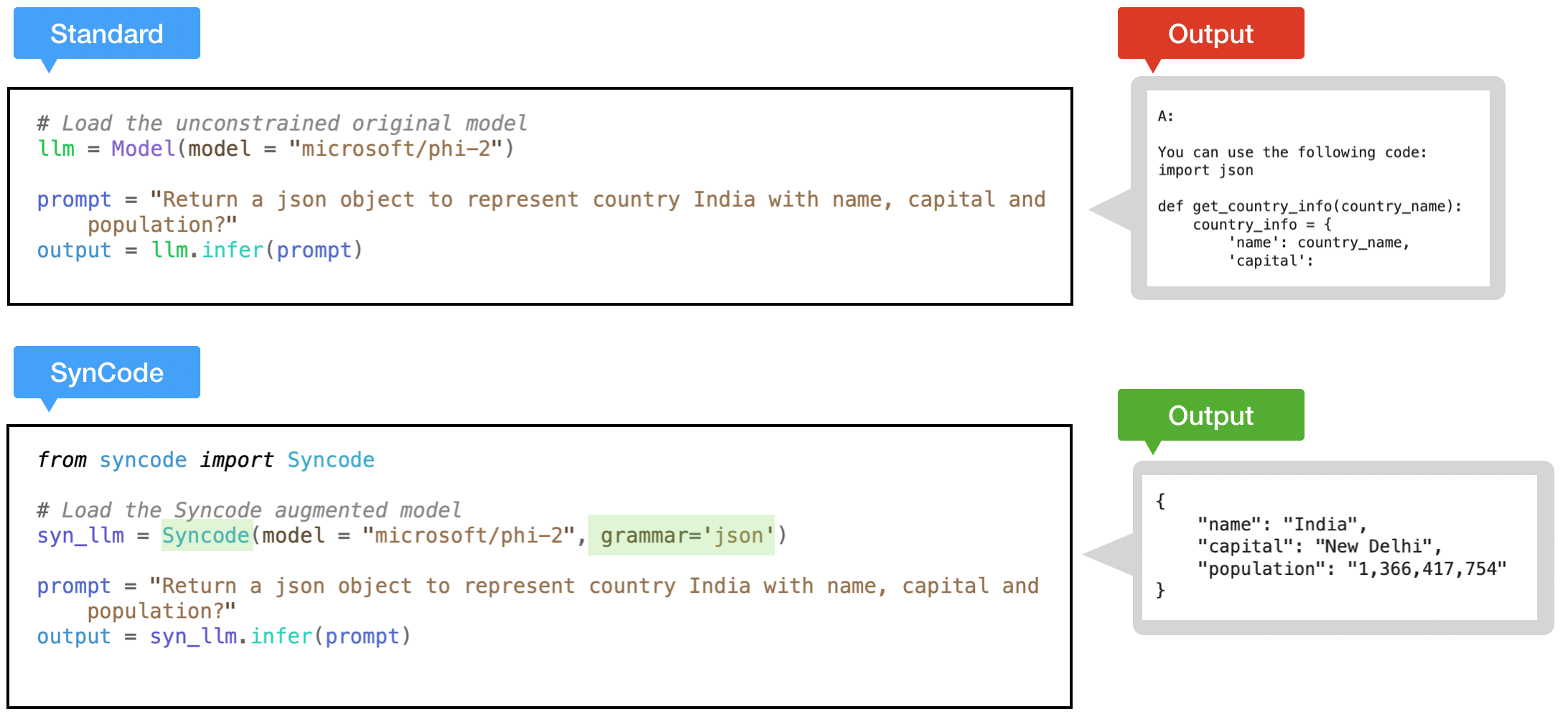}
\vspace{-.1in}
\caption{The upper section displays erroneous output from a standard LLM generation, failing to produce the intended JSON format. The lower segment showcases the fix achieved through the use of the \Tool{} framework.
} 
\label{fig:dfa}
\vspace{-.1in}
\end{figure}

Figure~\ref{fig:dfa} shows how \Tool{} framework can be used in practice by selecting a grammar. 
We next discuss other important features of the framework.

\noindent \textbf{Adding a New Grammar.} 
Our Python-based \Tool{} framework is shipped with several built-in grammars such as JSON, Python, Go, etc. 
A user can apply \Tool{} for arbitrary grammar by providing the grammar rules in EBNF syntax with little effort.
The grammar needs to be unambiguous LALR(1) or LR(1) grammar for using the respective base parsers.
% The power of the LALR(1) parser is sufficient for most mainstream formal languages~\cite{10.5555/42212}.

\noindent \textbf{Ignore Terminals.}
Our EBNF syntax adopted from Lark allows one to provide \emph{ignore terminals} as part of the grammar. 
Lark ignores those terminals while parsing. 
In the case of Python, this includes \emph{comments} and \emph{whitespaces}. 
\Tool{} handles these ignore terminals by adding a trivial 1-length accept sequence for each of these ignore terminals.

\noindent \textbf{Parsers.}
\Tool{} supports both LALR(1) and LR(1) as base parsers. 
We adapt Lark's~\cite{lark} LALR(1) parser generator for \Tool{}. 
Since Lark does not implement the LR(1) parser generator, we implemented the LR(1) parser generator on top of the Lark.
The generation of LR(1) parser which is performed offline may take longer time compared to the LALR(1) parser (e.g., up to 2 mins for our Python grammar), however, it is more efficient at inference time in computing the accept sequences.
Further, since the Lark-generated parser is non-incremental, we build the incremental parser on top of it by caching the parser state as described in Appendix~\ref{sec:incparse}.

\noindent \textbf{Non-CFG Fragments of PLs.}
\Tool{} can handle non-context-free fragments of PLs, such as \emph{indentation} in Python and end-of-scope markers in Go.
To support languages with indentation, such as Python and YAML, \Tool{} has a mechanism that tracks acceptable indentation for the next token, effectively masking tokens that violate indentation constraints at a given point.
This indentation constraint feature can be enabled with any new grammar. 
Similarly, for handling other custom parsing rules beyond CFGs, users can add additional constraints to the generation by overriding specific \Tool{} functions. 
For instance, in Go, semicolons are optional and may be automatically inserted at the end of non-blank lines. 
Implementing such constraints in \Tool{} programmatically requires minimal effort.
However, \Tool{} currently does not support enforcing semantic constraints. (e.g, if a variable in a program is defined before it is used.)

% \noindent \textbf{Opportunistic Mode.}
% \add{
% Additionally, \Tool{} incorporates a form of optimization also present in~\cite{beurerkellner2024guiding, llamacpp}, referred to as opportunistic masking. 
% Instead of initially computing the mask, as outlined in Algorithm~\ref{alg:main}, we first execute the decode step of the LLM and use the parser to examine the token suggested by the model. 
% Only if the token is incorrect\fTBD{how wo we know the overhead is going to be better this way than the other?} do we proceed to compute the remainder of the mask, thus allowing the LLM to direct the decoding process.\fTBD{Where do we use it in our experiments? Do we show this optimization is useful/nencessary? Does it perform better in our sys than in those from the 2 refs?}
% }

% \noindent \textbf{Under-approximation.}
% \add{
% Write here!
% }

\section{Experimental Methodology}
\newcommand{\CS}{C\nolinebreak\hspace{-.05em}\raisebox{.6ex}{\bf \#}}
\noindent \textbf{Models.}
In our evaluation, we select a diverse set of state-of-the-art open-weight LLMs of varying sizes. 
Since closed-source LLMs, such as GPT-4 or Gemini, do not expose generation logits through their APIs, applying a constrained generation approach in \Tool{} is not feasible. 
Therefore, we focus on enhancing smaller, open-source models in our evaluation.
We select the state-of-the-art models Llama-2-7B-chat ~\cite{touvron2023llama2openfoundation} and Gemma2-2B-it ~\cite{gemmateam2024gemma2improvingopen} for our JSON evaluation.
For text-2-SQL generation experiments, we use Llama-2-7B-chat, Llama-3.2-1B, Llama-3.2-3B, and Gemma-2-2B-it.
Furthermore, we chose models such as \llama{}~\cite{touvron2023llama}, \wizard{}~\cite{luo2023wizardcoder}, and \codegen{}~\cite{nijkamp2023codegen} for code completion. 
% These models were chosen since they were the top-performing small models featured on the BigCode Models Leaderboard~\cite{bigcode2023leaderboard}.
% We give more detail about these models in Appendix~\ref{sec:models}.

% \input{tables/models}
\noindent \textbf{Datasets.}
We focus our evaluation on generating JSON, SQL, Python, and Go outputs. 
We choose JSON as it is supported by the baselines~\cite{llamacpp, willard2023efficient}, which allows us to compare against them.
We selected Python since it is extensively present in the training data employed for LLM training and fine-tuning. 
Conversely, we opted for Go due to its lower standard LLM accuracy and a relatively smaller presence in the training data.
We consider JSON-Mode-Eval~\cite{jsoneval} dataset for text to JSON generation and HumanEval and MBXP~\cite{athiwaratkun2023multilingual} dataset for evaluating Python and Go code generation. 
We display examples of prompts from these datasets in Appendix~\ref{sec:prompts}.

\begin{itemize}
\itemsep0em 
\item \noindent \textbf{JSON-Mode-Eval~\cite{jsoneval}.} It consists of 100 zero-shot problems. Each problem prompt follows the chat format with a system prompt specifying a JSON schema and a user prompt requesting the LLM to generate a JSON object that contains specified contents.

\item \noindent \textbf{Spider text-2-SQL.} Spider~\cite{yu-etal-2018-spider} text-to-SQL dataset consists of 1,034 problems of varying difficulty levels: \emph{easy} (250), \emph{medium} (440), \emph{hard} (174), and \emph{extra hard} (170).

\item \noindent \textbf{Multilingual HumanEval~\cite{athiwaratkun2023multilingual}.}  It is an extension of the original HumanEval collection~\cite{chen2021evaluating}, which comprises 164 Python programming problems, to include other languages like Go. Each problem in the dataset consists of a function definition, and text descriptions of the function as a part of the function docstring.

\item \noindent \textbf{MBXP~\cite{athiwaratkun2023multilingual}.} It is extended from the MBPP~\cite{austin2021program} dataset for Python to support other languages such as Go. The dataset consists of 974 problems with the same format as HumanEval.

\end{itemize}

% \noindent \textbf{Parsers.}
% Our evaluation considers the LR(1) base parser as the default and considers the LALR(1) parser for the ablation.

% We build our parsing toolkit on top of Lark\cite{lark}.}
\noindent \textbf{Grammars.}
For Python, we used the readily available grammar from the Lark repository. 
For Go, we converted an existing LL(*) grammar from~\cite{antlr} implementation to LR(1) grammar for our use.
We write the CFG for these languages using the Extended Backus-Naur Form (EBNF) syntax.
We use a substantial subset of grammar for Python and Go syntactic generation with \Tool{}. 
The grammar has commonly used features of the language such as control flow, and loops, and excludes some features such as Python's support for lambda functions.  
Adding support for more features would require more engineering effort but it will not change the overall technique.
The grammars we used are available in Appendix~\ref{sec:grammar}.
The JSON grammar consists of 19 rules and 12 terminals.
The Python grammar we used contains 520 production rules and 94 terminals, whereas the Go grammar comprises 349 rules and 87 terminals.  

\noindent \textbf{Evaluating Syntax Errors.}
For evaluating the errors in the generated output in each of the languages, we use their respective standard compilers.

% \noindent \textbf{Hyperparameter Values.}
% We use these same hyperparameter values for our baselines.

\noindent \textbf{Experimental Setup.}
We run experiments on a 48-core Intel Xeon Silver 4214R CPU with 2 NVidia RTX A5000 GPUs. 
\Tool{} is implemented using PyTorch~\cite{NEURIPS2019_9015}, HuggingFace transformers library~\cite{wolf-etal-2020-transformers} and Lark library~\cite{lark}.

\noindent \textbf{Baselines.}
\add{
We evaluate three state-of-the-art baselines \outlines{}~\cite{willard2023efficient} v0.1.1, \guidance{}~\cite{guidance} v0.1.16, and \llamacpp{}~\cite{llamacpp} v0.3.1 in our study. 
% \outlines{} implementation is most similar to Synchromesh~\cite{poesia2022synchromesh} open-source reimplementation\cite{synchromesh_github} from a subset of original authors of the paper. 
% However, we found that this open-source reimplementation has several bugs in its implementation and it does not work on most grammars we tried in our evaluation. 
The algorithmic differences in the baselines and \Tool{} are discussed in Section~\ref{sec:related}.
We perform a warmup run for each experiment where we measure inference time to ensure that one-time precomputation time is not included in the inference runtime.
For a fair comparison with baselines, \Tool{} uses opportunistic masking~\cite{beurerkellner2024guiding}, an optimization used in \llamacpp{} and \guidance{}. 
Instead of computing the full logit vector mask upfront, the model generates a token and only computes the mask if the proposed token is incorrect.
}

\section{Experimental Results} \label{sec:exp_main}
\add{
In this section, we evaluate \Tool{} on generating various formal languages.
We compare \Tool{} with state-of-the-art baselines and perform various ablation studies. 
}

\add{
\Tool{} allows the model to generate a special EOS token (indicating the end of generation) only when the output belongs to $\lang(G)$. 
In practice, however, LLM generation typically stopped after a fixed maximum number of tokens, $n_\textit{max}$. 
Therefore, terminating with the EOS token within this limit is not always guaranteed potentially resulting in syntax errors. 
% Nevertheless, this approach significantly enhances syntactic accuracy, as we demonstrate in our evaluation.
}

\subsection{Effectiveness of \Tool{} for JSON Generation} 
\label{sec:exp_errors}
% % \begin{wraptable}{r}{.6\textwidth} 
% \begin{table}[b]
%     \tablesize
%     \centering
%     % \vspace{-0.15in}
%     \caption{Effectiveness of \Tool{} in generating syntactic Go code}
%     \begin{tabular}{@{}ll ccc@{}}
%         \toprule
%         Metric & Prompt Type & \Baseline{} & \Tool{} & $\downarrow$ \\
%         \hline
%          & CodeGen-350M  & 573 & 49 & 91\% \\
%         HumanEval & WizardCoder-1B  & 1031 & 50 & 95\% \\
%         & Llama-7B  & 725 & 10 & 99\% \\
%         \hline
%         & CodeGen-350M  & 212 & 2 & 99\% \\
%         MBXP & WizardCoder-1B  & 243 & 14 & 94\% \\
%         & Llama-7B & 414 & 1 & 99\% \\
        
%         \bottomrule
%     \end{tabular}
%     % \vspace{-0.1in}
%     \label{tab:json_eval}
% % \end{wraptable}
% \end{table}

\begin{table}[htb]
    \tablesize
    \centering
    % \vspace{-0.15in}
    \caption{Effectiveness of \Tool{} in generating JSON with original and explicit prompts.}
    % \begin{tabular}{@{}ll ccc@{}}
    \begin{tabular}{@{}l r r r r r r r@{}}
        \toprule
        % Dataset & Architecture & \Baseline{} & \Tool{} & $\downarrow$ \\
        Model & Tool & \multicolumn{2}{c}{Syntax Errors} & \multicolumn{2}{c}{Validation Accuracy (\%)}
        & \multicolumn{2}{c}{Generation Time (s)}\\
        & & Original & Explicit & Original & Explicit
         & Original & Explicit\\
        % \hline
        \midrule
        & \textbf{\Tool{}}  & \textbf{0} & \textbf{0} & \textbf{66\%} & \textbf{84\%} & \textbf{3.07} & \textbf{3.02} \\

        & \Baseline{}  & 98 & 41 & 2\% & 58\% & 3.58 & 3.11 \\
        Llama-2-7B-chat & \llamacpp{} & 23 & 23 & 63\% & 68\% & 21.91 & 20.84 \\
        & \guidance{} & 13 & 11 & 57\% & 65\% & 5.14 & 4.14 \\
        & \outlines{}$^\dag$  & 16 & 14 & 62\% & 56\% & 38.07 & 41.79 \\
        & \transformerscfg{}  & 2 & 0 & 62\% & 64\% & 6.08 & 4.01 \\
        \midrule

        & \textbf{\Tool{}}  & \textbf{0} & \textbf{0} & \textbf{99\%} & \textbf{100\%} & 4.82 & 4.74 \\

        & \Baseline{}  & 59 & 59 & 41\% & 41\% & 4.32 & 5.82 \\
        Gemma2-2B-it & \llamacpp{} & 7 & 7 & 92\% & 91\% & 22.06 & 21.97 \\
        & \guidance{} & 1 & 1 & 96\% & 96\% & 6.09 & 5.47 \\
        & \outlines{}  & 2 & 0 & 67\% & 90\% & \textbf{1.99} & \textbf{2.64} \\
         & \transformerscfg{}  & 1 & 0 & 96\% & 95\% & 19.12 & 9.49 \\
        
        \bottomrule
    \end{tabular}
    \begin{tablenotes}
    \footnotesize
    \item $\dag$ We observed issues when using Llama-2-7B-chat with Outlines v0.1.1 and therefore, we use older version v0.0.46. 
    
    % use v0.0.46 when running Outlines for Llama-2-7B-chat and v0.1.1 for Gemma2-2B-it as these are the latest releases that give the best performance on these models.
\end{tablenotes}
    % \vspace{-0.1in}
    \label{tab:json_eval}
% \end{wraptable}
\end{table}

We evaluate the effectiveness of \Tool{} in 
guiding LLMs with the JSON grammar to generate syntactically correct JSON. 
We run the inference with Llama-2-7B-chat and Gemma2-2B-it with \Tool{}, \llamacpp{}, \outlines{}, \guidance{}, \transformerscfg{}, and \baseline{} generation on the 100 problems from the JSON-Mode-Eval dataset. 
We select these models for the JSON experiment as they are supported by all considered baselines. 

% We chose \llamacpp{}, \outlines{}, and \guidance{} as they have JSON grammars that work with their frameworks. 
We run \llamacpp{} on a CPU as it requires a specific CUDA version not compatible with our machine. 
We set max new tokens $n_\textit{max} = 400$. 
We also report an evaluation of augmenting the prompts with an explicit request to output only JSON. 
We present an example of these explicit prompts in Appendix~\ref{sec:prompts}. 
We evaluate the correctness of JSON generated by an LLM by first evaluating whether the JSON string can be parsed and converted to a valid JSON object. 
We further evaluate whether the generated JSON is valid against the schema specified in the prompt. 
Although the \Tool{} does not enforce the specific schema to the JSON output for each task, we believe it is an important research question to check whether the reduced syntax errors due to \Tool{} can also lead to improved schema validity.

Table~\ref{tab:json_eval} presents our evaluation results. We report results for both the prompts taken directly from the dataset (denoted as "Original") and after augmenting these prompts with an explicit request to output JSON (denoted as "Explicit"). 
In the "Validation Accuracy" column, we compute the percentage of valid completions against their respective schemas. 
In the "Generation Time (s)" column, we report the average time taken to generate a completion to a prompt from the dataset. 
Guiding Llama-2-7B-chat and Gemma2-2B-it  with the JSON grammar via \Tool{} eliminates syntax errors in generated JSON. 
On the other hand, \baseline{} generation results in syntactically incorrect JSON for 98\% and 59\% of completions to the original prompts for the Llama-2-7B-chat and Gemma2-2B-it models respectively. A majority of these errors are due to the generation of natural language before and after the JSON. Explicit prompts somewhat mitigate this issue, but still results in syntactically invalid outputs to 41\% and 59\% of these prompts for \baseline{} Llama-2-7B-chat and Gemma2-2B-it generation respectively, primarily due to errors such as unmatched braces and unterminated string literals.\llamacpp{}, \outlines{}, \guidance{}, and \transformerscfg{} face similar problems with closing braces and terminating strings.
% Augmenting \chat{} with \outlines{}, results in the generation of nonsensical output, such as numbers, empty braces, empty strings, and incomplete JSON, that does not align with the prompts. 

Notably, \Tool{} significantly improves the JSON schema validation accuracy of Gemma2-2B-it completions over \baseline{} generation, from 41\% to 99\% and 41\% to 100\% for original and explicit prompts respectively. 
\add{Furthermore, \Tool{} outperforms \llamacpp{}, \outlines{},\guidance{}, and \transformerscfg{} in validation accuracy of Llama-2-7B-chat completions by 3\%, 4\%, 9\%, and 4\% respectively for original prompts and 16\%, 28\%, 19\%, and 20\% for explicit prompts.} The remaining schema validation errors with \Tool{} are semantic errors, including data type mismatch between the generation JSON and schema, missing fields required by the schema, and adding extra fields not allowed by the schema. \Tool{} is faster than all baseline grammar-guided generation methods for Llama-2-7B-chat and all but \outlines{} for Gemma2-2B-it. The low generation time with \outlines{} for Llama-2-7B-chat can largely be attributed to the fact that many of its completions to prompts are empty JSON (35\% of original and 7\% of explicit) which takes few tokens to generate, but often does not conform to the schema.

% \add{\Tool{} is 13.8x faster than \outlines{} for explicit prompts and 12.4x faster for original prompts. Similarly, \Tool{} is 1.4x faster than \guidance{} for explicit prompts and 1.6x faster for original prompts.} 
% \llamacpp{} runtimes are relatively slower than other approaches as it was run on a CPU. 
% We observed that JSON grammar-guided generation in \llamacpp{} reduces the average time to generate a completion by between 6.4\% to 10.9\% over standard \llamacpp{} generation.  

Interestingly, we observe that for Llama-2-7B-chat, \Tool{} also reduces the average generation time over \baseline{} generation. 
We attribute this finding to the fact that without grammar-guided generation, the model generates syntactically invalid output, such as natural language, in addition to JSON and thus generates more tokens in response to the same prompt than with \Tool{}. 
Thus, augmenting LLMs with \Tool{} can significantly improve syntactical correctness and runtime efficiency.

\subsection{Effectiveness of \Tool{} for SQL Generation} 
\label{sec:exp_errors}
This study demonstrates that \Tool{} improves text-to-SQL generation by enforcing grammar constraints, ensuring that generated SQL queries are syntactically accurate. 
We evaluate the following models for SQL generation: Llama-3.2-1B, Llama-3.2-3B (base models) and Llama-2-7B-chat, Gemma-2-2B-it (instruct-tuned models). 
We observe that despite explicitly prompting to only generate the SQL query, the instruct-tuned Gemma-2-2B-it model often enclosed generated SQL queries within markers, such as \str{\textasciigrave\textasciigrave\textasciigrave} or \str{\textasciigrave\textasciigrave\textasciigrave sql}. 
Thus, we consider another baseline for Gemma-2-2B where we extract the SQL query substring within these markers, handling cases where the output format is either \str{\textasciigrave \textasciigrave \textasciigrave \{SQL query\} \textasciigrave \textasciigrave \textasciigrave} or \str{\textasciigrave \textasciigrave \textasciigrave sql \{SQL query\} \textasciigrave \textasciigrave \textasciigrave}.

For evaluation, we use the Spider~\cite{yu-etal-2018-spider} text-to-SQL dataset, which consists of 1,034 problems of varying difficulty levels: \emph{easy} (250), \emph{medium} (440), \emph{hard} (174), and \emph{extra hard} (170). 
We prompt models with schema information and text queries, instructing them to generate SQL queries only. 
Using greedy decoding and \str{\textbackslash n\textbackslash n} is used as an additional stopping condition for all experiments.

\begin{table}[tbh]
    \scriptsize
    \centering
    \caption{Comparison of \Tool{} and unconstrained generation on SQL generation.}
    \begin{tabular}{llcccccccc}
        \toprule
        \multirow{2}{*}{\textbf{Model}} & \multirow{2}{*}{\textbf{Method}} & \multicolumn{5}{c}{\textbf{Accuracy (\%)}} & \multirow{2}{*}{\textbf{Execute (\%)}} & \multirow{2}{*}{\textbf{Tokens}} & \multirow{2}{*}{\textbf{Time (s)}} \\
        \cmidrule(lr){3-7}
        & & \textbf{Easy} & \textbf{Medium} & \textbf{Hard} & \textbf{Extra} & \textbf{Overall} & & & \\

        \midrule
        & \Baseline{} & 0.0 & 0.0 & 0.0 & 0.0 & 0.0 & 0.0 & 221.43 & 9.883 \\
Gemma2-2B-it & \Baseline{}+ & 44.8 & 18.9 & 21.9 & 17.1 & 25.4 & 77.6 & 221.43 & 9.893  \\
& \Tool{} & \textbf{45.8} & 18.7 & \textbf{23.3} & \textbf{20.6} & \textbf{26.3} & \textbf{78.2} & \textbf{135.17} & \textbf{5.876} \\

        \midrule
        & \Baseline{} & 34.4 & 22.0 & 12.1 & 4.1 & 20.4 & 32.6 & 44.74 & 1.148 \\
Llama-2-7b-chat & \Tool{} & \textbf{40.0} & \textbf{27.3 }& \textbf{13.8} & \textbf{5.9} & \textbf{24.6} & \textbf{41.6} & 50.33 & 1.483 \\

        \midrule
        & \Baseline{} & 40.8 & 24.8 & 20.7 & 10.6 & 25.6 & 51.1 & 48.00 & 0.509 \\
Llama-3.2-1B & \Tool{} & \textbf{46.8} & \textbf{28.2} & \textbf{23.0} & \textbf{10.6} & \textbf{28.8} & \textbf{59.0} & 56.36 & 0.916 \\

            \midrule
             & \Baseline{} & 38.0 & 29.5 & 28.2 & 12.9 & 28.6 & 67.4 & 47.78 & 0.846 \\
Llama-3.2-3B & \Tool{} & \textbf{47.2} & \textbf{34.8} & \textbf{32.8} & \textbf{19.4} & \textbf{34.9} & \textbf{81.4} & 47.63 & 1.164 \\

        \bottomrule
    \end{tabular}
    \label{tab:sql_comparison}
    % \vspace{-.2in}
\end{table}

Table~\ref{tab:sql_comparison} presents a comparison of \Tool{} and unconstrained generation across key metrics. The Accuracy (\%) column shows the percentage of correctly generated SQL queries across different difficulty levels.
Execute (\%) reflects the percentage of queries successfully executed without runtime errors in SQLite. 
The Tokens column reports the average number of tokens generated, and Time(s) shows the average generation time. 
\Baseline{}+ row for Gemma2 denotes the result for the additional baseline where we extract the SQL query from the full generation using regex matching.

We observe that \Tool{} achieves better performance over the baselines in terms of both execution percentage and execution accuracy.
For example, with the Llama-3.2-3B model, \Tool{} achieves an execution success rate of 81.4\%, compared to 67.4\% for unconstrained generation. 
Further, the execution accuracy improves from 28.6\% to 34.9\%.
In the case of the Gemma2-2B-it model, we observe that \Tool{} shows a moderate improvement over the \Baseline{}+ accuracy. 
However, it shows a significant gain in the speed (1.7x) of generation and a reduction in the number of tokens generated. 
Although the Gemma2-2B-it model has a good execution percentage without any runtime errors.
The instruct-tuned models tends to use large number of tokens that are not part of the query.
In applications where the goal is to use LLMs to generate SQL queries without additional explanations, the result with Gemma2-2B-it shows that \Tool{} is useful in improving the efficiency of LLM generation along with the improvements in accuracy. 

\subsection{Effectiveness of \Tool{} for GPL} 
\label{sec:exp_errors}

\begin{table}[tbh]
    \tablesize
    \centering
    % \vspace{-0.15in}
    \caption{Number of programs with syntax errors for \baseline{} and \Tool{} generation ($\downarrow$ shows how much \Tool{} reduces the occurrence of the syntax errors compared to Standard generation.}
    % \begin{tabular}{@{}ll ccc@{}}
    \begin{tabular}{@{}l l r r r r r r@{}}
        \toprule
        % Dataset & Architecture & \Baseline{} & \Tool{} & $\downarrow$ \\
        Dataset & Model & \multicolumn{3}{c}{Python} & \multicolumn{3}{c}{Go}\\
        & & \Baseline{} & \Tool{} & $\downarrow$ & \Baseline{} & \Tool{} & $\downarrow$ \\
        % \hline
        \midrule
        
         & CodeGen-350M  & 271 & \textbf{15} & 95\% & 573 & \textbf{49} & 91\% \\
        HumanEval & WizardCoder-1B  & 36 & \textbf{3} & 92\% & 1031 & \textbf{50} & 95\% \\
        & \llama{}  & 291 & \textbf{2} & 99\% & 725 & \textbf{10} & 99\% \\
        \hline
        & CodeGen-350M  & 78 & \textbf{4} & 95\% & 212 & \textbf{2} & 99\% \\
        MBXP & WizardCoder-1B  & 28 & \textbf{2} & 93\% & 243 & \textbf{14} & 94\% \\
        & \llama{} & 148 & \textbf{5} & 97\% & 414 & \textbf{1} & 99\% \\
        
        \bottomrule
    \end{tabular}
    % \vspace{-0.1in}
    \label{tab:main_eval_python}
% \end{wraptable}
\end{table}

%---------------------- 
We run inference with \codegen{}, \wizard{}, and \llama{} with \Tool{} and with the \baseline{} no-masking approch. 
We do not compare \Tool{} with the other baselines as none of these works support general-purpose programming language grammars. 
We experiment with both Python and Go programming languages, evaluating performance on zero-shot problems from the HumanEval and MBXP datasets. 
For each dataset, we generate $n = 20$ and $n = 1$ samples per problem with the LLM, respectively. 
We run the LLM-generated code completion against a predefined set of unit tests. 
For each unit test, we record the error type when running the generated program against that test case. 
We use the hyperparameters temperature $= 0.2$ and top $p = 0.95$. 
% Table \ref{tab:main_eval_python} presents our results for Python code. The columns \baseline{} and \Tool{} represent the total number of errors of a particular Error Type of LLM-generated code completions to problems in a particular dataset when utilizing that respective generation approach. The column $\downarrow$ represents the percentage reduction from the \baseline{} column to the \Tool{} column. Notably, our experiments reveal that \Tool{} reduces the number of syntax and indentation errors by over 90\% over the baseline in most experiments. Moreover, \Tool{} reduces the number of syntax or indentation errors to less than 1\% of the total test cases. 
Table~\ref{tab:main_eval_python} presents our results for Python and Go. 
The columns \baseline{} and \Tool{} represent the total number of generated programs with syntax errors for the respective approaches. 
The column $\downarrow$ designates the percentage reduction in syntax errors from the \baseline{} generation to the \Tool{} generation.
In this evaluation, across both HumanEval and MBXP datasets, we generate a total of 4154 samples for each language. 
On average, of all \baseline{} generated samples, ~6\% and ~25\%  have syntax errors for Python and Go, respectively. 
% The HumanEval dataset has 164 problems per language, and we generate $n = 20$ samples per problem, resulting in a total of 3280 generated samples per language. 
% The MBXP dataset has 974 problems per language and $n = 1$ samples are generated per problem. The LLMs with \baseline{} generation generate syntax errors for ~9\% and ~30\% of MBXP Python and Go samples respectively on average.

 Notably, our experiments reveal that \Tool{} reduces the number of syntax errors by over 90\% over the baseline in most experiments. 
 Moreover, \Tool{} reduces the number of syntax errors to less than 1\% of the total samples. 
 Interestingly, we observe significantly more Syntax errors in standard LLM-generated Go code than in Python code, likely because the LLMs are trained more extensively on Python code than Go. 
 Thus, \Tool{} can be especially effective for Go and more underrepresented programming languages, where LLMs are more likely to generate syntax errors due to a limited understanding of the language. 
 \Tool{} can bridge this gap by guiding the LLM to sample only the syntactically valid tokens during decoding.

% Furthermore, we report our findings for Go code in Table \ref{tab:main_eval_go} and demonstrate that \Tool{} reduces syntax errors by over 90\% compared to standard generation. Interestingly, we observe significantly more Syntax errors in standard LLM-generated Go code than in Python code, likely due to the fact that the LLMs are trained more extensively on Python code than Go. Thus, \Tool{} can be especially effective for Go and more underrepresented programming languages, where LLMs are more likely to generate syntax errors due to a limited understanding of the language. \Tool{} can bridge this gap by guiding the LLM to sample only the syntactically valid tokens during decoding. 

We further analyze the errors in Python and Go code generated by the LLMs augmented with \Tool{}, an example of which is presented in Appendix~\ref{sec:syncode_error}. 
All of the errors were because the LLM failed to generate a complete program within the maximum token limit. 
Recall, \Tool{} provides guarantees of completeness for syntactically correct partial programs. 
However, it does not guarantee convergence to a syntactically correct and complete program.
% If the LLM cannot generate a complete program within the maximum token limit, \Tool{} cannot guarantee that the generated program is syntactically correct.

% Across all datasets, programming languages, and model architectures evaluated, LLMs augmented with \Tool{} generate code with significantly fewer syntax and indentation errors, underscoring \Tool{}'s effectiveness in addressing syntactic issues.

% \input{tables/runtime}
% \noindent{\bf Runtime Comparison.}
% We compare the runtime of code generation with and without \Tool{}. 
% We used Python and Go prompts from the HumanEval dataset.
% For each example, we generate a single sample ($ n = 1 $) per problem, with the max new tokens parameter set to 200. 
% Table~\ref{table:runtime} reports the average time taken to generate a code completion to a prompt. 
% Our results reveal that \Tool{} increases the generation time by  1.22x on average. Despite these slight time increases, the benefits of \Tool{} in reducing syntax errors outweigh the incurred overhead. 

% As highlighted earlier, LLMs augmented with \Tool{} generate code with significantly fewer syntax errors. 
% Therefore, the marginal increase in generation time with \Tool{} is justified by the potential for improved code quality and syntactic correctness, as demonstrated by our findings. 

\begin{table}[tb] 
    \tablesize
    \centering
    % \vspace{-0.2in}
    \caption{Functional correctness on HumanEval problems}
    \begin{tabular}{@{}l l r r r r r@{}}%{@{}lll rrr@{}}
        \toprule
        Metric & Architecture & \multicolumn{2}{c}{Python} & \multicolumn{2}{c}{Go}\\
        
        & & \Baseline{} & \Tool{} & \Baseline{} & \Tool{} \\
        \midrule
        & CodeGen-350M & 6.8\% & \textbf{6.9\%} & 3.6\% & 3.6\% \\
        
        pass@1 & WizardCoder-1B & 20.0\% & 20.0\% & 9.3\% & \textbf{9.5\%} \\
        
        & \llama{} & 11.2\% & \textbf{11.5\%} & 3.8\% & \textbf{4.25\%} \\
        \hline
        & CodeGen-350M & 10.6\% & 10.6\% & 5.6\% & \textbf{6.1\%} \\
        
        pass@10 & WizardCoder-1B & 27.6\% & \textbf{28.4\%} & 12.5\% & \textbf{13.7\%} \\
        
        & \llama{} & 17.1\% & \textbf{18.9\%} & 8.8\% & 8.8\% \\

        \bottomrule
    \end{tabular}
    % \vspace{-0.4in}
    \label{tab:pass@1}
\end{table}
% \begin{table}[b] 
%     \tablesize
%     \centering
%     \vspace{-0.2in}
%     \caption{Functional correctness on HumanEval problems}
%     \begin{tabular}{l | c c c c | c c c c}%{@{}lll rrr@{}}
%         \toprule
%         Architecture & \multicolumn{4}{c|}{Python} & \multicolumn{4}{c}{Go}\\
        
%         % & & Std. & \Tool{} & Std. & \Tool{} \\
%         & \multicolumn{2}{c}{pass@1} & \multicolumn{2}{c|}{pass@10} & \multicolumn{2}{c}{pass@1} & \multicolumn{2}{c}{pass@10}\\

%         \midrule 
%         & Std. & \Tool{} & Std. & \Tool{} & Std. & \Tool{} & Std. & \Tool{} \\

%         % model python,p1,standard python,p10,standard, python,p1,tool python,p10,tool go,standard,p1 go,standard,p10 go,tool,p1, go,tool,p10

%         \midrule
%         CodeGen-350M & 6.8\% & 6.9\% & 10.6\% & 10.6\% & 3.6\% & 3.6\% & 5.6\% &  6.1\% \\
        
%         WizardCoder-1B & 20.0\% & 20.0\% & 27.6\% & 28.4\% & 9.3\% & 9.5\% & 12.5\% & 13.7\% \\
        
%         Llama-7B & 11.2\% & 11.5\% &  17.1\% & 18.9\% & 3.8\% & 4.25\% & 8.8\% & 8.8\%\\

%         \bottomrule
%     \end{tabular}
%     % \vspace{-0.4in}
%     \label{tab:pass@1}
% \end{table}
\noindent {\bf Functional Correctness for Code Generation.}
% \label{sec:pass@1}
We investigate whether augmenting LLMs with \Tool{} improves the functional correctness of the generated code. 
We evaluate functional correctness using the pass@k metric, where $k$ samples are generated per problem, and a problem is considered solved if any sample passes a set of unit tests, and the fraction of solved problems is calculated. 
Table~\ref{tab:pass@1} reports our results for pass@1 and pass@10 for generated code completions to problems from the HumanEval dataset. 
We observe that augmenting LLMs with \Tool{} has a slight improvement in functional correctness over \baseline{} generation.
%This observation indicates that for these state-of-the-art models, mere syntactic correction is not sufficient to improve their ability to generate logically correct code that passes the unit tests for these tasks.
This observation indicates that for these state-of-the-art models, syntactic correction can result in a small improvement in the logical correctness of the code.
% is not sufficient to improve their ability to generate logically correct code that passes the unit tests for these tasks.

% While \Tool{} significantly reduces the number of syntax errors in LLM-generated code, many of these code completions still end up failing test cases. 

\subsection{Mask Store Overhead} 
\label{sec:exp_overhead}
We analyze the time and memory overhead involved in generating a DFA mask store using \Tool{}. 
The DFA mask store for \chat{} took 113.72 seconds to create and consumes 181 MB of memory. 
Additionally, we report the creation time and memory overhead of DFA mask stores for models used for Python and Go in  Table~\ref{table:memory_overhead}. 
Each row shows the \Tool{} store generation time in seconds, and memory in GBs, for a particular LLM and grammar. 
The $|\vocab|$ column represents the total vocabulary size of the tokenizer of the particular LLM. 
We see that generating the store requires less than 2GB of memory and several minutes across the evaluated models and grammars.
This overhead is minimal for practical \Tool{} use cases, as the mask store is a one-time generation task. 
Thereafter, the mask store can be efficiently loaded into memory and used for repeated inference. 
We see smaller generation time and memory with \chat{} and JSON grammar as opposed to \llama{}, \wizard{}, and \codegen{} with Python and Go grammars since the size of the mask store is proportional to the number of terminals in the grammar.
\begin{table}[t]
\centering
\tablesize

% \begin{wraptable}{r}{.6\textwidth} \centering
% \
% \vspace{-.1in}
\caption{DFA Mask store creation time and memory} % Change 
\begin{tabular}
    {@{}l|c| r r r r@{}}
    \toprule
    & & \multicolumn{2}{c}{Python} & \multicolumn{2}{c}{Go}\\
    Model & $|\vocab|$ & Time(s) & Memory & Time(s) & Memory \\
    \midrule
    \codegen{} &  51200 & 602.26 & 1.87GB & 603.03 & 1.58GB \\
    \wizard{}  & 49153 & 588.28 & 1.83GB & 588.84 & 1.54GB \\
    \llama{}  & 32000 & 382.26 & 1.17GB & 380.49 & 1.06GB\\
    \bottomrule
\end{tabular}

\label{table:memory_overhead}
% \end{wraptable}
\end{table}

\section{Related Work}
\label{sec:related}

Our work focuses on enhancing the syntactical accuracy LLMs by using a constrained decoding algorithm.
Prior research has explored two other primary directions to enhance LLMs' accuracy in generating formal language: 
1) Fine-tuning or prompt engineering \cite{bassamzadeh2024comparativestudydslcode, weyssow2024exploringparameterefficientfinetuningtechniques}, which demands substantial data, compute resources, and time, often without any formal guarantees. 
2) Modifications to the LLM's architecture or tokenization \cite{murty2023pushdown, 10.1145/3597926.3598048, 10.1145/3597503.3639125}, although these techniques have not yet achieved performance comparable to the current state-of-the-art standard LLMs.
However, both fine-tuning and architectural changes are complementary to the grammar-guided decoding approach that we focus on in our work, and any gains through those techniques will improve the overall quality of LLM generation.

There are several recent works on constrained LLM generation~\cite{Wei_2023, 10.1145/3591300, guidance, willard2023efficient, scholak-etal-2021-picard, poesia2022synchromesh, llamacpp, geng2023grammar, beurerkellner2024guiding, agrawal2023guiding, melcer2024constraineddecodingfillinthemiddlecode}.
This includes recent works that have used language-server (tools built for communication between IDEs and programming language-specific tools like static analyzers and compilers) suggestions to enforce language-specific semantic constraints during decoding \cite{agrawal2023guiding, Wei_2023}.
These techniques do not guarantee syntactical accuracy and rely on the availability and efficiency of language servers.

\noindent{\bf Structured LLM Generation.} 
\add{
We focus our further discussion on comparison to the techniques that constrain LLM for structured generation according to a formal language. 
We compare \Tool{} with prior works in terms of precision and efficiency of the algorithms and generality and scalability of frameworks.
}
Table~\ref{tab:constrained_decoding_methods} presents the various recent techniques for structured LLM generation.
The columns "Regex" and "CFG" indicate regular expression and CFG constraining features, respectively. 
The "Precomputed" column denotes techniques that precompute certain structures to enhance generation efficiency. 
% "Sound" indicates whether these techniques provide soundness guarantees. 
The "GPL" column specifies if the tools support general-purpose PLs. 
"Max CFG" displays the number of production rules in the largest supported Grammar by these techniques.
\add{
We obtained these numbers by examining the built-in grammars that were provided in the corresponding libraries. 
}
Finally, the "Input Format" column indicates the format used to specify generation constraints.
In addition to the improvement over the baselines presented in the evaluation, our work focuses on rigorously formalizing the correctness of our CFG-guided generation approach.

% Regex guided
Recent works such as \guidance{}~\cite{guidance} and \lmql{}~\cite{10.1145/3591300} mitigate the unpredictability of LLM responses by using template or constraint-based controlled generation techniques. 
These libraries feature a templating engine where prompts are expressed with holes for the generation to fill. 
\lmql{}~\cite{10.1145/3591300} supports general regular expression constraints, but not CFG constraints.
\guidance{}~\cite{guidance} supports CFG-guided generation. 
It uses Earley parsing~\cite{10.1145/362007.362035} for constrained decoding. 
Similar to other related works, it incurs high inference overhead as it checks the syntactical validity of the entire model vocabulary at each step. 
It uses a trie similar to \cite{poesia2022synchromesh, willard2023efficient, beurerkellner2024guiding}. 
As shown in our evaluation it incurs higher overhead for JSON generation than \Tool{}. 
It iterates over the vocabulary in order of the next token probability to efficiently compute the next token. 
However, this leads to a lack of generality and it cannot be directly combined with an arbitrary decoding strategy. 
% Further, due to the non-standard input format for the grammar of the framework, it is difficult to introduce a new grammar. 

\begin{table}[t]
\centering
\footnotesize
\caption{Overview of various constrained decoding methods}
\begin{threeparttable}
\begin{tabular}{@{}lcccccc@{}}
\toprule  & Regex              & CFG               & \shortstack{Precomputed} & GPL & Max CFG  & Input format \\
\midrule
\lmql~\cite{10.1145/3591300}   & \yes & \no        & \no    & \no & 50-100 & LMQL DSL \\
\guidance~\cite{guidance}  & \yes  & \yes    & \no      & \no & 50-100 & Python DSL \\
\outlines~\cite{willard2023efficient}  & \yes & \yes   & \yes & \no & 50-100 & Lark EBNF \\
\picard~\cite{scholak-etal-2021-picard}        & \yes  & \yes & \no  & \no & 50-100 & Haskell \\
\synchromesh~\cite{poesia2022synchromesh}   & \yes  & \yes  & \no & \no & $\ddag$ & ANTLR \\
\llamacpp~\cite{llamacpp}       & \yes & \yes        & \no & \no & 50-100 & GBNF DSL \\
\tgcd~\cite{geng2023grammar} & \yes  & \yes        & \no & \no & 50-100 & GF  \\ 
\domino~\cite{beurerkellner2024guiding}         & \yes & \yes & \yes & \no & 50-100 & GBNF DSL \\ 
\midrule
\Tool{} (ours)  & \yes & \yes & \yes & \yes & 500+ & Lark EBNF \\ 
\bottomrule
\end{tabular}
\begin{tablenotes}
    \footnotesize
    \item $\dag$ Implementation issues
    $\ddag$ Synchromesh is closed-source and the information about DSL grammars is unavailable
    \item GF: Grammatical Framework, GBNF is a DSL defined by LLAMA.CPP 
\end{tablenotes}
\end{threeparttable}
\normalsize
\label{tab:constrained_decoding_methods}
\end{table}

\outlines{}~\cite{willard2023efficient} is a library originally focused on regular expression-guided generation and recently extended to support grammar-guided generation. 
During LLM generation, \outlines{} employs an incremental Lark-based LALR parser to determine the next acceptable terminals based on the grammar. 
It constructs a larger regular expression by computing the union of regular expressions from all terminals, which is then converted into a DFA during inference. 
It then iterates over all LLM tokens and collects the set of tokens that lead to a valid path through this combined DFA. 
As shown in our evaluation, \Tool{} performs better than \outlines{} on generating with JSON grammar and it currently lacks support for large GPL grammars.
% The approach needs iterating over the vocabulary during inference and performing DFA-based matches across the entire vocabulary. 
% Despite optimizations such as DFA caching and a recent introduction of tensor-based caching, \outlines{} is slower than \Tool{}. 

\llamacpp{}~\cite{llamacpp}, has also recently introduced support for grammar-guided generation. 
This approach models a nondeterministic pushdown automaton with $N$ stacks to maintain possible parse states. 
\llamacpp{} defines a new grammar syntax and implements a simplified basic parser in C++. 
While this implementation in C++ reduces some parsing overhead compared to heavier LR(1) parsers implemented in Python on top of Lark for \Tool{}, it is algorithmically inefficient. 
This inefficiency again is due to the requirement to iterate over the entire vocabulary and update stack states during inference. 
Moreover, the non-standard grammar syntax and limited support for general grammar features restrict its evaluation to simpler grammars such as JSON. 
We anticipate that \llamacpp{} and \outlines{} would perform even slower on grammars with more rules, terminals, and complex regular expressions, such as those found in Python and Go.
As shown in our evaluation, \Tool{} is more efficient and results in fewer syntax errors.

\synchromesh{}~\cite{poesia2022synchromesh} is a proprietary
% ~\footnote{While there exists a publicly available unofficial reimplementation of Synchromesh~\cite{synchromesh_github} that operates on a non-incremental Lark parser, it has bugs and did not work with our grammar despite reasonable efforts to fix the issues.} 
a tool from Microsoft that supports CFG-guided syntactic decoding of LLMs.
Similar to \outlines{}, it creates a union of regular expressions of terminals during LLM generation.
Further, Synchromesh uses a non-incremental parser for parsing.
Both of which lead to lower time complexity.
Synchromesh uses techniques like Target Similarity Tuning for semantic example selection and Constrained Semantic Decoding to enforce user-defined semantic constraints and works on DSLs.
In contrast, our work, \Tool{} focuses exclusively on syntactic generation.

\picard{}~\cite{scholak-etal-2021-picard} uses a specific decoding strategy that maintains a beam of multiple candidate outputs and promptly rejects the candidates that violate the syntax.
It utilizes an incremental monadic parser and was developed specifically to support SQL generation. 
Introducing a new grammar into \picard{} necessitates considerable effort, as it lacks support for a grammar-defining language to provide grammar rules.

% Domino
Recent work Domino~\cite{beurerkellner2024guiding} does CFG-guided LLM generation.
It avoids traversing the whole vocabulary during inference by precomputing a prefix tree corresponding to each NFA state of the terminals of the grammar.
The purpose of creating this structure is similar to \Tool{}'s DFA mask store. 
We believe that \Tool{}'s mask store is more efficient than Domino's prefix tree since on modern machines (especially with GPUs) the union of the boolean masks from mask store can be performed quite efficiently in practice~\cite{paszke2019pytorch}.
Domino defines the \emph{minimally invasive} property which is equivalent to \Tool{}'s soundness property.
One key difference between \Tool{} and Domino is that Domino applies under-approximation, permitting only tokens that align with the lookahead of the parser, while \Tool{} adopts a conservative over-approximation approach, allowing tokens as long as their prefixes match the parser lookahead.
Due to the under-approximation, they claim that it requires $\infty$ parser lookahead to get this soundness, whereas \Tool{} ensures soundness for any lookahead.
% Further, the largest grammar that Domino can support currently is the While Programming Language grammar in C with 70 rules with roughly 25\% overhead. 
Further, the largest grammar that Domino can support currently is highly simplified C grammar with 70 rules with roughly 25\% overhead. 
Domino's code is not available yet to experimentally compare it with \Tool{}. 

% \Comment{ 
% Very recently, a concurrent work Domino~\cite{beurerkellner2024guiding} was released that does grammar-guided LLM generation. 
% One key difference between \Tool{} and Domino is that Domino performs under-approximation for selecting acceptable tokens compared to \Tool{}'s over-approximation.
% Domino defines the "minimally invasive" property which is equivalent to \Tool{}'s soundness property.
% Due to the under-approximation, they claim that it requires $\infty$ parser lookahead to get this soundness, whereas \Tool{} ensures soundness for any lookahead.   
% To avoid traversing the whole vocabulary during inference by precomputing a prefix tree corresponding to each NFA state. 
% This is similar to \Tool{}'s DFA mask store. 
% However, in the worst case assembling all acceptable tokens for Domino may require traversing the whole prefix terminal tree which can take $O(|\vocab|)$ time.
% \Tool{}'s mask store is more efficient than Domino's prefix tree since on modern machines with GPU the union of these boolean masks can be performed in constant time. 
% Further, the largest grammar that Domino can support currently is highly simplified C grammar with only 70 rules which leads to 25\% to 50\% overhead. 
% Domino was released quite recently and the code for the tool is not available yet for comparison. 
% }

\noindent{\bf Fixed Schema Generation.}
Many recent works perform constrained LLM decoding to ensure that the generated output follows a fixed schema of JSON or XML~\cite{sglang, beurerkellner2024guiding, willard2023efficient, jsonformer}.
When employing a fixed schema, many intermediate points in the generation process offer either a single syntactical choice (e.g., key in the JSON schema) or present only a handful of distinct options. 
In cases where only one choice exists, the generation of the next token through the LLM can be entirely skipped.
 Alternatively, when there are multiple but limited choices, techniques like speculative decoding can be used to expedite the generation process~\cite{ChenBILSJ23, leviathan2023fastinferencetransformersspeculative}.
\Tool{} does not focus on generation problems with fixed schema, it solely focuses on CFG-guided generation.
We made the same observation as in~\cite{beurerkellner2024guiding}, techniques such as speculation are not useful for CFGs where the schema is not fixed. 

% \noindent{\bf Constrained Code Generation.} 

% \noindent{\bf LLM for Code Generation.} 
% Recent research has advanced the development of LLMs for code generation~\cite{chen2021evaluating,nijkamp2023codegen,chowdhery2022palm,fried2023incoder,touvron2023llama,luo2023wizardcoder}. These models, including Codex~\cite{chen2021evaluating}, CodeGen~\cite{nijkamp2023codegen}, PaLM-Coder~\cite{chowdhery2022palm}, InCoder~\cite{fried2023incoder}, LLaMA~\cite{touvron2023llama}, WizardCoder~\cite{luo2023wizardcoder}, and others, have shown remarkable capabilities in understanding and generating code.
% Many of these models utilize fine-tuning, a process where these pre-trained models are further trained on specific code datasets to enhance their task-specific performance.
% Additionally, these models leverage in-context learning~\cite{olsson2022incontext} such as few-shot learning~\cite{brown2020language} techniques, allowing them to adapt and respond accurately with few examples in the prompt.

\section{Conclusion}
Existing methods for guiding LLMs to produce syntactically correct output have been notably slow and restrictive. 
In this paper, we present \Tool{}, an efficient and general framework to enhance LLMs' ability to generate syntactical output for various formal languages. 
During decoding, \Tool{} incrementally parses the partially generated output, computes the unparsed remainder and acceptable terminal sequences, and then leverages the remainder, accept sequences, and pre-computed DFA mask store to compute a mask to constrain the LLM's vocabulary to only syntactically valid tokens. 
We evaluated \Tool{} on generating syntactically correct JSON, SQL, Python, and Go code with different combinations of datasets, models, and tasks. 
\Tool{} eliminates syntax errors in JSON completions and significantly improves JSON schema validation over the baselines. 
Furthermore, \Tool{} reduces the number of syntax errors in generated Python and Go code by \average{} on average compared to \baseline{} generation.
We believe that our approach will pave the way for more efficient and higher-quality structured LLM generation in real-world applications.

% \input{limitations}
%
% ---- Bibliography ----
%
\bibliographystyle{ACM-Reference-Format}
\bibliography{ref}

\appendix
\section{Appendix}

\subsection{List of Symbols}
\label{sec:symbols}
\begin{tabular}{ll}
% Grammar related
    $G$ & Formal Grammar\\
    $L(G)$ & Language of a grammar \\
    $L_p(G)$ & Prefix language of a grammar\\
    $l$ & lexical tokens \\
    $l_i$ & $i$-th lexical token in the parsed output \\
    $\terminal$ & A terminal in the grammar \\
    $\terminal_i$ & Terminal type of $i$-th lexical token \\
    $\allterminals$ & Set of all terminals in the grammar \\
    $\lang^\allterminals(G)$ & Language of terminals for grammar $G$ \\
    $\lang^\allterminals_p(G)$ & Prefix language of terminals \\
    $\parser$ & Parser \\
    $\sequence$ & Sequence of terminals \\
    $\tokenizer$ & Tokenizer in an LLM \\
    $\vocab$ & Vocabulary of an LLM\\
    $\vocab_k$ & Subset of vocabulary containing acceptable tokens at $k$-th LLM generation iteration\\
    $\regex_\terminal$ & Regular expression for a terminal $\terminal$ \\
    $\regex_i$ & Regular expression corresponding to $i$-th lexical token \\
    $\greater$ & Partial order over set of terminal sequences\\

% Syncode related
    $\remainder$ & Remainder from \Tool{} parsing the partial output \\
    $\partialcode$ & Partial output at $k$-th iteration of LLM generation\\
    $\fixpartialcode$ & Parsed prefix of partial output $\partialcode$ at $k$-th iteration of LLM generation\\
    $\accepts$ & Set of accept sequences \\
    $\dmap{\alpha}$ & DFA lookup store function for terminal sequences of length $\alpha$ \\
    $\dmatch$ &  Match with DFA walk as defined in Section~\ref{sec:technical}\\
    $\pmatch$  & Partial match with regular expression\\
    $\partialparse$ & Partial parsing function \\
    $m$ & Boolean mask \\

% DFA related
    $\dfa$ & Deterministic finite automaton \\
    $\dfastates$ & States in a DFA \\
    $\alphabets$ & Set of characters i.e. alphabet for DFA \\
    $\transitions$ & Transition function in a DFA \\
    $\compute$ &  Extended transition function in a DFA\\
    $\dfastart$ & Start state of a DFA \\
    $\dfafinal$ & Set of final states in DFA\\
    $\live$ & Live states of the DFA \\
    $\dfastates_\Omega$ & Set containing all DFA states for DFAs of all terminals in the grammar\\

% Parser related 
    $\curaccepts$ & Set of terminals acceptable for current lexical token \\
    $\nextaccepts$ & Set of terminals acceptable as for next lexical token \\
    $\lex$ & Lexer function \\
    $\len$ & Length of a sequence \\
    $\curtokens$ & Current set of tokens \\
    $\psmap$ & Map for storing parser state \\
\end{tabular}

% \newcommand{\partialparse}{\textit{pparse}}

% \newcommand{\set}{\textit{set}}

% \newcommand{\str}[1]{\enquote*{\emph{#1}}} 

% % Parser related
% \newcommand{\next}{\emph{Next}\xspace}
% \newcommand{\follow}{\emph{Follow}\xspace}
% % \newcommand{\lexertokens}{\textit{LT}}

% % regex
% \newcommand{\pmatch}{\textit{pmatch}}
% \newcommand{\len}{\textit{len}}

% % DFA related
% \newcommand{\dmatch}{\textit{dmatch}}

\newpage
\subsection{Proofs for Theorems}
\label{sec:proofs}

\eq* 
\begin{proof}
\begin{enumerate} [label=(\alph*)]
    \item First we prove $\dmatch(w, \dfastart^{\terminal_{f+1}}, \sequence^p) \implies \pmatch(w, \regex_\sequence)$
    We prove this using induction on the length $i$ of $w$. \\
    For $i=0$, $\pmatch(w, \regex_\sequence)$ is trivially true. \\
    Now, we assume that for $w$ of length $i < k$, $\dmatch(w, \dfastart^{\terminal_{f+1}}, \sequence^p) \implies \pmatch(w, \regex_\sequence)$. \\
    We consider $w$ of length $k$ and $\dmatch(w, \dfastart^{\terminal_{f+1}}, \sequence^p)$.\\ 
    We consider 3 conditions from Definition~\ref{def:dmatch}. \\
    If condition 1 is true, $\compute_{\terminal_{f+1}}(w, \dfastart^{\terminal_{f+1}}) \in \live(Q_{\terminal_{f+1}})$. 
    Let $q_1 = \compute(w, \dfastart^{\terminal_{f+1}})$. 
    By Definition~\ref{def:live}, $\exists w_1 \text{ s.t. } \compute_{\terminal_{f+1}}(w_1, q_1) \in F_{\terminal_{f+1}}$.
    Hence, 
    \[\compute(w.w_1, \dfastart^{\terminal_{f+1}}) \in F_{\terminal_{f+1}} \implies w.w_1 \in \lang(\regex_{\terminal_{f+1}})\]
    We assume that each terminal $\lang(\terminal_i)$ is non-empty. Hence, 
    \[\exists w_2 \in \lang(\regex_{\sequence^p}) \implies w.w_1.w_2 \in \lang(\regex_{\sequence}) \]
    Hence, by condition 2 from Definition~\ref{def:pmatch}, $\pmatch(w, \regex_\sequence)$. 
    
    If condition 2 is true, $\exists w_1, w_2$ such that $w_1.w_2 = w$, $\compute_{\terminal_{f+1}}(w_1, \dfastart^{\terminal_{f+1}}) \in F \text{ and } \sequence^p= \{\} $. 
    Here, $w_1 \in \lang(\regex_{\terminal_{f+1}})$. 
    Since $\sequence^p = \{\}$, $\regex_\sequence = \regex_1$, and hence, $w_1 \in \lang(\regex_\sequence)$. 
    Hence by condition 1 from Definition~\ref{def:pmatch}, $\pmatch(w, \regex_\sequence)$. 

    If condition 3 is true, $\exists w_1, w_2$ such that $w_1.w_2 = w$, $\compute_{\terminal_{f+1}}(w_1, \dfastart^{\terminal_{f+1}}) \in F_{\terminal_{f+1}}$, \\ 
    and $\dmatch(w_2, q_{0}^{\terminal_{f+2}}, \{\terminal_{f+3} \dots \terminal_{f+d}\}) = \true$. 
    \[\compute_{\terminal_{f+1}}(w_1, \dfastart^{\terminal_{f+1}}) \in F_{\terminal_{f+1}} \implies w_1 \in \lang(\regex_{\terminal_{f+1}}) \]
    Since length of $w_2 < k$, by our induction hypothesis, $\pmatch(w_2, \regex_{\sequence^p}) = \true.$ By Definition~\ref{def:pmatch}, there are two possibilities. 
    Suppose $\exists w_2 = w_3.w_4$ such that $w_3 \in \lang(\regex_{\sequence^p})$. \\
    \[w_1.w_3 \in \lang(\regex_{\sequence}) \implies \pmatch(w, \regex_{\sequence}) = \true \] 
    Alternatively, if $\exists w_3$ such that $w_2.w_3 \in \lang(\regex_{\sequence^p})$ 
    \[ w_1.w_2.w_3 \in \lang(\regex_\sequence) \implies \pmatch(w, \regex_{\sequence}) = \true
    \]
    Hence, our induction proof is complete and $\pmatch(w, \regex_{\sequence}) = \true$

    \item Next we prove $\pmatch(w, \regex_\sequence) \implies \dmatch(w, \dfastart^{\terminal_{f+1}}, \sequence^p)$
    We prove this using induction on the length $i$ of $w$. \\
    For $i=0$, $\dmatch(w, \dfastart^{\terminal_{f+1}}, \sequence^p)$ is trivially true. \\
    Now, we assume that for $w$ of length $i < k$, $ \pmatch(w, \regex_\sequence) \implies \dmatch(w, \dfastart^{\terminal_{f+1}}, \sequence^p)$ \\
    Now we consider $w$ of length $k$ and $\pmatch(w, \regex_\sequence)$.\\ 
    By Definition~\ref{def:pmatch}, there are two possible conditions \\
    
    \textbf{Case 1:} $\exists w_1 \in \alphabets^*, w_2 \in \alphabets^+\text{ such that } w = w_1.w_2$ and $w_1 \in \lang(\regex_\sequence)$ \\
    Hence, $\exists w_3, w_4 \text{ such that } w_1 = w_3.w_4 $ and $w_3 \in \lang(\regex_{\terminal_{f+1}})$ and $w_4 \in \lang(\regex_{\sequence^p})$. By induction hypothesis, 
    \[
        \pmatch(w_4.w_2, \regex_{\sequence^p}) \implies \dmatch(w_4 w_2, \{\terminal_{f+2}, \terminal_{f+3} \dots \terminal_{f+d} \})
    \]
    Since $w = w_3.w_4.w_2$ and
    \[
        w_3 \in \lang(\regex_{\terminal_{f+1}}) \implies \compute_{\terminal_{f+1}}(w_3, \dfastart^{\terminal_{f+1}}) \in F_{\terminal_{f+1}}
    \]
    Hence, by condition 3 in Definition~\ref{def:dmatch}, $\dmatch(w, \dfastart^{\terminal_{f+1}}, \sequence^p)$  

    \textbf{Case 2:} $\exists w_1$ such that $w.w_1 \in \lang(\regex_\sequence)$ \\
    Hence, $\exists w_2, w_3 \text{ s.t } w.w_1 = w_2.w_3 $ and $w_2 \in \lang(\regex_{\terminal_{f+1}})$ and $w_3 \in \lang(\regex_{\sequence})$ \\
    Now there are two possibilities, either $w$ is prefix of $w_2$ or $w_2$ is prefix of $w_2$ \\
    Supoose $w$ is prefix of $w_2$, then $\compute_{\terminal_{f+1}}(w, \dfastart^{\terminal_{f+1}}) \in \live(\dfastates_{\terminal_{f+1}})$ and hence by Definition~\ref{def:dmatch}, $\dmatch(w, \dfastart^{\terminal_{f+1}}, \sequence^p)$
    Alternatively, if $w_2$ is prefix of $w$ then $\exists w_4 \text{ s.t. } w = w_2 w_4$ \\
    Hence, $w_4.w_1 = w_3 \in \lang(\regex_{\terminal_{f+1}})$ and thus $\pmatch(w_4, \regex_{\sequence^p})$ \\
    By induction hypothesis $\dmatch(w_4, \dfastart^{\terminal_{f+2}}, \{\terminal_{f+3}, \terminal_4 \dots \terminal_{f+d} \})$ \\
    and since $w = w_2.w_4 $ and $\compute_{\terminal_{f+1}}(w_2, \dfastart^{\terminal_{f+1}}) \in F_{\terminal_{f+1}}$. 
    We get $\dmatch(w, \dfastart^{\terminal_{f+1}}, \sequence^p)$
    
\end{enumerate}

\end{proof}

\dm*
\begin{proof}
\begin{enumerate}[label=(\alph*)]
    \item First, we prove $\dmatch(t, q, \sequence) \implies \dmatch(r.t, \dfastart^{\terminal}, \sequence)$. \\
    From Definition~\ref{def:dmatch}, either of the 3 conditions hold true for $\dmatch(t, q, \sequence)$. \\
    
    If condition 1 is true then 
\[\compute_{\terminal_1}(t, q) \in \live(Q_{\terminal}) \implies \compute_{\terminal}(r.t, \dfastart^{\terminal}) \in \live(Q_{\terminal}) \implies \dmatch(r.t, \dfastart^{\terminal}, \sequence)\]

    If condition 2 is true,  $\exists w_1, w_2$ such that $w_1.w_2 = t$, $\compute_{\terminal}(w_1, q) \in F \text{ and } \sequence= \{\}$. Therefore,  
\[\compute_{\terminal}(r.w_1, q) \in F \implies \dmatch(r.t, \dfastart^{\terminal}, \sequence)\]

    If condition 3 is true, $\exists w_1, w_2$ such that $w_1.w_2 = t$, $\compute_{\terminal}(w_1, q) \in F$ \\ and $\dmatch(w_2, q_{0}^{\terminal_1}, \{\terminal_2 \dots \terminal_d\}) = \true$. Therefore, \[\compute_{\terminal}(r.w_1, q) \in F \implies \dmatch(r.t, \dfastart^{\terminal}, \sequence)\]

Therefore, in all cases, $\dmatch(rt, \dfastart^{\terminal}, \sequence)$ must hold. 

    \item Now, we prove $\dmatch(rt, \dfastart^{\terminal}, \sequence)  \implies \dmatch(t, q, \sequence)$.

     From Definition~\ref{def:dmatch}, either of the 3 conditions hold true for $\dmatch(r.t, \dfastart^{\terminal}, \sequence)$. \\

     If condition 1 is true then 
    \[
    \compute_{\terminal_1}(r.t, \dfastart^{\terminal}) \in \live(Q_{\terminal}) \implies \compute_{\terminal}(t, q) \in \live(Q_{\terminal}) \implies \dmatch(t, q, \sequence)
    \]

    If condition 2 is true,  $\exists w_1, w_2$ such that $w_1.w_2 = r.t$, $\compute_{\terminal}(w_1, \dfastart^{\terminal}) \in F \text{ and } \sequence= \{\}$. Since no prefix of $r$ is accepted by $\lang(\terminal)$, $\exists w_3 \text{ s.t. } w_3 w_4 = t$ and
    \[
    \compute_{\terminal}(w_3, q) \in F \implies \dmatch(t, q, \sequence)
    \]

    If condition 3 is true, $\exists w_1, w_2$ such that $w_1.w_2 = r.t$, $\compute_{\terminal}(w_1, \dfastart^{\terminal}) \in F$ \\ 
    and $\dmatch(w_2, q_{0}^{\terminal_1}, \{\terminal_2 \dots \terminal_d\}) = \true$. Since no prefix of $r$ is accepted by $\lang(\terminal)$, $\exists w_3 \text{ s.t. } w_3 w_4 = t$ and  
    \[\compute_{\terminal}(w_3, q) \in F \implies \dmatch(t, q, \sequence)\]

    Therefore, in all cases, $\dmatch(t, q, \sequence)$ must hold. 
    
\end{enumerate}

\end{proof}

\Sound*
\begin{proof}
% Since $\partialcode.t \in \lang_p(G)$, $\exists w \text{ s.t. } \partialcode.t.w \in \lang(G)$. \\
Let $r, \sequence^{\square} = \partialparse(\partialcode)$ where $\sequence^{\square} = \terminal_1, \terminal_2 \dots \terminal_{f}$ and let $r_1, \sequence_1 = \partialparse(\partialcode.t)$ where $\sequence_1 = \terminal_1, \terminal_2 \dots \terminal_{f} \dots \terminal_{f+g}$ \\
Hence, we can split $r.t$ such that for $w \in \alphabets^*$, $r.t = w.r_1$ and $w \in \lang(\terminal_{f+1} \dots \terminal_{f+g})$ \\
There are two possible cases: \\
\noindent\textbf{Case 1:} $g < d$ \\
\[
    w \in \lang(\terminal_{f+1} \dots \terminal_{f+g})
\]
\[
    \implies w \in \lang_p(\terminal_{f+1} \dots \terminal_{f+g})
\]
By our assumption on $\accepts_d$ there must exist $\sequence_2 = \terminal_{f+1} \dots \terminal_{f+d}$ s.t. $\terminal_{f+1} \dots \terminal_{f+g}$ is prefix of $\sequence_2$. Hence, 
\[
    \implies w \in \lang_p(\sequence_2)
\]
\[
    \implies \pmatch(r.t, \sequence_2)
\]

\noindent\textbf{Case 2:} $g \geq d$ \\
Since we assume that $\accepts_d$ contains all possible accept sequence of length $d$, $\sequence_2 = \terminal_{f+1} \dots \terminal_{f+d}$ must be contained in $\accepts_d$ \\
Hence, $\exists w_1, w_2 \in \alphabets^*$ such that $w = w_1.w_2$ and
\[
    w_1 \in \lang(\sequence_2)
\]
\[
    \implies w \in \lang_p(\sequence_2)
\]
\[
    \implies \pmatch(r.t, \sequence_2)
\]
In both cases, $\pmatch(r.t, \sequence_2)$. Using Lemma~\ref{lemma:eq},
\[
    \implies \dmatch(r.t, \dfastart^{\terminal_{f+1}} , \{\terminal_{f+2} \dots \terminal_{f+d}\}) 
\]
Using Lemma~\ref{lemma:dmatch} if $q = \compute_{\terminal_{f+1}}(r, \dfastart^{\terminal_{f+1}})$
\[
    \dmatch(r.t, \dfastart^{\terminal_{f+1}} ,\{\terminal_{f+2} \dots \terminal_{f+d}\}) \implies \dmatch(t, q, \{\terminal_{f+2} \dots \terminal_{f+d}\})
\]
Here from Definition~\ref{def:lookup}, if $\dmap{d-1}(q, \{\terminal_{f+2} \dots \terminal_{f+d}\}) = m_2$ then $t \in \set(m_2)$. \\
Since $m_2 \subseteq m$, we have our result $t \in \set(m)$.

\end{proof}

\Porder*
\begin{proof}
Since $\forall \sequence_2 \in \accepts_2 \exists \sequence_1 \in \accepts_1 \exists \sequence_3 \in \allterminals^*  $  
s.t. $\sequence_2 = \sequence_1.\sequence_3$, Hence 
\[
\pmatch(w, \regex_{\sequence_2}) \implies \pmatch(w, \regex_{\sequence_1})
\]
Hence, for the mask $\set(m_2) \subseteq \set(m_1)$
\end{proof}

\Complete*
For the simplicity of presenting the proof, we assume that $d > 2$. 

Since $t \in \set(m)$ for some $\sequence_1 = \{\terminal_{f+1}, \terminal_{f+2} \dots \terminal_{f+d}\} \in \accepts$
\[
    \implies \dmatch(t, q , \{\terminal_{f+2} \dots \terminal_{f+d}\}) \implies \dmatch(r.t, \dfastart^{\terminal_{f+1}}, \{\terminal_{f+2} \dots \terminal_{f+d}\}) 
\]
\[
    \implies \pmatch(r.t, \{\regex_{\terminal_{f+1}}.\regex_{\terminal_{f+2}} \dots \regex_{\terminal_{f+d}}\})
\]
By Definition~\ref{def:pmatch}, there are two possible cases:

\begin{enumerate}
    \item $\exists w_1 \in \alphabets^*, w_2 \in \alphabets^+$ such that $r.t = w_1.w_2 $ and $w_1 \in \lang(\regex_{\terminal_{f+1}}.\regex_{\terminal_{f+2}} \dots \regex_{\terminal_{f+d}})$ \\
    We show that this case is not possible since our terminal sequence $\sequence_1$ is long enough that no prefix of $r.t$ cannot be in $\lang(\regex_{\terminal_{f+1}}.\regex_{\terminal_{f+2}} \dots \regex_{\terminal_{f+d}})$ \\
    We can infer that $\len(w_1) < \len(r.t) \implies \len(w_1) < \len(r) + \len(t)$ \\
    Further, from the assumption $ d > \textit{len}(t)$, we have
    \[
        \len(w_1) < d + \len(r) 
    \]    
    Firstly, note that $r \not\in \lang(\regex_{\terminal_{f+1}}.\regex_{\terminal_{f+2}})$ by the definition of remainder $r$ \\
    Note that we assume no terminal contains empty string i.e. $\epsilon \not \in \lang(\regex_{\terminal_i})$ \\
    Hence, every string in $\lang(\regex_{\terminal_{f+2}} \dots \regex_{\terminal_{f+d}})$ should have length at least $d-1$  \\

    % it's not that clear 
    Clearly, $r$ is prefix of $w_1$. Let $w_3 \in \alphabets^*$, $r.w_3 = w_1$ and hence $\len(w_3) > d-1$ \\
    Hence, 
    \[
        \len(r) + d - 1 < \len(w_1)
    \]
    \[ 
    \len(r) + d - 1 < \len(w_1) < d + \len(r) 
    \]
    This is not possible and hence such $w_1$ cannot exist. 
    
    \item $\exists w_1 \in \alphabets^*$ such that $r.t.w_1 \in \lang(\regex_{\terminal_{f+1}}.\regex_{\terminal_{f+2}} \dots \regex_{\terminal_{f+d}})$

    By Definition~\ref{def:pparse}, we have $\sequence^{\square}, r = \partialparse(\partialcode)$ s.t $\partialcode = \fixpartialcode.r$, $\sequence^{\square} = \terminal_{1}, \terminal_{2} \dots \terminal_{f}$ $\fixpartialcode \in \lang(\regex_{\terminal_{1}}.\regex_{\terminal_{2}} \dots \regex_{\terminal_{f}})$. \\
    Let $\sequence_1 = \terminal_{f+1}, \terminal_{f+2} \dots \terminal_{f+d}$ \\
    Since, $\partialcode. t = \fixpartialcode.r.t$, $\fixpartialcode \in \lang(\sequence^{\square})$ and $r.t.w_1 \in \lang(\sequence_1)$, we have 
    \[
        \fixpartialcode.r.t.w_1 \in \lang(\sequence^{\square}.\sequence_1)
    \]
    \[
        \partialcode.t.w_1 \in \lang(\sequence^{\square}.\sequence_1)
    \]
     By Definition~\ref{def:acc} of accept sequence, $\sequence^{\square} . \sequence_1 \in \lang_p^\allterminals(G)$, Hence
     \[
        \partialcode.t.w_1 \in \lang_p(G) \implies \partialcode.t \in \lang_p(G)
     \]     
\end{enumerate}
Thus, our proof is complete and $\partialcode.t \in \lang_p(G)$

\newpage
\subsection{Incremental Parsing Algorithm}
\label{sec:incparse}
\begin{wrapfigure}{R}{0.53\textwidth}
\vspace{-.2in}
\begin{minipage}{0.55\textwidth}

\begin{algorithm}[H]
\caption{Incremental Parsing}
\label{alg:parse}
\textbf{Inputs:} $\partialcode$: partial output, $\psmap$: state map
\begin{algorithmic}[1]
\Function{Parse}{$\partialcode$}
\State $l_1, l_2 \dots l_f \gets \text{\lex}(\partialcode)$
% \State $j \gets \text{Len}(L)$
\State $\gamma, S_{\gamma} \gets \text{RestoreState}(\psmap, L)$
\State $\parser \gets \text{Initialize}(S_\gamma)$
\State $\textit{parsed} \gets l_1.l_2\dots l_{\gamma-1}$
\For{$l_i \in l_\gamma, l_{\gamma+1} \dots l_f$}
\State $\text{\next}(\parser, l_i)$
\If{$\parser.state = Error$}
\State break
\EndIf
\State $\textit{parsed} \gets \textit{parsed} + l_i$
\State $\curaccepts \gets \nextaccepts$
\State $\nextaccepts \gets \text{\follow}(\parser)$
\State $S_i \gets \parser.state$
\State $\text{Store}(\psmap, \textit{parsed}, S_i)$
\EndFor
% \State $\text{output} \gets \text{decode}(T, \curtokens)$
% \State $r \gets \partialcode[\textit{parsed}:]$
\If{$\partialcode = \textit{parsed}$}
\State $r = l_f$
\State $\accepts \gets \{\terminal_f, \nextaccepts[0]\}, \{\terminal_f, \nextaccepts[1]\} \dots \}$
\State \qquad $\cup \{\curaccepts[0]\}, \{\curaccepts[1]\} \dots \}$
\Else
\State $r = \partialcode - \textit{parsed}$
\State $\accepts \gets \{\nextaccepts[0]\}, \{\nextaccepts[1]\} \dots $
\EndIf

\State \Return $\accepts, r$
\EndFunction
%\item[]
\end{algorithmic}
\end{algorithm}

\end{minipage}
\vspace{-0.2in}
\end{wrapfigure}
Our parsing algorithm achieves incrementality in LLM generation by utilizing a map $\psmap$ to store the parser state. 
This map associates a list of lexical tokens with the corresponding parser state after parsing those tokens. 
Frequently, in subsequent LLM generation iterations, the count of lexical tokens remains the same—either the next vocabulary token is appended to the final lexical token, or it increases. 
Although uncommon, there are cases where the number of parsed lexical tokens may decrease during iterations.
For example, in Python, an empty pair of double quotes, "", is recognized as a complete lexical token representing an empty string. 
On the other hand, """ serves as a prefix to a docstring, constituting an incomplete parser token. 
Consequently, the addition of a single double quote " reduces the overall count of lexical tokens in these iterations.
We observe that while the total count of lexer tokens at the end may undergo slight changes during these iterations, the majority of prefixes of the parsed lexical tokens remain consistent. 
Hence, we establish a mapping between lists of prefixes of lexical tokens and the corresponding parser state after parsing those tokens. 
Subsequently, when parsing a new list of lexer tokens, we efficiently determine the maximum length prefix of the lexer token list that is already present in $\psmap$. 
This incremental approach significantly reduces the complexity of our parsing algorithm. 
% The pseudocode for our parsing algorithm is presented in Appendix~\ref{sec:incparsealgo}.

% Lexing not incremental
While it could be feasible to introduce incrementality in the lexing operation, our experiments revealed that lexing consumes insignificant time in comparison to parsing. 
As a result, we opted to focus only on performing parsing incrementally.

% base parser
Our incremental parsing algorithm uses a standard non-incremental base parser $\parser$ that maintains a parser state and supports two functions \next and \follow. 
The \next function accepts the next lexer token and then updates the parser state. 
The \follow function returns a list of acceptable terminals at the current parser state. 
These functions are present in common parser generator tools \cite{lark,antlr}.

% Algorithm description
The Algorithm~\ref{alg:parse} presents our incremental parsing algorithm. 
The algorithm utilizes a lexer to tokenize the partial output.
The function RestoreState is used to restore the state of the parser to the maximal matching prefix of lexical tokens that exist in $\psmap$.
The main loop iterates through the tokens and maintains a parser state map. 
For each token, it updates the parser state, stores the mapping in $\psmap$, and retrieves the next set of acceptable terminals. 
The process continues until the end of the partial output. 
The algorithm returns accept sequences $\accepts$ and remainder $r$.

\subsection{Ablation Studies}
In this section, we perform an ablation study for incremental parsing and max new tokens. 
% In Appendix~\ref{sec:lalr_ablation} we perform ablations comparing LALR and LR parsers as a base parser.

\label{sec:ablation}
\noindent{\bf Incremental Parsing.}
We compare the runtime efficiency of utilizing incremental parsing over re-running parsing from scratch in \Tool{}. 
We run inference on \codegen{} with \Tool{} using incremental parsing and parsing from scratch on Python prompts from the HumanEval dataset. 
We generate $n = 1$ samples and control the max new tokens in the code completion. 
Our results are presented in Figure~\ref{fig:inc_parser}, where the x-axis represents the max new tokens and the y-axis represents the average generation time, in seconds, with and without incremental parsing. 
As shown in the figure, the average generation time when re-parsing from scratch increases significantly as the maximum length of code that the LLM can generate increases. 
On the other hand, the average generation time increases slowly with incremental parsing. 
For max new tokens = 300, \Tool{} with incremental parsing achieves 9x speedup over running parsing from scratch. 
Our results collectively demonstrate that augmenting \Tool{} with incremental parsing dramatically improves generation time, especially when generating longer completions.
% \TBD{Increase the font of the labels in Fig 10! }
%---------------------------------------------------------------%
\begin{figure}[!htbp]
\centering
\vspace{-.1in}
\begin{subfigure}[b]{0.48\textwidth}
    \includegraphics[width=1.10\textwidth]{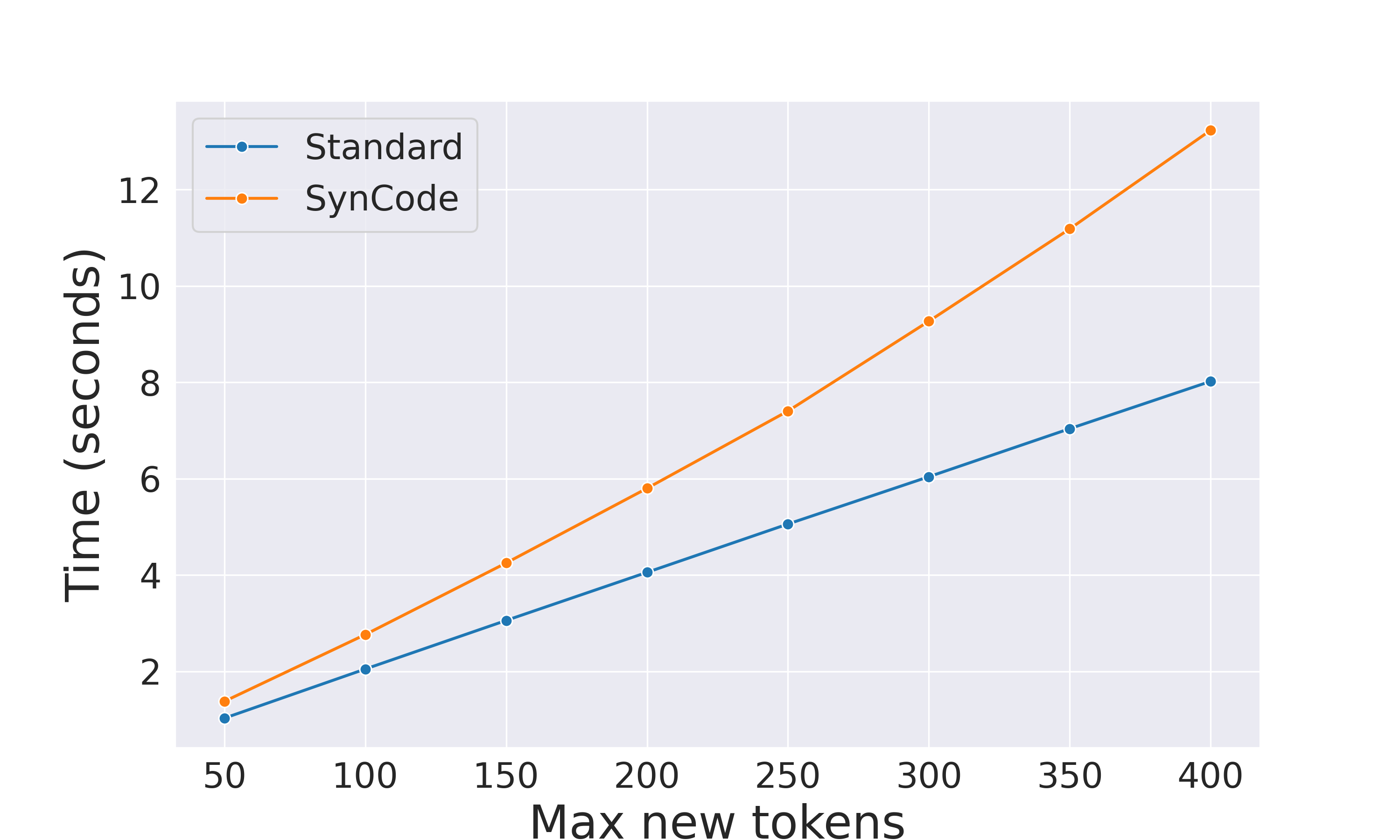}
    \caption{Average generation time for different max new tokens}
    \label{fig:SynCode_vs_Standard}
\end{subfigure}
\hspace{3mm}
\begin{subfigure}[b]{0.48\textwidth}
    \includegraphics[width=1.10\textwidth]{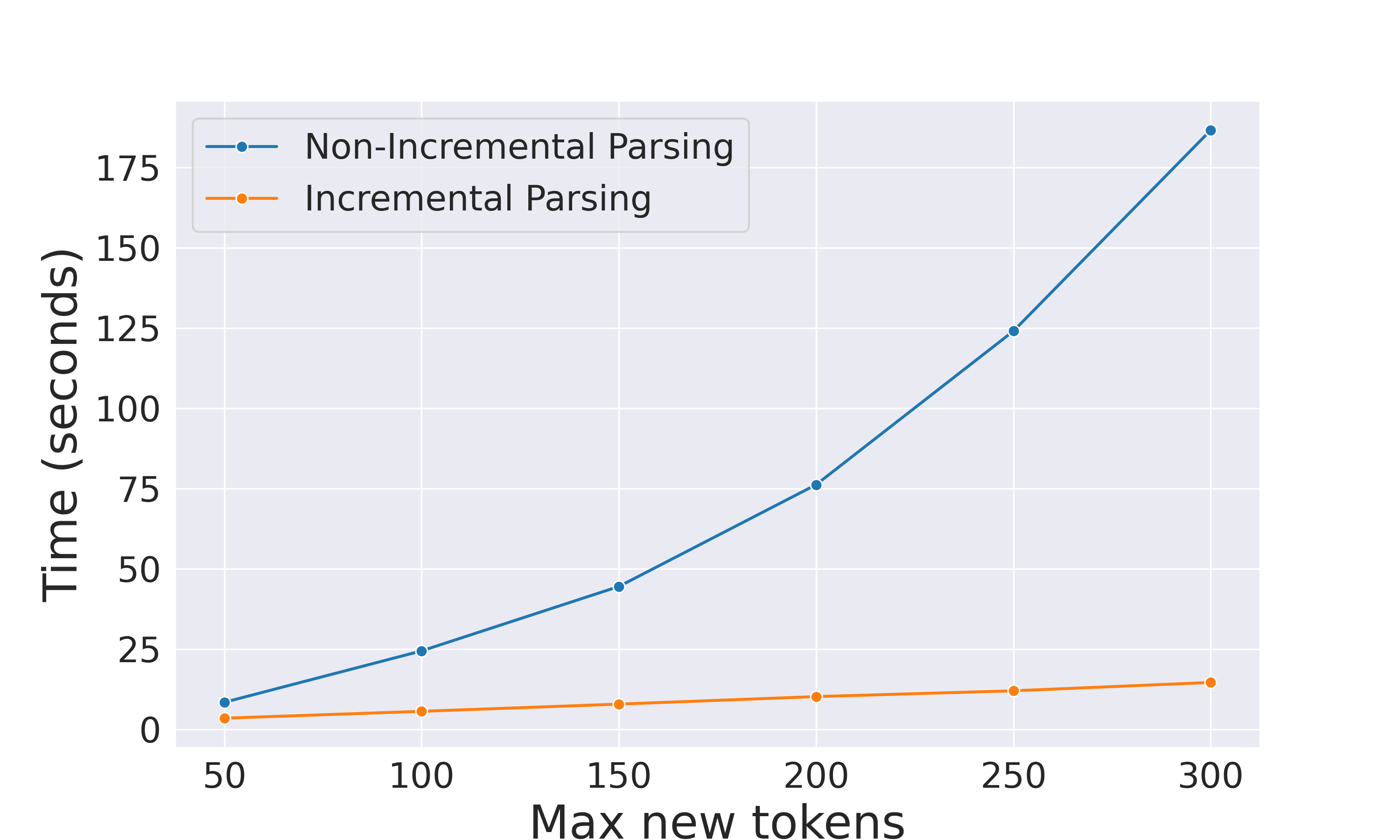}
    \caption{Average generation time with and without incremental parser}
    \label{fig:inc_parser}
\end{subfigure}
 \vspace{-.1in}
\hfill
\caption{Ablation studies on \codegen{} model.}
\vspace{-0.05in}
\end{figure} 

%---------------------------------------------------------------%
\noindent{\bf Max New Tokens.} 
We conduct an ablation study into the relationship between the maximum length of code that the LLMs can generate and generation times. 
We used Python prompts from the HumanEval dataset and leveraged \codegen{} to generate the code completions, both with and without the augmentation of the \Tool{}.
%We compute the average generation time, which is measured by dividing the total generation time for the dataset by the total number of problems. 
As shown in Figure~\ref{fig:SynCode_vs_Standard}, as we increase the max new tokens, we observe a corresponding increase in generation time. 

% \input{tables/parsers_runtime}
% \noindent{\bf LR(1) and LALR(1).} 
% % \subsection{LR(1) and LALR(1)}
% \label{sec:lalr_ablation}
% We compare the runtime efficiency of utilizing LR(1) and LALR(1) parsing in \Tool{}.
% We run inference on \codegen{}, \wizard{}, and \llama{} with \Tool{} with LALR(1) parser and with LR(1) parser for Python and Go on the HumanEval dataset. 
% We generate a single sample $(n = 1)$ per prompt with the max new tokens parameter set to 200. 
% Table~\ref{table:parsers_runtime} reports the average time taken to generate each prompt from the datasets. 
% As shown in Table~\ref{table:parsers_runtime}, we observe that \Tool{} with LR(1) parser outperforms the LALR(1) parser with respective overheads of 1.22x on average and 1.76x on average compared to the \baseline{} generation result from~\ref{table:runtime}.

\subsection{Few-Shot Prompting}
\label{sec:few_shot}
\begin{table}
    \tablesize
    \centering
    \vspace{-0.1in}
    \caption{\Tool{} on few-shot prompting}
    \begin{tabular}{@{}ll ccc@{}}
        \toprule
        Architecture & Error Type & \Baseline{} & \Tool{} & $\downarrow$ \\
        \hline
         CodeGen-350M  & Syntax & 53 & 0 & 100\% \\
          & Indentation & 15 & 3 & 80\% \\
        WizardCoder-1B  & Syntax & 40 & 2 & 95\% \\
        & Indentation & 22 & 1 & 95\% \\
        Llama-7B & Syntax & 110 & 0 & 100\% \\
        & Indentation & 40 & 5 & 88\% \\
        
        \bottomrule
    \end{tabular}
    \vspace{-0.1in}
    \label{tab:few_shot}
\end{table}
Few-shot prompting \cite{ren2018metalearning} refers to the idea that language models do not need to be specifically
trained for a downstream task such as classification or question answering. Rather, it is sufficient to train them on broad text-sequence prediction datasets and to provide context in the form of examples when invoking them. We study the performance of utilizing \Tool{} on few-shot prompting code generation tasks. We selected Python few-shot examples from the MBXP dataset and generated code completions with \codegen{}, \llama{}, and \wizard{} with \Tool{} and the \baseline{} no-masking generation. We present our results in Table \ref{tab:few_shot}. The columns standard and SynCode represent the total number of errors of a particular Error Type of LLM-generated code completions to problems in a particular dataset when utilizing that respective generation approach. The column ↓ represents the percentage reduction from the standard column to the SynCode column. As shown in the table, \Tool{} exhibits effectiveness not only in zero-shot but also in the context of few-shot prompting tasks. This signifies the versatility of \Tool{} in enhancing code generation across different prompt engineering techniques.

\subsection{\Tool{} Syntax Errors}
\label{sec:syncode_error}
\begin{figure}[tbh]
    \centering
    \includegraphics[width=\textwidth]{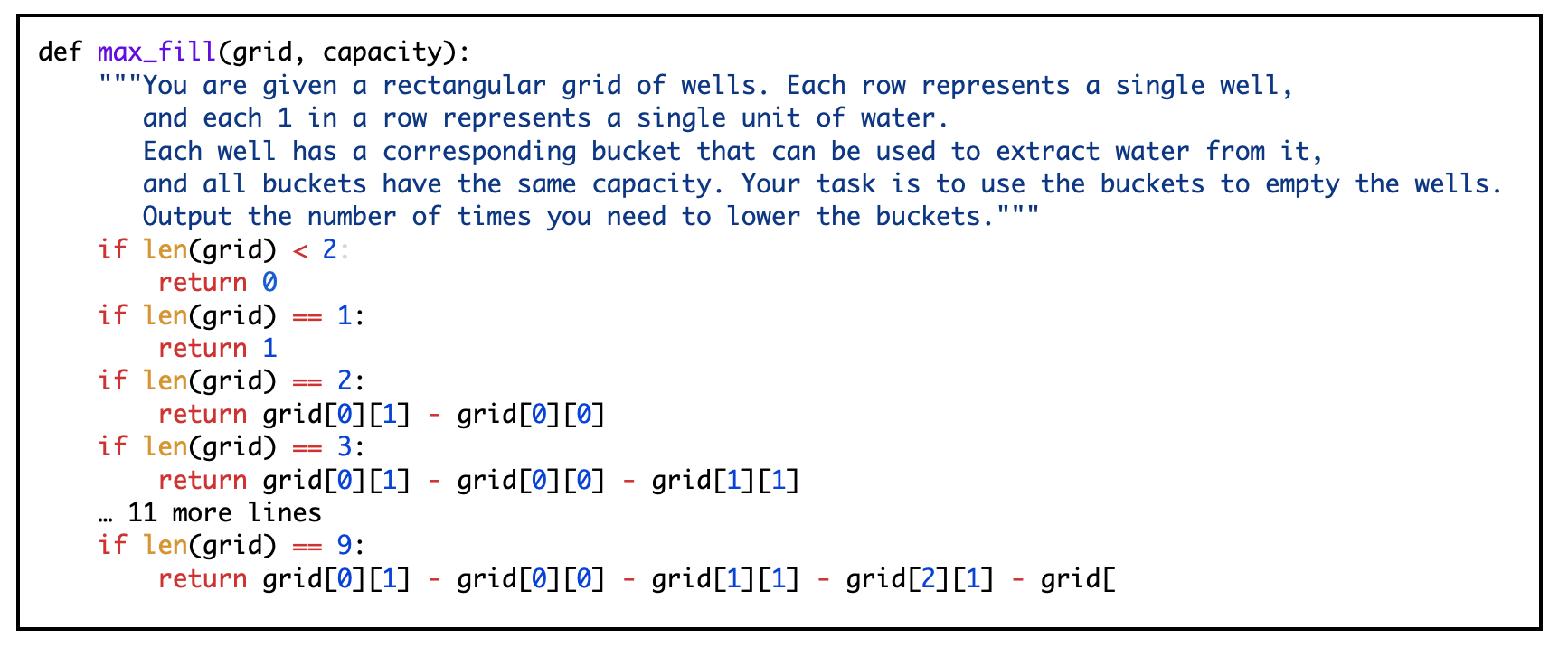}
    \caption{Syntactically Incorrect \Tool{} Program}
    \label{fig:syncode_error}
\end{figure}

Figure \ref{fig:syncode_error} presents an example of when the \Tool{} augmented LLM fails to generate a complete program within the maximum token limit for a problem from the HumanEval dataset. 
While the code is a syntactically correct partial program, it is not a syntactically correct complete program.  
Recall, that \Tool{} guarantees completeness for syntactically correct partial programs but does not guarantee termination with a syntactically correct complete program.

\newpage
\subsection{Prompts Used in the Evaluation}
\label{sec:prompts}
\lstdefinestyle{myGrammarStyle}{
    basicstyle=\scriptsize\ttfamily, % Reduce font size
    commentstyle=\color{green},
    keywordstyle=\color{blue},
    stringstyle=\color{orange},
    numbers=left, % Line numbers on left
    numberstyle=\tiny\color{gray}, % Line numbers styling
    breaklines=true, % Wrap long lines
    frame=single, % Frame around the code
    framesep=3pt, % Adjust frame separation
    xleftmargin=5pt, % Adjust left margin
    xrightmargin=5pt, % Adjust right margin
    backgroundcolor=\color{yellow!0}, % Background color
    tabsize=2, % Tab size
    captionpos=b, % Caption position bottom
    aboveskip=5pt, % Reduce space above the listing
    belowskip=5pt, % Reduce space below the listing
    linewidth=0.9\linewidth, % Set custom line width to less than text width
    escapeinside={(*@}{@*)}, % for escaping to LaTeX
}

\begin{lstlisting}[style=myGrammarStyle, caption= Example original JSON Prompt from the JSON-Mode-Eval dataset~\cite{jsoneval}. The prompt consists of a system message that specifies a schema and a user message requesting JSON output given certain parameters.]
<s>[INST] <<SYS>>
You are a helpful assistant that answers in JSON. Here's the json schema you must adhere to:
<schema>
{'title': 'Person', 'type': 'object', 'properties': {'firstName': {'type': 'string', 'description': "The person's first name."}, 'lastName': {'type': 'string', 'description': "The person's last name."}, 'age': {'description': 'Age in years which must be equal to or greater than zero.', 'type': 'integer', 'minimum': 0}}, 'required': ['firstName', 'lastName', 'age']}
</schema>

<</SYS>>

Please generate a JSON output for a person's profile that includes their first name, last name, and age. The first name should be 'Alice', the last name 'Johnson', and the age 35. [/INST]

\end{lstlisting}
\label{prompt:json_prompt}
\lstdefinestyle{myGrammarStyle}{
    basicstyle=\scriptsize\ttfamily, % Reduce font size
    commentstyle=\color{green},
    keywordstyle=\color{blue},
    stringstyle=\color{orange},
    numbers=left, % Line numbers on left
    numberstyle=\tiny\color{gray}, % Line numbers styling
    breaklines=true, % Wrap long lines
    frame=single, % Frame around the code
    framesep=3pt, % Adjust frame separation
    xleftmargin=5pt, % Adjust left margin
    xrightmargin=5pt, % Adjust right margin
    backgroundcolor=\color{yellow!0}, % Background color
    tabsize=2, % Tab size
    captionpos=b, % Caption position bottom
    aboveskip=5pt, % Reduce space above the listing
    belowskip=5pt, % Reduce space below the listing
    linewidth=0.9\linewidth, % Set custom line width to less than text width
    escapeinside={(*@}{@*)}, % for escaping to LaTeX
    morekeywords={Output only JSON}, % Add specific words to be highlighted
    keywordstyle=\color{blue}\bfseries, % Style for additional keywords
}

\begin{lstlisting}[style=myGrammarStyle, caption= Example JSON prompt from the JSON-Mode-Eval dataset ~\cite{jsoneval} after augmentation with an explicit request to only output JSON. ]
<s>[INST] <<SYS>>
You are a helpful assistant that answers in JSON. Here's the json schema you must adhere to:
<schema>
{'title': 'Person', 'type': 'object', 'properties': {'firstName': {'type': 'string', 'description': "The person's first name."}, 'lastName': {'type': 'string', 'description': "The person's last name."}, 'age': {'description': 'Age in years which must be equal to or greater than zero.', 'type': 'integer', 'minimum': 0}}, 'required': ['firstName', 'lastName', 'age']}
</schema>

<</SYS>>

Please generate a JSON output for a person's profile that includes their first name, last name, and age. The first name should be 'Alice', the last name 'Johnson', and the age 35. Output only JSON. [/INST]
\end{lstlisting}
\label{prompt:json_prompt_explicit}

% (*@\hl{Output only JSON.}@*)
\lstdefinestyle{myGrammarStyle}{
    basicstyle=\scriptsize\ttfamily, % Reduce font size
    commentstyle=\color{green},
    keywordstyle=\color{blue},
    stringstyle=\color{orange},
    numbers=left, % Line numbers on left
    numberstyle=\tiny\color{gray}, % Line numbers styling
    breaklines=true, % Wrap long lines
    frame=single, % Frame around the code
    framesep=3pt, % Adjust frame separation
    xleftmargin=5pt, % Adjust left margin
    xrightmargin=5pt, % Adjust right margin
    backgroundcolor=\color{yellow!0}, % Background color
    tabsize=2, % Tab size
    captionpos=b, % Caption position bottom
    aboveskip=5pt, % Reduce space above the listing
    belowskip=5pt, % Reduce space below the listing
    linewidth=0.9\linewidth, % Set custom line width to less than text width
    escapeinside={(*@}{@*)}, % for escaping to LaTeX
}

\begin{lstlisting}[style=myGrammarStyle, caption= text-2-SQL prompt.]

db_id: concert_singer
db_info: # stadium ( stadium_id , location , name , capacity , highest , lowest , average )
# singer ( singer_id , name , country , song_name , song_release_year , age , is_male )
# concert ( concert_id , concert_name , theme , stadium_id , year )
# singer_in_concert ( concert_id , singer_id )
# concert.stadium_id = stadium.stadium_id
# singer_in_concert.singer_id = singer.singer_id
# singer_in_concert.concert_id = concert.concert_id

question: How many singers do we have? Only output the SQL query. 
SQL:

\end{lstlisting}
\label{prompt:sql_prompt}
\lstdefinestyle{myGrammarStyle}{
    basicstyle=\scriptsize\ttfamily, % Reduce font size
    commentstyle=\color{green},
    keywordstyle=\color{blue},
    stringstyle=\color{orange},
    numbers=left, % Line numbers on left
    numberstyle=\tiny\color{gray}, % Line numbers styling
    breaklines=true, % Wrap long lines
    frame=single, % Frame around the code
    framesep=3pt, % Adjust frame separation
    xleftmargin=5pt, % Adjust left margin
    xrightmargin=5pt, % Adjust right margin
    backgroundcolor=\color{yellow!0}, % Background color
    tabsize=2, % Tab size
    captionpos=b, % Caption position bottom
    aboveskip=5pt, % Reduce space above the listing
    belowskip=5pt, % Reduce space below the listing
    linewidth=0.9\linewidth, % Set custom line width to less than text width
    escapeinside={(*@}{@*)}, % for escaping to LaTeX
    % morekeywords={Output, JSON}, % Add specific words to be highlighted
    % keywordstyle=\color{red}\bfseries, % Style for additional keywords
}

\begin{lstlisting}[style=myGrammarStyle, caption= Example Python prompt from the HumanEval dataset ~\cite{athiwaratkun2023multilingual}]
def has_close_elements(numbers: List[float], threshold: float) -> bool:
        """ Check if in given list of numbers, are any two numbers closer to each other than
        given threshold.
        >>> has_close_elements([1.0, 2.0, 3.0], 0.5)
        False
        >>> has_close_elements([1.0, 2.8, 3.0, 4.0, 5.0, 2.0], 0.3)
        True
        """

\end{lstlisting}
\label{prompt:python_prompt}
\lstdefinestyle{myGrammarStyle}{
    basicstyle=\scriptsize\ttfamily, % Reduce font size
    commentstyle=\color{green},
    keywordstyle=\color{blue},
    stringstyle=\color{orange},
    numbers=left, % Line numbers on left
    numberstyle=\tiny\color{gray}, % Line numbers styling
    breaklines=true, % Wrap long lines
    frame=single, % Frame around the code
    framesep=3pt, % Adjust frame separation
    xleftmargin=5pt, % Adjust left margin
    xrightmargin=5pt, % Adjust right margin
    backgroundcolor=\color{yellow!0}, % Background color
    tabsize=2, % Tab size
    captionpos=b, % Caption position bottom
    aboveskip=5pt, % Reduce space above the listing
    belowskip=5pt, % Reduce space below the listing
    linewidth=0.9\linewidth, % Set custom line width to less than text width
    escapeinside={(*@}{@*)}, % for escaping to LaTeX
}

\begin{lstlisting}[style=myGrammarStyle, caption= Example Go prompt from the HumanEval dataset ~\cite{athiwaratkun2023multilingual}]
package main

import (
	"encoding/json"
	"reflect"
)
// You're an expert Golang programmer
// Check if in given list of numbers, are any two numbers closer to each other than
// given threshold.
// >>> has_close_elements([1.0, 2.0, 3.0], 0.5)
// False
// >>> has_close_elements([1.0, 2.8, 3.0, 4.0, 5.0, 2.0], 0.3)
// True
// 
func has_close_elements (numbers []float64, threshold float64) bool {

\end{lstlisting}
\label{prompt:go_prompt}

\subsection{Grammars Used in the Evaluation}
\label{sec:grammar}

\subsubsection{JSON Grammar}
\  

\lstdefinestyle{myGrammarStyle}{
    basicstyle=\scriptsize\ttfamily, % Reduce font size
    commentstyle=\color{gray},
    keywordstyle=\color{blue},
    stringstyle=\color{orange},
    numbers=left, % Line numbers on left
    numberstyle=\tiny\color{gray}, % Line numbers styling
    breaklines=true, % Wrap long lines
    frame=single, % Frame around the code
    framesep=3pt, % Adjust frame separation
    xleftmargin=5pt, % Adjust left margin
    xrightmargin=5pt, % Adjust right margin
    backgroundcolor=\color{yellow!4}, % Background color
    tabsize=2, % Tab size
    captionpos=b, % Caption position bottom
    aboveskip=5pt, % Reduce space above the listing
    belowskip=5pt, % Reduce space below the listing
    linewidth=0.9\linewidth, % Set custom line width to less than text width
    escapeinside={(*@}{@*)}, % for escaping to LaTeX
}

\begin{lstlisting}[style=myGrammarStyle, caption=JSON Grammar]
?start: value

?value: object
| array
| UNESCAPED_STRING
| SIGNED_NUMBER      -> number
| "true"             -> true
| "false"            -> false
| "null"             -> null

array  : "[" [value ("," value)*] "]"
object : "{" [pair ("," pair)*] "}"
pair   : UNESCAPED_STRING ":" value

UNESCAPED_STRING: /\"[^"]*\"/

DIGIT: "0".."9"
HEXDIGIT: "a".."f"|"A".."F"|DIGIT
INT: DIGIT+
SIGNED_INT: ["+"|"-"] INT
DECIMAL: INT "." INT? | "." INT


_EXP: ("e"|"E") SIGNED_INT
FLOAT: INT _EXP | DECIMAL _EXP?
NUMBER: FLOAT | INT
SIGNED_NUMBER: ["+"|"-"] NUMBER
WS: /[ \t\f\r\n]/+

%ignore WS
\end{lstlisting}
\label{gram:json_grammar}

\subsubsection{SQL Grammar}
\  

\lstdefinestyle{myGrammarStyle}{
    basicstyle=\scriptsize\ttfamily, % Reduce font size
    commentstyle=\color{gray},
    keywordstyle=\color{blue},
    stringstyle=\color{orange},
    numbers=left, % Line numbers on left
    numberstyle=\tiny\color{gray}, % Line numbers styling
    breaklines=true, % Wrap long lines
    frame=single, % Frame around the code
    framesep=3pt, % Adjust frame separation
    xleftmargin=5pt, % Adjust left margin
    xrightmargin=5pt, % Adjust right margin
    backgroundcolor=\color{yellow!4}, % Background color
    tabsize=2, % Tab size
    captionpos=b, % Caption position bottom
    aboveskip=5pt, % Reduce space above the listing
    belowskip=5pt, % Reduce space below the listing
    linewidth=0.9\linewidth, % Set custom line width to less than text width
    escapeinside={(*@}{@*)}, % for escaping to LaTeX
}

\begin{lstlisting}[style=myGrammarStyle, caption=SQL Grammar]

start: set_expr ";"? -> final

set_expr: query_expr
        | set_expr "UNION"i ["DISTINCT"i] set_expr -> union_distinct
        | set_expr "UNION"i "ALL"i set_expr -> union_all
        | set_expr "INTERSECT"i ["DISTINCT"i] set_expr -> intersect_distinct
        | set_expr "EXCEPT"i ["DISTINCT"i] set_expr -> except_distinct
        | set_expr "EXCEPT"i "ALL"i set_expr -> except_all

query_expr: select [ "ORDER"i "BY"i (order_by_expr ",")*  order_by_expr] [ "LIMIT"i limit_count [ "OFFSET"i skip_rows ] ]

select: "SELECT"i [SELECT_CONSTRAINT] [(select_expr ",")*] select_expr "FROM"i [(from_expr ",")*] from_expr [ "WHERE"i where_expr ] [ "GROUP"i "BY"i [(groupby_expr ",")*] groupby_expr ] [ "HAVING"i having_expr] [ "WINDOW"i window_expr ]

where_expr: bool_expression

select_expr.0: expression_math [ "AS"i alias ] -> select_expression

?from_expr: from_item -> from_expression

order_by_expr: order -> order_by_expression

having_expr: bool_expression

groupby_expr: expression -> group_by

window_expr: [window_expr ","] _window_name "AS"i ( window_definition )

from_item: table_name [ "AS"i alias ] -> table
            | join -> join
            | cross_join -> cross_join_expression
            | subquery
table_name: name

subquery: ( "(" (query_expr | join | cross_join) ")" ) [ "AS"i alias ]

cross_join: from_item "CROSS"i "JOIN"i from_item
join: from_item JOIN_EXPR from_item [ "ON"i bool_expression ] -> join_expression

JOIN_EXPR.5: (JOIN_TYPE WS)? "JOIN"i
JOIN_TYPE: "INNER"i | "OUTER"i? | JOIN_DIRECTION (WS "OUTER"i)? | JOIN_DIRECTION
JOIN_DIRECTION: "FULL"i | "LEFT"i | "RIGHT"i

?expression_math: expression_product
               | expression_math "+" expression_product -> expression_add
               | expression_math "-" expression_product -> expression_sub
               | "CASE"i (when_then)+ "ELSE"i expression_math "END"i -> case_expression
               | "CAST"i "(" expression_math "AS"i TYPENAME ")" -> as_type
               | "CAST"i "(" literal "AS"i TYPENAME ")" -> literal_cast
               | AGGREGATION expression_math ")" [window_form] -> sql_aggregation
               | "RANK"i "(" ")" window_form -> rank_expression
               | "DENSE_RANK"i "(" ")" window_form -> dense_rank_expression
               | "COALESCE"i "(" [(expression_math ",")*] expression_math ")" -> coalesce_expression
               | subquery -> subquery_expression

window_form: "OVER"i "(" ["PARTITION"i "BY"i (partition_by ",")* partition_by] ["ORDER"i "BY"i (order ",")* order [ row_range_clause ] ] ")"

partition_by: expression_math

row_range_clause: ( ROWS | RANGE ) frame_extent
frame_extent: frame_between | frame_preceding
frame_between: "BETWEEN"i frame_bound "AND"i frame_bound
frame_bound: frame_preceding | frame_following | "CURRENT"i "ROW"i
frame_preceding: UNBOUNDED PRECEDING | INT_NUMBER PRECEDING
frame_following: UNBOUNDED FOLLOWING | INT_NUMBER FOLLOWING
RANGE: "RANGE"i
ROWS: "ROWS"i
UNBOUNDED: "UNBOUNDED"i
PRECEDING: "PRECEDING"i
FOLLOWING: "FOLLOWING"i

when_then: "WHEN"i bool_expression "THEN"i expression_math
order: expression_math ["ASC"i] -> order_asc
          | expression_math "DESC"i -> order_desc


?expression_product: expression_parens
                  | expression_product "*" expression_parens -> expression_mul
                  | expression_product "/" expression_parens -> expression_div

?expression_parens: expression
                  | "(" expression_parens "*" expression ")" -> expression_mul
                  | "(" expression_parens "/" expression ")" -> expression_div
                  | "(" expression_parens "+" expression ")" -> expression_add
                  | "(" expression_parens "-" expression ")" -> expression_sub

column_name: [name "."] (name | STAR)
?expression: column_name -> column_name
            | literal


SELECT_CONSTRAINT.9: "ALL"i | "DISTINCT"i
TYPENAME:  "object"i
         | "varchar"i
         | "integer"i
         | "int16"i
         | "smallint"i
         | "int32"i
         | "int64"i
         | "int"i
         | "bigint"i
         | "float16"i
         | "float32"i
         | "float64"i
         | "float"i
         | "bool"i
         | "datetime64"i
         | "timestamp"i
         | "time"i
         | "date"i
         | "cateSQLry"i
         | "string"i
AGGREGATION.8: ("SUM("i | "AVG("i | "MIN("i | "MAX("i | "COUNT("i "DISTINCT"i | "COUNT("i)
alias: name -> alias_string
_window_name: name
limit_count: INT_NUMBER -> limit_count
skip_rows: INT_NUMBER
bool_expression: bool_parentheses
                 | bool_expression "AND"i bool_parentheses -> bool_and
                 | bool_expression "OR"i bool_parentheses -> bool_or
bool_parentheses: comparison_type
                 | "(" bool_expression "AND"i comparison_type ")" -> bool_and
                 | "(" bool_expression "OR"i comparison_type ")" -> bool_or
                 | "EXISTS"i subquery -> exists
comparison_type: equals | not_equals | greater_than | less_than | greater_than_or_equal
| less_than_or_equal | between | in_expr | not_in_expr | subquery_in | subquery_not_in | is_null | is_not_null | like_expr | not_like_expr

equals: expression_math "=" expression_math
is_null: expression_math "IS"i "NULL"i
is_not_null: expression_math "IS"i "NOT"i "NULL"i
not_equals: expression_math ("<>" | "!=") expression_math
greater_than: expression_math ">" expression_math
less_than: expression_math "<" expression_math
greater_than_or_equal: expression_math ">=" expression_math
less_than_or_equal: expression_math "<=" expression_math
between: expression_math "BETWEEN"i expression_math "AND"i expression_math

// `LIKE` and `NOT LIKE`
like_expr: expression_math "LIKE"i expression_math
not_like_expr: expression_math "NOT"i "LIKE"i expression_math

// `IN` and `NOT IN`
in_expr: expression_math "IN"i "(" [expression_math ","]* expression_math ")"
subquery_in: expression_math "IN"i subquery
not_in_expr: expression_math "NOT"i "IN"i "(" [expression_math ","]* expression_math ")"
subquery_not_in: expression_math "NOT"i "IN"i subquery

?literal: boolean -> bool
       | number_expr -> number
       | /'([^'])+'|''/ -> string
       | timestamp_expression -> timestamp_expression
boolean: "TRUE"i -> true
       | "FALSE"i -> false
?number_expr: product

?product: INT_NUMBER -> integer
       | FLOAT -> float

INT_NUMBER: /[1-9][0-9]*/

STAR: "*"
window_definition:
timestamp_expression: "NOW"i "(" ")" -> datetime_now
                    | "TODAY"i "(" ")" -> date_today

date: YEAR "-" MONTH "-" DAY
YEAR: /[0-9]{4}/
MONTH: /[0-9]{2}/
DAY: /[0-9]{2}/
time: HOURS ":" MINUTES ":" SECONDS
HOURS: /[0-9]{2}/
MINUTES: /[0-9]{2}/
SECONDS: /[0-9]{2}/
name: CNAME | ESCAPED_STRING

_STRING_INNER: /(?:[^"\\]|\\.)*?/
ESCAPED_STRING: "\"" _STRING_INNER "\""

%import common.CNAME
%import common.WS
%import common.SQL_COMMENT
%import common.WS_INLINE
%import common.FLOAT

%ignore WS
%ignore SQL_COMMENT


\end{lstlisting}
\label{gram:sql_grammar}

\subsubsection{Python Grammar}
\  

\lstdefinestyle{myGrammarStyle}{
    basicstyle=\scriptsize\ttfamily, % Reduce font size
    commentstyle=\color{gray},
    keywordstyle=\color{blue},
    stringstyle=\color{orange},
    numbers=left, % Line numbers on left
    numberstyle=\tiny\color{gray}, % Line numbers styling
    breaklines=true, % Wrap long lines
    frame=single, % Frame around the code
    framesep=3pt, % Adjust frame separation
    xleftmargin=5pt, % Adjust left margin
    xrightmargin=5pt, % Adjust right margin
    backgroundcolor=\color{yellow!4}, % Background color
    tabsize=2, % Tab size
    captionpos=b, % Caption position bottom
    aboveskip=5pt, % Reduce space above the listing
    belowskip=5pt, % Reduce space below the listing
    linewidth=0.9\linewidth, % Set custom line width to less than text width
    escapeinside={(*@}{@*)}, % for escaping to LaTeX
}

\begin{lstlisting}[style=myGrammarStyle, caption=Python Grammar]
single_input: _NL | simple_stmt | compound_stmt _NL
start: (_NL | stmt)*
eval_input: testlist _NL*

!decorator: "@" dotted_name [ "(" [arguments] ")" ] _NL
decorators: decorator+
decorated: decorators (classdef | funcdef | async_funcdef)

async_funcdef: "async" funcdef
funcdef: "def" NAME "(" parameters? ")" ["->" test] ":" ( suite | _NL)

!parameters: paramvalue ("," paramvalue)* ["," [ starparams | kwparams]]
          | starparams
          | kwparams
starparams: "*" typedparam? ("," paramvalue)* ["," kwparams]
kwparams: "**" typedparam

?paramvalue: typedparam ["=" test]
?typedparam: NAME [":" test]

!varargslist: (vfpdef ["=" test] ("," vfpdef ["=" test])* ["," [ "*" [vfpdef] ("," vfpdef ["=" test])* ["," ["**" vfpdef [","]]] | "**" vfpdef [","]]]
  | "*" [vfpdef] ("," vfpdef ["=" test])* ["," ["**" vfpdef [","]]]
  | "**" vfpdef [","])

vfpdef: NAME

?stmt: (simple_stmt | compound_stmt ) ["eof"]
!?simple_stmt: small_stmt (";" small_stmt)* [";"] _NL
?small_stmt: (expr_stmt | del_stmt | pass_stmt | flow_stmt | import_stmt | global_stmt | nonlocal_stmt | assert_stmt)
?expr_stmt: testlist_star_expr (annassign | augassign (yield_expr|testlist)
         | ("=" (yield_expr|testlist_star_expr))*)
annassign: ":" test ["=" test]
!?testlist_star_expr: (test|star_expr) ("," (test|star_expr))* [","]
!augassign: ("+=" | "-=" | "*=" | "@=" | "/=" | "%=" | "&=" | "|=" | "^=" | "<<=" | ">>=" | "**=" | "//=")
// For normal and annotated assignments, additional restrictions enforced by the interpreter
del_stmt: "del" exprlist
pass_stmt: "pass"
flow_stmt: break_stmt | continue_stmt | return_stmt | raise_stmt | yield_stmt
break_stmt: "break"
continue_stmt: "continue"
return_stmt: "return" [testlist]
yield_stmt: yield_expr
raise_stmt: "raise" [test ["from" test]]
import_stmt: import_name | import_from
import_name: "import" dotted_as_names
// note below: the ("." | "...") is necessary because "..." is tokenized as ELLIPSIS
import_from: "from" (dots? dotted_name | dots) "import" ("*" | "(" import_as_names ")" | import_as_names)
!dots: "."+
import_as_name: NAME ["as" NAME]
dotted_as_name: dotted_name ["as" NAME]
!import_as_names: import_as_name ("," import_as_name)* [","]
dotted_as_names: dotted_as_name ("," dotted_as_name)*
dotted_name: NAME ("." NAME)*
global_stmt: "global" NAME ("," NAME)*
nonlocal_stmt: "nonlocal" NAME ("," NAME)*
assert_stmt: "assert" test ["," test]

compound_stmt: if_stmt | while_stmt | for_stmt | try_stmt | with_stmt | funcdef | classdef | decorated | async_stmt
async_stmt: "async" (funcdef | with_stmt | for_stmt)
if_stmt: "if" test ":" suite ("elif" test ":" suite)* ["else" ":" suite]
while_stmt: "while" test ":" suite ["else" ":" suite]
for_stmt: "for" exprlist "in" testlist ":" suite ["else" ":" suite]
try_stmt: ("try" ":" suite ((except_clause ":" suite)+ ["else" ":" suite] ["finally" ":" suite] | "finally" ":" suite))
with_stmt: "with" with_item ("," with_item)*  ":" suite
with_item: test ["as" expr]
// NB compile.c makes sure that the default except clause is last
except_clause: "except" [test ["as" NAME]]
suite: simple_stmt | _NL _INDENT stmt+ _DEDENT

?test: or_test ["if" or_test "else" test] | lambdef
?test_nocond: or_test | lambdef_nocond
lambdef: "lambda" [varargslist] ":" test
lambdef_nocond: "lambda" [varargslist] ":" test_nocond
?or_test: and_test ("or" and_test)*
?and_test: not_test ("and" not_test)*

?not_test: "not" not_test -> not
         | comparison
?comparison: expr (_comp_op expr)*
star_expr: "*" expr
?expr: xor_expr ("|" xor_expr)*
?xor_expr: and_expr ("^" and_expr)*
?and_expr: shift_expr ("&" shift_expr)*
?shift_expr: arith_expr (_shift_op arith_expr)*
?arith_expr: term (_add_op term)*
?term: factor (_mul_op factor)*
?factor: _factor_op factor | power

!_factor_op: "+"|"-"|"~"
!_add_op: "+"|"-"
!_shift_op: "<<"|">>"
!_mul_op: "*"|"@"|"/"|"%"|"//"
// <> isn't actually a valid comparison operator in Python. It's here for the
// sake of a __future__ import described in PEP 401 (which really works :-)
!_comp_op: "<"|">"|"=="|">="|"<="|"<>"|"!="|"in"|"not" "in"|"is"|"is" "not"

?power: await_expr ["**" factor]
!await_expr: ["await"] atom_expr

?atom_expr: atom_expr "(" [arguments] ")"      -> funccall
          | atom_expr "[" subscriptlist "]"  -> getitem
          | atom_expr "." NAME               -> getattr
          | atom

?atom: "(" [yield_expr|testlist_comp] ")" -> tuple
     | "[" [testlist_comp] "]"  -> list
     | "{" [dictorsetmaker] "}" -> dict
     | NAME -> var
     | number | string+
     | "(" test ")"
     | "..." -> ellipsis
     | "None"    -> const_none
     | "True"    -> const_true
     | "False"   -> const_false

!?testlist_comp: (test|star_expr) [comp_for | ("," (test|star_expr))+ [","] | ","]
!subscriptlist: subscript ("," subscript)* [","]
subscript: test | [test] ":" [test] [sliceop]
sliceop: ":" [test]
!exprlist: (expr|star_expr) ("," (expr|star_expr))* [","]
!testlist: test ("," test)* [","]
!dictorsetmaker: ( ((test ":" test | "**" expr) (comp_for | ("," (test ":" test | "**" expr))* [","])) | ((test | star_expr) (comp_for | ("," (test | star_expr))* [","])) )

classdef: "class" NAME ["(" [arguments] ")"] ":" suite
!arguments: argvalue ("," argvalue)*  ["," [ starargs | kwargs]]
         | starargs
         | kwargs
         | test comp_for

!starargs: "*" test ("," "*" test)* ("," argvalue)* ["," kwargs]
kwargs: "**" test

?argvalue: test ["=" test]

comp_iter: comp_for | comp_if | async_for
async_for: "async" "for" exprlist "in" or_test [comp_iter]
comp_for: "for" exprlist "in" or_test [comp_iter]
comp_if: "if" test_nocond [comp_iter]

// not used in grammar, but may appear in "node" passed from Parser to Compiler
encoding_decl: NAME

yield_expr: "yield" [yield_arg]
yield_arg: "from" test | testlist


number: DEC_NUMBER | HEX_NUMBER | OCT_NUMBER | FLOAT_NUMBER 

string: STRING | LONG_STRING 

// Tokens
NAME: /[a-zA-Z_]\w*/
COMMENT: /#.*(\n[\t ]*)+/ | LONG_STRING
_NL: ( /(\r?\n[\t ]*)+/ | COMMENT)+

LONG_STRING: /[ubf]?r?("""(?<!\\).*?"""|'''(?<!\\).*?''')/is

DEC_NUMBER: /0|[1-9]\d*/i
HEX_NUMBER.2: /0x[\da-f]*/i
OCT_NUMBER.2: /0o[0-7]*/i
BIN_NUMBER.2 : /0b[0-1]*/i
FLOAT_NUMBER.2: /((\d+\.\d*|\.\d+)(e[-+]?\d+)?|\d+(e[-+]?\d+))/i

%import common.WS_INLINE

%declare _INDENT _DEDENT
%ignore WS_INLINE
%ignore /\\[\t \f]*\r?\n/   // LINE_CONT
%ignore COMMENT

\end{lstlisting}
\label{gram:python_grammar}

\newpage
\subsubsection{Go Grammar}
\  

\lstdefinestyle{myGrammarStyle}{
    basicstyle=\scriptsize\ttfamily, % Reduce font size
    commentstyle=\color{gray},
    keywordstyle=\color{blue},
    stringstyle=\color{orange},
    numbers=left, % Line numbers on left
    numberstyle=\tiny\color{gray}, % Line numbers styling
    breaklines=true, % Wrap long lines
    frame=single, % Frame around the code
    framesep=3pt, % Adjust frame separation
    xleftmargin=5pt, % Adjust left margin
    xrightmargin=5pt, % Adjust right margin
    backgroundcolor=\color{yellow!4}, % Background color
    tabsize=2, % Tab size
    captionpos=b, % Caption position bottom
    aboveskip=5pt, % Reduce space above the listing
    belowskip=5pt, % Reduce space below the listing
    linewidth=0.9\linewidth, % Set custom line width to less than text width
    escapeinside={(*@}{@*)}, % for escaping to LaTeX
}

\begin{lstlisting}[style=myGrammarStyle, caption=Go Grammar]

start: package_clause eos (import_decl eos)* ((function_decl | method_decl | declaration) eos "eoc"?)*

package_clause: "package" NAME

import_decl: "import"  (import_spec | "(" (import_spec eos)* ")")

import_spec: ("." | NAME)? import_path

import_path: string_

declaration: const_decl | type_decl | var_decl

const_decl: "const"  (const_spec | "(" (const_spec eos)* ")")

const_spec: identifier_list (type_? "=" expression_list)?

identifier_list: NAME ("," NAME)*

expression_list: expression ("," expression)*

type_decl: "type" (type_spec | "(" (type_spec eos)* ")")

type_spec: alias_decl | type_def

alias_decl : NAME "=" type_

type_def : NAME type_parameters? type_

type_parameters : "[" type_parameter_decl ("," type_parameter_decl)* "]"

type_parameter_decl : identifier_list type_element

type_element : type_term ("|" type_term)*

type_term : "~"? type_

// Function declarations

function_decl: "func" NAME type_parameters? signature ("{" statement_list? ("}" | "eof"))? 

method_decl: "func" receiver NAME signature block?

receiver: parameters

var_decl: "var" (var_spec | "(" (var_spec eos)* ")")

var_spec: identifier_list (type_ ("=" expression_list)? | "=" expression_list)

block: "{" statement_list? "}"

statement_list: ((";"? | EOS?) statement eos)+

statement: declaration | labeled_stmt | simple_stmt | go_stmt | return_stmt | break_stmt | continue_stmt | goto_stmt | fallthrough_stmt | block | if_stmt | switch_stmt | select_stmt | for_stmt | defer_stmt

simple_stmt: send_stmt | inc_dec_stmt | assignment | expression | short_var_decl

send_stmt: expression  "<-" expression

inc_dec_stmt: expression ("++" | "--")

assignment: expression assign_op expression | expression_list "=" expression_list

assign_op: "+=" | "-=" | "|=" | "^=" | "*=" | "/=" | "%=" | "<<=" | ">>=" | "&=" | "&^="

short_var_decl: expression_list ":=" expression_list

labeled_stmt: NAME ":" statement?

return_stmt: "return" expression_list?

break_stmt: "break" NAME?

continue_stmt: "continue" NAME?

goto_stmt: "goto"  NAME

fallthrough_stmt: "fallthrough" 

defer_stmt: "defer" expression

if_stmt: "if"  ( expression | eos expression | simple_stmt eos expression) block ("else" (if_stmt | block))?

switch_stmt: expr_switch_stmt | type_switch_stmt

expr_switch_stmt: "switch"  (expression? | simple_stmt? eos expression?) "{" expr_case_clause* "}"

expr_case_clause: expr_switch_case ":" statement_list?

expr_switch_case: "case" expression_list | "default"

type_switch_stmt: "switch"  ( type_switch_guard | eos type_switch_guard | simple_stmt eos type_switch_guard) "{" type_case_clause* "}"

type_switch_guard: (NAME ":=")? NAME "." "(" "type"  ")"

type_case_clause: type_switch_case ":" statement_list?

type_switch_case: "case" type_list | "default"

type_list: (type_ | "nil" ) ("," (type_ | "nil"  ))*

select_stmt: "select" "{" comm_clause* "}"

comm_clause: comm_case ":" statement_list?

comm_case: "case" (send_stmt | recv_stmt) | "default"

recv_stmt: (expression_list "=" | identifier_list ":=")? expression

for_stmt: "for" [for_clause] block

for_clause: simple_stmt (eos expression eos simple_stmt)? | range_clause

range_clause: (expression_list "=" | expression_list ":=") "range"  expression

go_stmt: "go"expression

type_: literal_type | var_or_type_name type_args? | "(" type_ ")" 

channel_type

type_args : "--"

var_or_type_name: NAME "." NAME | NAME | NAME "." "(" type_ ")"

array_type: "[" array_length "]" element_type

array_length: expression

element_type: type_

pointer_type: "*" type_

interface_type: "interface" "{" ((method_spec | type_element ) eos)* "}"

slice_type: "[" "]" element_type

// It's possible to replace `type` with more restricted type_lit list and also pay attention to nil maps
map_type: "map" "[" type_ "]" element_type

channel_type: ("'chan"  | "chan"   "<-" |  "<-" "chan" ) element_type

method_spec: NAME parameters result | NAME parameters

function_type: "func" signature

signature: parameters result?

result: parameters | type_

parameters: "(" parameter_decl ("," parameter_decl)* ","? ")" | "(" ")" 

// a comma-separated list of either (a) name, (b) type, or (c) name and type 
// parameter_decl: identifier_list? "..."? type_


// Although following is overapproximate it's an easy way to avoid reduce/reduce conflicts
parameter_decl: (type_ | "..."? type_ | NAME type_)


expression: primary_expr 
            | ("+" | "-" | "!" | "^" | "*" | "&" | "<-") expression 
            | expression ("*" | "/" | "%" | "<<" | ">>" | "&" | "&^") expression 
            | expression ("+" | "-" | "|" | "^") expression 
            | expression ("==" | "!=" | "<" | "<=" | ">" | ">=") expression 
            | expression "&&" expression 
            | expression "||" expression

primary_expr: operand | primary_expr ("." (NAME | "(" type_ ")") | index | slice_ | arguments) | type_

// Giving operand higher precedence than type_ is a hack to avoid reduce/reduce conflicts
operand.3: literal | NAME | "(" expression ")" // removed NAME type_args?

literal: basic_lit | composite_lit | function_lit

basic_lit: "nil" | integer | string_ | FLOAT_LIT | CHAR_LIT

integer: DECIMAL_LIT | BINARY_LIT | OCTAL_LIT | HEX_LIT

DECIMAL_LIT: /0|[1-9]\d*/i
HEX_LIT.2: /0x[\da-f]*/i
OCTAL_LIT.2: /0o[0-7]*/i
BINARY_LIT.2 : /0b[0-1]*/i
FLOAT_LIT.2: /((\d+\.\d*|\.\d+)(e[-+]?\d+)?|\d+(e[-+]?\d+))/i
CHAR_LIT: /'.'/i

composite_lit: literal_type literal_value

literal_type: struct_type | array_type | "[" "..." "]" element_type | slice_type | map_type  | "interface" "{" "}"

literal_value: "{" (element_list ","?)? "}"

element_list: keyed_element ("," keyed_element)*

keyed_element: (key ":")? element

key: expression | literal_value

element: expression | literal_value

struct_type: "struct" "{" (field_decl eos)* "}"

field_decl: (identifier_list type_ | embedded_field) string_?

string_: RAW_STRING_LIT | INTERPRETED_STRING_LIT

RAW_STRING_LIT: /`.*?`/
INTERPRETED_STRING_LIT: /".*?"/i

embedded_field: "*"? (NAME "." NAME | NAME)  type_args?

function_lit: "func" signature block // function

index: "[" expression "]"

slice_: "[" ( expression? ":" expression? | expression? ":" expression ":" expression) "]"

type_assertion: "." "(" type_ ")"

arguments: "(" ( expression_list? "..."? ","?)? ")"

eos: ";" | EOS // | {this.closingBracket()}?
	
NAME : /[a-zA-Z_]\w*/
EOS: _NL | ";" | "/*' .*? '*/"       

COMMENT : /\/\/[^\n]*/ 
_NL: ( /(\r?\n[\t ]*)+/ | COMMENT)+

%ignore /[\t ]/
%ignore /\\[\t \f]*\r?\n/   // LINE_CONT

\end{lstlisting}
\label{gram:go_grammar}

\end{document}